\documentclass{article} %

\usepackage{etoolbox}
\usepackage{comment}
\newtoggle{icml}
\togglefalse{icml}

\newcommand{\icml}[1]{\iftoggle{icml}{#1}{}}
\newcommand{\arxiv}[1]{\iftoggle{icml}{}{#1}}

\icml{
  \usepackage{icml2026}
}

\arxiv{
\usepackage[letterpaper, left=1in, right=1in, top=1in,bottom=1in]{geometry}
  \usepackage{parskip}
  \usepackage[colorlinks=true, linkcolor=blue!70!black, citecolor=blue!70!black,urlcolor=black,breaklinks=true]{hyperref}
  \usepackage[dvipsnames]{xcolor}
}

\icml{
  \usepackage{parskip}
    \usepackage[colorlinks=true, linkcolor=blue!70!black, citecolor=blue!70!black,urlcolor=black,breaklinks=true, bookmarks=true]{hyperref}
}

\PassOptionsToPackage{hypertexnames=false}{hyperref}  %

\usepackage{amsmath}
\usepackage{microtype}
\usepackage{hhline}

\usepackage{amsthm}

\usepackage{bbm}
\usepackage{amsfonts}
\usepackage{amssymb}
\usepackage[nameinlink,capitalize]{cleveref}

\usepackage{mathtools}

\usepackage{xargs}

\arxiv{
\usepackage{algorithm}

\usepackage{natbib}
\bibliographystyle{plainnat}
\bibpunct{(}{)}{;}{a}{,}{,}
}

\usepackage{xpatch}

\theoremstyle{plain}
\newtheorem{theorem}{Theorem}[section]
\newtheorem{proposition}[theorem]{Proposition}
\newtheorem{lemma}[theorem]{Lemma}
\newtheorem{corollary}[theorem]{Corollary}
\theoremstyle{definition}
\newtheorem{definition}[theorem]{Definition}
\newtheorem{assumption}[theorem]{Assumption}
\theoremstyle{remark}

\newtheorem{example}[theorem]{Example}

\newcommand{\pfref}[1]{Proof of \cref{#1}}

\renewcommand{\eqref}[1]{\texorpdfstring{\hyperref[#1]{(\ref*{#1})}}{(\ref*{#1})}}

\crefformat{equation}{#2Eq.\,(#1)#3}
\Crefformat{equation}{#2Eq.\,(#1)#3}
\Crefformat{figure}{#2Figure~#1#3}
\Crefformat{assumption}{#2Assumption~#1#3}
\Crefname{assumption}{Assumption}{Assumptions}
\crefname{fact}{Fact}{Facts}
\Crefformat{figure}{#2Figure #1#3}
\Crefformat{assumption}{#2Assumption #1#3}

\usepackage{crossreftools}
\makeatletter
\renewenvironment{proof}[1][Proof]%
{%
  \par\noindent{\bfseries\upshape {#1.}\ }%
}%
{\qed\newline}
\makeatother

\xpatchcmd{\proof}{\itshape}{\normalfont\proofnameformat}{}{}
\newcommand{\proofnameformat}{\bfseries}

\usepackage[most]{tcolorbox}

\tcbset {
  base/.style={
    arc=0mm, 
    bottomtitle=0.5mm,
    boxrule=0mm,
    colbacktitle=!10!white, 
    coltitle=black, 
    fonttitle=\bfseries, 
    left=2.5mm,
    leftrule=1mm,
    right=3.5mm,
    title={#1},
    toptitle=0.75mm, 
  }
}

\newtcolorbox{mainbox}[1]{
  colframe=blue!10!black,
  colbacktitle=blue!50!black!30!white,
  colback=blue!2!white,
  enhanced,
  fonttitle=\bfseries,
  attach boxed title to top left={yshift=-2.5mm},
  boxed title style={size=small,colframe=blue!40!black,colback=blue!40!black},
  title={\small\textcolor{white}{\textsc{#1}}}
}

\newtcolorbox{minbox}[1]{
  colframe=blue!10!black,
  colbacktitle=blue!50!black!30!white,
  colback=blue!2!white,
  enhanced,
  fonttitle=\bfseries,
}

\DeclarePairedDelimiter{\abs}{\lvert}{\rvert} %
\DeclarePairedDelimiter{\brk}{[}{]}
\DeclarePairedDelimiter{\crl}{\{}{\}}
\DeclarePairedDelimiter{\prn}{(}{)}
\DeclarePairedDelimiter{\nrm}{\|}{\|}
\DeclarePairedDelimiter{\tri}{\langle}{\rangle}

\def\ddefloop#1{\ifx\ddefloop#1\else\ddef{#1}\expandafter\ddefloop\fi}
\def\ddef#1{\expandafter\def\csname bb#1\endcsname{\ensuremath{\mathbb{#1}}}}
\ddefloop ABCDEFGHIJKLMNOPQRSTUVWXYZ\ddefloop
\def\ddefloop#1{\ifx\ddefloop#1\else\ddef{#1}\expandafter\ddefloop\fi}
\def\ddef#1{\expandafter\def\csname b#1\endcsname{\ensuremath{\mathbf{#1}}}}
\ddefloop ABCDEFGHIJKLMNOPQRSTUVWXYZ\ddefloop
\def\ddef#1{\expandafter\def\csname sf#1\endcsname{\ensuremath{\mathsf{#1}}}}
\ddefloop ABCDEFGHIJKLMNOPQRSTUVWXYZ\ddefloop
\def\ddef#1{\expandafter\def\csname c#1\endcsname{\ensuremath{\mathcal{#1}}}}
\ddefloop ABCDEFGHIJKLMNOPQRSTUVWXYZ\ddefloop
\def\ddef#1{\expandafter\def\csname h#1\endcsname{\ensuremath{\widehat{#1}}}}
\ddefloop ABCDEFGHIJKLMNOPQRSTUVWXYZ\ddefloop
\def\ddef#1{\expandafter\def\csname hc#1\endcsname{\ensuremath{\widehat{\mathcal{#1}}}}}
\ddefloop ABCDEFGHIJKLMNOPQRSTUVWXYZ\ddefloop
\def\ddef#1{\expandafter\def\csname t#1\endcsname{\ensuremath{\widetilde{#1}}}}
\ddefloop ABCDEFGHIJKLMNOPQRSTUVWXYZ\ddefloop
\def\ddef#1{\expandafter\def\csname tc#1\endcsname{\ensuremath{\widetilde{\mathcal{#1}}}}}
\ddefloop ABCDEFGHIJKLMNOPQRSTUVWXYZ\ddefloop
\def\ddef#1{\expandafter\def\csname #1#1\endcsname{\ensuremath{\mathbb{#1}}}}
\ddefloop ABCDEFGHIJKLMNOPQRSTUVWXYZ\ddefloop
\def\ddef#1{\expandafter\def\csname #1\endcsname{\ensuremath{\mathbb{#1}}}}
\ddefloop ABCDEFGHIJKLMNOPQRSTUVWXYZ\ddefloop
\def\ddef#1{\expandafter\def\csname D#1\endcsname{\ensuremath{\Delta(\mathcal{#1})}}}
\ddefloop ABCDEFGHIJKLMNOPQRSTUVWXYZ\ddefloop
\def\ddefloop#1{\ifx\ddefloop#1\else\ddef{#1}\expandafter\ddefloop\fi}
\def\ddef#1{\expandafter\def\csname scr#1\endcsname{\ensuremath{\mathscr{#1}}}}
\ddefloop ABCDEFGHIJKLMNOPQRSTUVWXYZ\ddefloop

\let\Pr\undefined

\DeclareMathOperator{\En}{\mathbb{E}}

\DeclareMathOperator{\Pr}{Pr}

\newcommand{\mc}[1]{\mathcal{#1}}

\newcommand{\wt}[1]{\widetilde{#1}}
\newcommand{\wh}[1]{\widehat{#1}}
\newcommand{\wb}[1]{\widebar{#1}}

\newcommand{\kl}[2]{D_{\mathsf{KL}}\prn*{#1\,\|\,#2}}
\newcommand{\Dkl}[2]{D_{\mathsf{KL}}\prn*{#1\,\|\,#2}}

\newcommand{\Dcov}[3][{\M}]{\cE_{#1}\prn*{#2\,\|\,#3}}

\newcommand{\Dhel}[2]{D_{\mathsf{H}}\prn*{#1,#2}}

\newcommand{\Dhels}[2]{D^{2}_{\mathsf{H}}\prn*{#1,#2}}

\newcommand{\Dchis}[2]{D_{\chi^2}\prn*{#1\dmid{}#2}}

\newcommand{\Dtv}[2]{D_{\mathsf{TV}}\prn*{#1,#2}}

\newcommand{\Dbl}[2]{D_{\mathsf{BL}}\prn*{#1,#2}}

\newcommand{\Dren}[3][\lambda]{D_{#1}\prn*{#2\,\|\,#3}}
\newcommand{\Ren}[3][\lambda]{\mathcal{R}_{#1}\prn*{#2\,\|\,#3}}
\newcommand{\Dsys}[3][\lambda]{\bar{D}_{#1}\prn*{#2\,\|\,#3}}

\newcommand{\eps}{\epsilon}

\newcommand{\ldef}{\vcentcolon=}

\newcommand{\pclip}[2][\B]{\mathsf{Clip}_{#1}(#2)}
\newcommand{\trunc}[2][\B]{\tau_{#1}(#2)}

\newcommand{\AlgWSC}{\textsf{FORS}}

\newcommandx{\gam}[3][1=x,2=z]{\gamma_{#2,#3}(#1)}
\newcommandx{\gamp}[3][1=x,2=z]{\dot\gamma_{#2,#3}(#1)}

\newcommandx{\gamz}[4][1=x,2=z,3=\lr,4=\xz]{\gamma_{#2,#3,#4}(#1)}
\newcommandx{\gamzp}[4][1=x,2=z,3=\lr,4=\xz]{\dot\gamma_{#2,#3,#4}(#1)}

\newcommand{\xz}{x_0}
\newcommand{\const}{\mathrm{const}}
\newcommand{\lr}{r}

\renewcommand{\B}{B}

\newcommand{\xhat}{\wh{x}}

\newcommand{\tp}{^\top}
\newcommand{\nrmF}[1]{\nrm{#1}_{\mathrm{F}}}
\newcommand{\tr}{\mathrm{tr}}

\newcommand{\Cov}{\mathrm{Cov}}

\newcommand{\epssc}{\varepsilon_{\mathsf{score}}}
\newcommand{\epssct}[1][k]{\varepsilon_{#1,\mathsf{score}}}

\newcommand{\epsjac}{\varepsilon_{\mathsf{Jacobi}}}

\newcommand{\CLSI}{C_{\mathsf{LSI}}}
\newcommand{\CPI}{C_{\mathsf{PI}}}

\newcommand{\rgo}{\nu}

\newcommand{\etabar}{\wb{\eta}}

\newcommand{\scfs}{\scf^\star}

\newcommand{\Poi}{\mathsf{Poisson}}

\newcommand{\xbar}{\bar{x}}

\newcommand{\cXbar}{\wb{\cX}}

\newcommand{\dstar}{\mathsf{d}_\star}

\newcommand{\alr}{a_{\lr}}
\newcommand{\blr}{b_{\lr}}
\newcommand{\alrp}{a_{\lr}'}
\newcommand{\blrp}{b_{\lr}'}
\newcommand{\scf}{\mathsf{s}}

\newcommand{\xp}{x_+}

\newcommand{\xr}[1]{X\sups{\leftarrow}_{#1}}

\newcommand{\pdata}{p_{\mathsf{data}}}

\newcommand{\alp}{\alpha}
\newcommand{\alpbar}{\bar{\alpha}}

\newcommand{\nuhat}{\wh{\nu}}

\newcommand{\CP}{C_{\rm P}}

\newcommand{\normal}[1]{\mathsf{N}\prn*{#1}}
\newcommand{\nor}{\mathsf{N}}

\newcommand{\Dn}{\mathsf{D}}
\newcommand{\Ds}{\Dn^\star}

\newcommand{\eqd}{\stackrel{\mathrm{d}}{=}}

\newcommand{\Xbar}{\bar{X}}
\newcommand{\Netabar}{\normal{0,\etabar\Id}}
\newcommand{\Wstar}{W^\star}
\newcommand{\LipF}[1][\delta]{L_{\mathrm{F},#1}}
\newcommand{\Lipop}[1][\delta]{L_{\mathrm{op},#1}}

\newcommand{\rhostar}{\rho^\star}
\newcommand{\rhohat}{\wh{\rho}}
\newcommand{\Enxp}{\En_{\cdot\mid \gd(x_+)}}
\newcommand{\Enxps}{\En_{\cdot\mid \gds(x_+)}}
\newcommand{\gd}{\mathsf{g}}
\newcommand{\gds}{\gd^\star}

\newcommand{\Ybar}{\wb{Y}}
\newcommand{\cov}{\mathrm{cov}}
\newcommand{\rhobar}{\wb{\rho}}

\renewcommand{\xr}[1]{X_{#1}}

\newcommand{\leqsim}{\approxleq}

\DeclarePairedDelimiter{\nrmop}{\|}{\|_{\mathrm{op}}}

\newcommand{\whp}[1][\delta]{with probability at least $1-#1$}

\newcommand{\Capx}{C_{\texttt{apx}}}

\newcommand{\Unif}{\mathsf{Unif}}

\newcommand{\var}{\mathrm{Var}}
\newcommand{\Var}{\var}

\newcommand{\approxleq}{\lesssim}

\newcommand{\Id}{\mathbf{I}}

\newcommand{\indic}{\mathbb{I}}

\renewcommand{\Pr}{\bbP}

\newcommand{\poly}{\mathrm{poly}}
\newcommand{\polylog}{\mathrm{polylog}}

\newcommand{\Ber}{\mathsf{Ber}}

\newcommand{\dmid}{\;\|\;}

\newcommand{\unif}{\mathsf{Unif}}

\newcommand{\deq}{\coloneqq}

\newcommand{\supp}{\mathrm{supp}}

\newcommand{\muhat}{\widehat{\mu}}

\def\multiset#1#2{\ensuremath{\left(\kern-.3em\left(\genfrac{}{}{0pt}{}{#1}{#2}\right)\kern-.3em\right)}}

\newcommand{\pbar}{\overline{p}}
\newcommand{\phat}{\wh{p}}

\newcommand{\norm}[1]{\left \lVert #1 \right \rVert}

\input{widebar}

\usepackage[utf8]{inputenc} %
\usepackage[T1]{fontenc}    %
\usepackage{url}            %
\usepackage{booktabs}       %
\usepackage{amsfonts}       %
\usepackage{nicefrac}       %
\usepackage{microtype}      %
\usepackage{makecell}
\usepackage{enumitem}
\usepackage{breakcites}
\usepackage{mathrsfs}

\usepackage[normalem]{ulem}

\arxiv{
\usepackage{algorithm}
\usepackage{algorithmic}
\usepackage{verbatim}
}

\icml{

}

\usepackage{multicol}
\usepackage{colortbl}
\usepackage{setspace}
\usepackage{transparent}
\usepackage{upgreek}

\usepackage{inconsolata}
\usepackage[scaled=.90]{helvet}
\usepackage{xspace}

\icml{
\setlist[enumerate]{leftmargin=*}
\setlist[itemize]{leftmargin=*}
}

\usepackage{graphicx}
\icml{\usepackage{subfigure}}

\usepackage[suppress]{color-edits}
\addauthor{fc}{blue}
\addauthor{sc}{orange}
\addauthor{sr}{red}

\let\oldparagraph\paragraph

\renewcommand{\paragraph}[1]{\oldparagraph{#1.}}

\makeatletter
\g@addto@macro\appendix{%
  \crefalias{section}{appendixsection}%
  \crefalias{subsection}{appendixsubsection}%
  \crefalias{subsubsection}{appendixsubsubsection}%
}
\makeatother
\crefname{appendixsection}{Appendix}{Appendices}
\Crefname{appendixsection}{Appendix}{Appendices}
\crefname{appendixsubsection}{Appendix}{Appendices}
\Crefname{appendixsubsection}{Appendix}{Appendices}
\crefname{appendixsubsubsection}{Appendix}{Appendices}
\Crefname{appendixsubsubsection}{Appendix}{Appendices}

\begin{document}

\icml{
\twocolumn[
  \icmltitle{High-accuracy sampling for diffusion models and log-concave distributions}

  \icmlsetsymbol{equal}{*}

  \begin{icmlauthorlist}
    \icmlauthor{Firstname1 Lastname1}{equal,yyy}
    \icmlauthor{Firstname2 Lastname2}{equal,yyy,comp}
    \icmlauthor{Firstname3 Lastname3}{comp}
    \icmlauthor{Firstname4 Lastname4}{sch}
    \icmlauthor{Firstname5 Lastname5}{yyy}
    \icmlauthor{Firstname6 Lastname6}{sch,yyy,comp}
    \icmlauthor{Firstname7 Lastname7}{comp}
    \icmlauthor{Firstname8 Lastname8}{sch}
    \icmlauthor{Firstname8 Lastname8}{yyy,comp}
  \end{icmlauthorlist}

  \icmlaffiliation{yyy}{Department of XXX, University of YYY, Location, Country}
  \icmlaffiliation{comp}{Company Name, Location, Country}
  \icmlaffiliation{sch}{School of ZZZ, Institute of WWW, Location, Country}

  \icmlcorrespondingauthor{Firstname1 Lastname1}{first1.last1@xxx.edu}
  \icmlcorrespondingauthor{Firstname2 Lastname2}{first2.last2@www.uk}

  \icmlkeywords{Machine Learning, ICML}

  \vskip 0.3in
]

\printAffiliationsAndNotice{}  %
}

\arxiv{
  \title{High-accuracy sampling for diffusion models and log-concave distributions}
\author{
  Fan Chen \\ {\small MIT} \\ {\small \texttt{fanchen@mit.edu}} 
  \and Sinho Chewi \\ {\small Yale University} \\ {\small \texttt{sinho.chewi@yale.edu}} 
  \and Constantinos Daskalakis \\ {\small MIT} \\ {\small \texttt{costis@csail.mit.edu}}
  \and Alexander Rakhlin \\ {\small MIT} \\ {\small \texttt{rakhlin@mit.edu}}
}
  \maketitle
}

\begin{abstract}
We present algorithms for diffusion model sampling which obtain $\delta$-error in $\mathrm{polylog}(1/\delta)$ steps, given access to $\widetilde O(\delta)$-accurate score estimates in $L^2$. This is an exponential improvement over all previous results.
Specifically, under minimal data assumptions, the complexity is $\widetilde O(d_\star\,\mathrm{polylog}(1/\delta))$ where $d_\star$ is the intrinsic dimension of the data. Further, under a non-uniform $L$-Lipschitz condition, the complexity reduces to $\widetilde O(L\,\mathrm{polylog}(1/\delta))$. 
Our approach also yields the first $\mathrm{polylog}(1/\delta)$ complexity sampler for general log-concave distributions using only gradient evaluations.
\end{abstract}

\section{Introduction}

What is the complexity of sampling from a continuous probability distribution, given access to evaluations of the gradient of the log-density?
In particular, can one design algorithms whose iteration complexity scales as $\polylog(1/\delta)$, where $\delta$ is the target accuracy, or must they necessarily take $\poly(1/\delta)$ steps?

Complexity bounds which scale as $\polylog(1/\delta)$, indicating that algorithms converge exponentially fast, are known as ``high-accuracy'' guarantees, as they ensure that one can draw an extremely accurate sample without too many iterations.
When access to the log-density of the target distribution is available, and not just its gradient, then high-accuracy samplers abound, based on accept-reject mechanisms such as rejection sampling or the Metropolis--Hastings filter.

Without density evaluations, the answer is less clear.
In this setting, existing sampling methods are typically based on discretizations of stochastic differential equations, and the need to control the discretization error precludes high-accuracy guarantees.
(Note that this is unlike the setting of optimization, in which gradient descent enjoys a high-accuracy guarantee under strong convexity---there, discretization does not bias the algorithm.)
A notable exception is discretization of piecewise deterministic Markov processes (PDMPs), for which~\citet{LuWan22Zigzag} show that $\polylog(1/\delta)$ evaluations of the gradient suffice~\fcedit{from a warm start}.

This question has become particularly interesting in light of recent developments on diffusion-based generative modeling.
Such models are based on implementing a certain reverse Markov process, in which each iteration requires evaluation of a score function---the gradient of the log-density along a diffusion process.
Since these models only learn the score function and not the density function itself (unlike earlier approaches such as energy-based models), it is natural to ask what the best achievable complexity is using score evaluations alone.

More precisely, we are interested in bounding the number of steps (and queries to the score function) required to sample from a distribution $\phat$ such that
\begin{align}\label{eq:goal}
  D(\pdata,\phat)\leq \delta+ \Capx\cdot  \epssc\,,
\end{align}
where $D$ is a measure of discrepancy, $\delta\in(0,1)$ is the target accuracy, $\epssc$ measures the error in the score function estimates, and $\Capx$ is a target approximation factor.

The initial work of~\citet{Chen+23SGM, LeeLuTan23SGMGeneral} showed that with $L^2$-accurate score estimates, Denoising Diffusion Probabilistic Models (DDPMs) achieve a query complexity of $1/\delta^2$ in total variation distance for sampling from an early stopped distribution, provided that $\pdata$ has bounded support, with an overhead that is polynomial in the dimension and radius, and with $\Capx = \wt O(1)$.
By a standard argument, this can be converted to guarantees for sampling from $\pdata$ itself in a weaker metric.
Importantly, these results impose almost no assumptions on the data distribution, providing strong theoretical justification for diffusion models.

Since then, there has been an explosion of works aimed at extending and refining these guarantees~\citep[e.g.,][]{chen2023improved,li2023towards,Ben+24Diffusion,gao2025wasserstein,ConDurSil25KLDiffusion, LiYan25Diffusion,Chen+23FlowODE,li2024sharp, GaoZhu25PFODE, huang2025convergence, JaiZha26SharpKL}, which we cannot fully survey here.
However, it is known that a complexity of $\Omega(1/\delta)$ is unimprovable for DDPM~\citep{jiao2025optimal}, which motivates changing the algorithm.
Under minimal assumptions, the algorithm of~\citet{li2024provable} achieves a query complexity of $1/\delta^{1/2}$.
Other works achieved similar speed-ups assuming that the Jacobian of the score error is bounded~\citep{li2024accelerating,li2025faster}.

A recent line of works~\citep[e.g.,][]{huang2025convergence, huang2025fast, li2025faster} studied higher-order discretization methods, with the resulting query complexity scaling roughly as $\CP\, d^{1+1/p}/\delta^{1/p}$,
where $p\geq 1$ is the chosen acceleration order, and $\CP$ hides dependence on other problem parameters\footnote{For example, in the results of \citet{huang2024reverse,huang2025convergence,huang2025fast}, $\CP$ hides polynomial dependence on higher-order Lipschitz parameters or the diameter of the support of $\pdata$.} 
along with an implicit (often exponential) dependence on $p$. These results show that diffusion sampling can be performed with sub-polynomial dependence on $1/\delta$,\footnote{
  Roughly, we should expect $\CP\geq e^{\Omega(p)}$. Hence, the complexity is at least $d\exp(O(\sqrt{\log(d/\delta)}))$ with the best choice of acceleration order $p$, which is better than $\poly(1/\delta)$ but much worse than $\poly\log(1/\delta)$.
}
but they still fall well short of the desired poly-logarithmic complexity.\footnote{Furthermore, the approximation factor $\Capx$ in \citet{huang2025convergence,huang2025fast} scales polynomially in the early stopping time, and \citet{li2025faster} additionally relies on the second-order score error $\epsjac$, making the results sensitive to score errors.}

Meanwhile, \citet{huang2024reverse, wainwright2025score} propose to use density evaluations to achieve high-accuracy samplers, and~\citet{Hua+25QuantizedDiffusion} propose a high-accuracy sampler which requires learning a certain ``quantized'' score.
As discussed above, these are not compatible with the current practice of diffusion sampling, which only learns score estimates.

Therefore, we state our main question of interest:
\begin{center}
    \emph{Is there a diffusion model sampler which achieves a high-accuracy guarantee using only score evaluations, under minimal assumptions on both the data distribution and the score error?}
\end{center}

\subsection{Our contribution}

We answer this question affirmatively via a new meta-algorithm, which we call \emph{first-order rejection sampling} (\AlgWSC); see \cref{alg:fors}.
This algorithm aims at simulating rejection sampling using only first-order (gradient) queries.
We give consequences of our method for diffusion sampling and for log-concave sampling.

\paragraph{Diffusion sampling}
We show the following results, where the error is measured in the bounded Lipschitz metric~\eqref{def:BL-metric}, $\epssc$ denotes the $L^2$-error of the score estimates (\cref{def:score-error}), and $\dstar$ is the \emph{intrinsic} dimension of the data distribution (\cref{def:intrinsic}), which is always bounded by the embedding dimension $d$.
\begin{itemize}
    \item Under \textbf{minimal data assumptions}---namely, $\pdata$ has a finite second moment $\mathsf M_2^2$---we obtain $\delta$ error in $O(\dstar\log^3((d+\mathsf M_2^2)/\delta))$ queries with $\Capx=O(1)$.
\end{itemize}
This result strictly improves upon all prior results in the literature in this setting~\citep[e.g.,][]{chen2023improved,Ben+24Diffusion,ConDurSil25KLDiffusion,azangulov2024convergence, li2024adapting, li2025dimension, liang2025low, PotAzaDel25Manifold, tang2025adaptivity, HuaWeiChe26LowDim}.

Beyond this setting, many works in the literature aim to sample with a number of steps which is sublinear in the dimension~\citep{Chen+23FlowODE, jiao2024instance,jiao2025optimal, zhang2025sublinear}.
We also incorporate these advances into our framework. \fcedit{Specifically, we show that under a non-uniform $L$-Lipschitz condition (with respect to the \emph{Frobenius norm}, \cref{asmp:Lip-Frob}--\ref{asmp:Lip-DDPM}), we obtain $\delta$ error in $O(L\log^3((d+\mathsf M_2^2)/\delta))$ steps, with $\Capx = O(1)$. This result is ``almost dimension-free'' and it also recovers our previous result, as the non-uniform $L$-Lipschitz condition always holds with $L\leq \wt{O}(\dstar)$. Further, it also improves upon prior results in the setting with a non-uniform $L$-Lipschitz condition under operator norm (\cref{asmp:non-unif-Lip}), as the implied complexity is $\widetilde O(\min\{d_\star^{2/3}L^{1/3}, \sqrt{dL}\})$.}

\paragraph{Log-concave sampling}
For sampling from log-concave (and isoperimetric) densities with gradient evaluations of the log-density, we recover state-of-the-art results from~\citet{FanYuaChe23ImprovedProx}, except that \emph{we do not require density evaluations}; see \cref{sec:log_concave}.

\subsection{Related work}

As we do not have space to survey the vast literature on diffusion model guarantees, we focus on the ones most relevant to our work.

\paragraph{High-accuracy diffusion sampling}
Prior to this work, \citet{huang2024reverse,wainwright2025score} studied diffusion sampling with additional \emph{zeroth-order} queries. Specifically, given access to estimates of the unnormalized log-densities with bounded error, query complexity bounds of $\tO(d^2\log(1/\delta))$~\citep{huang2024reverse} and $\tO(\sqrt{d}\log^3(1/\delta))$~\citep{wainwright2025score} were shown under additional Lipschitz conditions on the score functions. Both works resort to standard high-accuracy sampling methods (e.g., Metropolis-adjusted Langevin). However, estimating the densities can be extremely challenging in practice. Moreover,~\citet{Hua+25QuantizedDiffusion} propose a method based on learning a quantized score, which requires changing the score matching objective.

\paragraph{Concurrent work} In concurrent work,~\citet{gatmiry2026high} also obtain a high-accuracy guarantee for diffusion sampling.
\fcedit{
  Assuming that $\pdata$ is a Gaussian convolution, i.e., $\pdata = p_\star * \normal{0,\sigma^2 I}$ where $p_\star$ is supported on a ball of radius $R$, they obtain a query complexity of $\wt O((R/\sigma)^2 \log^2(1/\delta))$ assuming that the score errors are bounded in the \emph{sub-exponential} norm.
  Our results also imply a query complexity of $\wt O((R/\sigma)^2 \log^3(1/\delta))$ through the intrinsic dimension (\cref{ssec:intrinsic}) while only requiring $L_2$-score errors.}
We defer a lengthier comparison to
\cref{sec:concurrent}.

\paragraph{Log-concave sampling}
Our application to log-concave sampling is based on the proximal sampler~\citep{LeeSheTia21RGO, Chen+22ProxSampler}, and our results can be compared to~\citet{AltChe24Warm, FanYuaChe23ImprovedProx}.
See the book draft~\citet{Chewi26Book} for an overview of the subject.

\subsection{Notation}

We make use of the following divergences between probability distributions: the \emph{total variation} (TV) distance $\Dtv{\mu}{\nu} \deq \frac{1}{2}\int|\mu-\nu|$; the \emph{Hellinger distance} $\Dhels{\mu}{\nu} \deq \frac12\int (\sqrt\mu-\sqrt\nu)^2$; the \emph{KL divergence} $\Dkl{\mu}{\nu} \deq \En_\mu\log(\mu/\nu)$; the \emph{chi-squared divergence} $\Dchis{\mu}{\nu} \deq \E_\mu (\mu/\nu)^2 - 1$; and the \emph{Wasserstein distance} $W_2^2(\mu,\nu) \deq \inf_{\gamma~\text{coupling of }\mu,\nu} \En_{(X,Y)\sim\gamma}\|X-Y\|^2$.
In addition, we consider the bounded Lipschitz metric which metrizes weak convergence:
\begin{align}
\label{def:BL-metric}
    \Dbl{\mu}{\nu} \deq \sup\{\E_\mu f - \E_\nu f : |f| \le 1,\,\|f\|_{\rm Lip} \le 1\}.
\end{align}
We define
\begin{align}\notag
    \pclip[\B]{x}\deq \max\crl{-B,\min\crl{B,x}}\,, \qquad \forall x\in\RR\,.
\end{align}
We use $\leqsim$ and $O(\cdot)$ to hide absolute constants, i.e., $f\leqsim g$ (and $f=O(g)$) if there is an absolute constant such that $f\leq Cg$.
The notation $\wt O(\cdot)$ hides logarithmic factors.

\section{Background on diffusion models}

Recall that Denoising Diffusion Probabilistic Models (DDPMs) are based on the following forward process that transforms a sample $X_0\sim \pdata$ to noise:
\begin{align}\label{eq:forward}
    X_0\sim \pdata\,, \quad 
    X_{k+1}\sim \normal{\alp_k X_k, \alp_k^2 \eta_k\Id}\,,~~k\in[K]\,.
\end{align}
Note that this is a Markov chain, and it is easy to see that $X_k\mid X_0\sim \normal{\alpbar_k X_0, \sigma_k^2\Id}$, where 
\begin{align}\notag
    \alpbar_k=\prod_{s=0}^{k-1} \alp_s\,, \qquad \sigma_k^2=\alp_{k-1}^2\,(\sigma_{k-1}^2+\eta_{k-1})\,.
\end{align}
Conversely, given the parameter sequence $(\alpbar_k, \sigma_k)_{k\in[K]}$, the corresponding $(\alpha_k, \eta_k)_{k\in[K]}$ is given by
\begin{align}\notag
    \alpha_k=\frac{\alpbar_{k+1}}{\alpbar_k}\,, \qquad
    \eta_k=\frac{\sigma_{k+1}^2}{\alpha_k^2}-\sigma_k^2
    =\frac{\sigma_{k+1}^2}{\alpbar_{k+1}^2}\,\alpbar_k^2-\sigma_k^2\,.
\end{align}
We say that the DDPM is \emph{variance-preserving} if $\alpbar_k^2+\sigma_k^2=1$, and hence $\alp_k^2=\frac{1-\sigma_{k+1}^2}{1-\sigma_k^2}$, $\sigma_{k+1}^2=\frac{1-\sigma_{k+1}^2}{1-\sigma_k^2}\,(\sigma_{k}^2+\eta_{k})$. The DDPM is \emph{variance-exploding} if $\alpbar_k\equiv 1$, and hence $\sigma_{k+1}^2=\sigma_{k}^2+\eta_{k}$.

We let $p_k$ be the probability density function of $X_k$, and let $\rho_k(\cdot\mid x')$ be the probability density function of the backward transition kernel $\PP(X_k=\cdot\mid X_{k+1}=x')$. Then, by Bayes rule, it holds that
\begin{align}\label{eq:backward_kernel}
    \rho_k(x\mid x')\propto_x p_k(x) \exp\prn[\Big]{-\frac{\nrm{x-\alp_k^{-1}x'}^2}{2\eta_k}}\,.
\end{align} 
In this work, we focus on the task of sampling from $p_1$ with extremely small values of $1-\alpha_0$ and $\sigma_1$. 
This corresponds to \emph{early stopping}; see \cref{sec:diffusion} for further discussion. %

\paragraph{Score function and estimates}
As above, we assume that we have access to approximate score functions $(\scf_k)_{k\in[K]}$ such that $\scf_k\approx \scfs_k\ldef \nabla \log p_k$ with controlled mean-squared error. Recall that by Tweedie's identity,
\begin{align}
    \scfs_k(x)=\nabla \log p_k(x)=\frac{1}{\sigma_k^2}\En[\alpbar_k X_0-X_k\mid X_k=x]\,,
\end{align}
where the conditional expectation is taken over $X_0\sim \pdata$, $X_k\sim \normal{\alpbar_k X_0, \sigma_k^2\Id}$.
\fcedit{We define $\Ds_k(x)\deq \En[X_0\mid X_k=x]$ to be the posterior mean, and let $\Dn_k(x)\deq \alpbar_k^{-1}(x+\sigma_k^2\scf_k(x))$ be the denoising function corresponding to the score function $\scf$.}

\begin{definition}[Score estimation error]\label{def:score-error}
Given the score estimates $(\scf_k)_{k\in [K]}$, we define
\begin{align}\notag
    \epssct[k]^2\ldef \En_{X_k\sim p_k}\nrm{\scf_k(X_k)-\scfs_k(X_k)}^2\,, \qquad k\in[K]\,.
\end{align}
\end{definition}

\paragraph{Road map}
To approximately generate a sample from $\pdata$, we learn approximate score functions $\scf_k\approx \nabla \log p_k$.
DDPM uses the approximation 
\begin{align}\label{eq:DDPM-approx}
    \rho_k(\cdot\mid x') \approx \normal{\alpha_k^{-1} x' + \eta_k \alpha_k \scf_{k+1}(x'), \eta_k' \Id}.
\end{align}
In our work, we develop algorithms for sampling from~\eqref{eq:backward_kernel} directly. In \cref{sec:gaussian_tilts}, we motivate our approach by assuming that $\scf_k=\nabla \log p_k$ is \emph{exact}, which has implications for log-concave sampling described in \cref{sec:log_concave}. In \cref{sec:diffusion}, we instantiate our methods on diffusion sampling and show that the convergence guarantees are \emph{robust} with respect to the error in the score function estimates.

\section{Key subroutine: Gaussian tilts}\label{sec:gaussian_tilts}

In this section, we focus on the problem of sampling from a Gaussian tilt of the form
\begin{align}\label{eq:RGO}
    \rgo(x)\propto \exp\prn[\Big]{-f(x)-\frac{\nrm{x-\xz}^2}{2\eta}}\,,
\end{align}
where we assume access to first-order queries for $f$. For our eventual application to diffusion models, we will choose $f=\scedit{-}\log p_t$ in view of \cref{eq:backward_kernel} and approximate $\scedit{-}\nabla f\approx \scf_t$. 

\subsection{First-order rejection sampling (FORS)}\label{sec:motivation}

We motivate our approach by considering the simple problem of sampling from a density $p\propto e^{-f}$, where $f : [0,1]\to\R$, $f(0) = 0$, and $-1 \le f' \le 1$.
Our goal is to develop a \emph{high-accuracy sampler}---a sampler whose sampled distribution has $\delta$ error in total variation distance to $p$---in $\polylog(1/\delta)$ steps, using only queries to $f'$.

Consider performing rejection sampling with the base measure $\unif([0,1])$. To do so, we must generate randomness $b\sim \Ber(ce^{-f(x)})$ for any given $x\in[0,1]$. However, we only have access to $f'$. The identity $f(x)=\int_{0}^x f'(y)\,dy$, written as
\begin{align}
\label{eq:1d_representation}
    f(x)=\En_{y\sim \unif([0,x])}[xf'(y)]\,,
\end{align}
suggests an unbiased estimate of $f(x)$. Is it possible to sample from $\Ber(ce^{-f(x)})$ if we have access to an unbiased estimate of $f(x)$?
A more general version of this idea is known as the ``Bernoulli factory'' problem \citep{keane1994bernoulli,nacu2005fast}. 
It can be stated as the following abstract task:
\begin{center}
    \textbf{Task:} Given i.i.d.\ random variables $W_1,W_2,W_3,\dotsc$ in $[-1,1]$, generate a sample $b\sim \Ber(ce^{\En W_1})$.
\end{center}
To solve this, we write the Taylor series as
\begin{align}\notag
    e^{\En W_1}=e^{-1}\cdot e^{\En[1+W_1]}
    =\sum_{j\geq 0} \frac{e^{-1}}{j!}\, \bigl(\En[1+W_1]\bigr)^j\,.
\end{align}
Suppose that $J\sim \Poi(2)$ is independent of the i.i.d. sequence $W_1,W_2,W_3,\dotsc$. Then we notice that
\begin{align}\notag
    e^{\En W_1}=e\En\brk[\Big]{\prod_{j=1}^J \bigl(\frac{1+W_j}{2}\bigr)}\,,
\end{align}
and simply set $b\sim \Ber\prn[\big]{\prod_{j=1}^J \bigl(\frac{1+W_j}{2}\bigr)}$.
Indeed, $\Pr(b=1) = \E\prod_{j=1}^J \bigl(\frac{1+W_j}{2}\bigr) = e^{-1+\En W_1}$.

In summary, we can generate a sample $b\sim \Ber\prn{ce^{-f(x)}}$ \emph{without} having to compute or accurately approximate $f(x)$ (this requires the integration step $f(x)=\int_{0}^x f'(y)\,dy$ and can be expensive). Instead, it is sufficient to have access to (a random number of) unbiased estimates of $f(x)$. The latter can be achieved with derivative information, in view of \eqref{eq:1d_representation}.

We now generalize this setup via the following meta-algorithm, called \emph{first-order rejection sampling} (\emph{FORS}). Given a proposal distribution $q$ and a tilt function $w$, the goal of \cref{alg:fors} is to produce a sample from $\widehat p(x)\propto q(x)\, e^{w(x)}$ without having access to the \emph{value} $w(x)$. Instead, for each $x\in\RR^d$, we can generate i.i.d.\ samples $W_1,W_2, W_3\dotsc$ such that $\En[W_1\mid x]=w(x)$. 
Let $\mathcal W_x$ denote the conditional distribution of $W_1$ given $x$.

\begin{algorithm}
\caption{First-order rejection sampling (FORS)}\label{alg:fors}
\begin{algorithmic}
\STATE \textbf{Input:} Parameter $B > 0$, proposal distribution $q$ over $\R^d$, estimator distributions $(\cW_x)_{x\in\R^d}$ supported on $[-B,B]$
\FOR{$i=1,2,3,\dotsc$}
\STATE Sample $x\sim q$.
\STATE Sample $J\sim \Poi(2\B)$.
\STATE Sample i.i.d. $W_1,\dotsc,W_J \sim \mathcal W_x$.
\STATE Output $x$ with probability $\prod_{j=1}^J \frac{B+W_j}{2B}$.
\ENDFOR
\end{algorithmic}
\end{algorithm}

\begin{theorem}[FORS guarantee]\label{thm:fors}
    \cref{alg:fors} outputs a random point with density $\widehat p(x)\propto q(x)\, e^{\En[W_1\mid x]}$.
    The number of sampled $W_j$'s is bounded, with probability at least $1-\delta$, by $3Be^{2B}\log(2/\delta)$.

    Moreover, if~\cref{alg:fors} is called $T$ times, then with probability at least $1-\delta$, the total number of sampled $W_j$'s is $O(Be^{2B}\,(T+\log(1/\delta)))$.
\end{theorem}

\fcedit{We remark that variants of this idea have been applied to \emph{exactly} simulate SDEs~\citep[e.g.,][]{Wag1988SDE,beskos2005exact,beskos2006retrospective, Pap11Diffusion}.}

\subsection{Sampling from Gaussian tilts with an exact oracle}

Now, we return to the problem of sampling from a general Gaussian tilt:
\begin{align*}
    \rgo(x)\propto \exp\prn[\Big]{-f(x)-\frac{\nrm{x-\xz}^2}{2\eta}}\,.
\end{align*}
\fcedit{From \cref{sec:motivation}, it suffices to construct a proposal distribution $q$ and a tilt function $w$ such that (a) $\nu(x)\propto q(x)\cdot e^{w(x)}$, and (b) a bounded estimator for $w(x)$ can be constructed from $\nabla f$.

The condition (a) is equivalent to
\begin{align*}
    w(x)=-f(x)-\frac{\nrm{x-\xz}^2}{2\eta}- \log q(x)+\const.
\end{align*}
To ensure both (a) and (b), the natural idea is to choose $q$ as a Gaussian approximation to $\nu$ obtained via a first-order expansion of $f$.
Concretely, by Taylor's expansion, for a fixed  $\xp\in\RR^d$, we have
\begin{align*}
    f(x)\approx f(\xp)+\tri{x-\xp, \nabla f(\xp)}\,,
\end{align*}
and hence it is natural to choose $q=\normal{\xz-\eta \nabla f(\xp),\eta\Id}$ so that
\begin{align*}
    q(x) \propto \exp\Bigl(-\tri{x-\xp, \nabla f(\xp)} - \frac{1}{2\eta}\norm{x-x_0}^2\Bigr)\,.
\end{align*}
}
Then, we can express
\begin{align*}
    \icml{&~}\log \rgo(x)-\log q(x)-\const \icml{\\}
    =&~\tri{x-\xp, \nabla f(\xp)}-f(x) + f(x_+) \\
    =&~\int_{0}^1 \tri{x-\xp, \nabla f(\xp)-\nabla f(\lr x+(1-\lr)\xp)}\,d\lr\,.
\end{align*}
In particular, we can set
\begin{align*}
    W_{r,x}\ldef \tri{x-\xp, \nabla f(\xp)-\nabla f(\lr x+(1-\lr)\xp)}\,,
\end{align*}
so that
\begin{align*}
    \rgo(x)\propto q(x)\exp\prn*{ \En_{r} W_{r,x}}\,,
\end{align*}
where the expectation $\En_r[\cdot]$ is taken over $r\sim \unif([0,1])$. 

Further, assuming that $\nabla f$ is $\beta$-Lipschitz, we can bound $\abs{W_{r,x}}\leq \beta\,\nrm{x-\xp}^2$. Therefore, as long as $\xp\approx \xz-\eta\nabla f(\xp)$ and $\beta d\eta\ll 1$, we can guarantee that $\sup_{r}{\abs{W_{r,x}}}\leq 1$ with high probability over $x\sim q$. This implies the following distribution is a good approximation of $\rgo$:
\begin{align*}
    \wh{\rgo}(x)\propto q(x)\exp\prn{\En_{r} \hW_{r,x}}\,,
\end{align*}
where $\hW_{r,x}=\pclip{W_{r,x}}$ and $B=\Theta(1)$ is a tuneable parameter.

\paragraph{General path integral}
We can generalize the above pipeline with any \emph{path integral} and demonstrate (a) that in the gradient Lipschitz case, we in fact only need to ensure $\beta\sqrt{d} \eta\ll 1$, and (b) the Lipschitz continuity of $\nabla f$ can be weakened to H\"older continuity for any exponent $s \in [0,1]$.

Fix a distribution $P$ (to be determined later) over $\RR^d$ and a \emph{path function} $\gam{\lr}\ldef \gamma(x;z,r)$ such that $\gam{1}=x$ and $\gam{0}=\wb{\gamma}(z)$ is independent of $x$. Then we can express any smooth function $h:\RR^d\to \RR$ as
\begin{align}\label{eq:path-integral}
\begin{aligned}
    \icml{&~}h(x)-\En_{z\sim P}[h(\wb{\gamma}(z))] \icml{\\}
    =&~\int_{0}^1 \En_{z\sim P}\tri{\gamp{\lr}, \nabla h(\gam{\lr})}\,d\lr \\
    =&~\En_{z\sim P,\,r\sim \unif([0,1])}\tri{\gamp{\lr}, \nabla h(\gam{\lr})}\,,
\end{aligned}
\end{align}
where $\gamp{\lr}=\frac{d}{d\lr}\gam{\lr}$ is the path derivative with respect to $\lr$. 
This inspires us to consider
\begin{align}\label{eq:Wrzx}
    W_{r,z,x}\ldef \tri{\gamp{\lr}, \nabla f(\xp)- \nabla f(\gam{\lr})}\,,
\end{align}
so that
\begin{align*}
    \rgo(x)\propto q(x)\exp\prn*{ \En_{r,z} W_{r,z,x}}\,,
\end{align*}
where the expectation $\En_{r,z}[\cdot]$ is taken over $r\sim \unif([0,1])$ and $z\sim P$. 
Again, we will truncate $W_{r,z,x}$ to ensure that the estimator lies in $[-B,B]$.

For better dimension dependence, we consider the following path function ($\xhat=\xz-\eta\nabla f(x_+)$): 
\begin{align}\label{eq:def-path}
\begin{aligned}
    &~\gam{\lr}=\alr x + (1-\alr) \xhat + \blr z\,, \\
    &~\alr=\sin(\pi\lr/2)\,, ~~\blr=\cos(\pi\lr/2)\,,
\end{aligned}
\end{align}
so that $\gamp{\lr}=\alrp (x -\xhat) + \blrp z$. 

\begin{assumption}\label{ass:holder}
    There exists $s \in [0,1]$ and $\beta_s \ge 0$ such that $\norm{\nabla f(x) -\nabla f(y)} \le \beta_s\norm{x-y}^s$ for all $x,y\in\R^d$.
\end{assumption}

\begin{theorem}\label{thm:gaussian_tilt}
    Suppose that~\cref{ass:holder} holds, $B=\Theta(1)$, and $\nrm{\xz-\eta\nabla f(x_+) - x_+} \leq (d\eta)^{1/2}$.
    
    Consider instantiating~\cref{alg:fors} with the choices $q = \normal{\xz-\eta \nabla f(\xp),\eta\Id}$, and $\mathcal W_x$ the law of $\pclip{W_{r,z,x}}$, where $W_{r,z,x}$ is defined in~\eqref{eq:Wrzx}-\eqref{eq:def-path} and $z\sim \normal{0,\eta\Id}$, $r\sim \unif([0,1])$. 
    Then, the law $\wh\nu$ of~\cref{alg:fors} satisfies $\Dchis{\nu}{\wh\nu} \le \delta^2$, provided that
    \begin{align*}
        \eta^{-1} \gg \prn[\Big]{ \beta_s^2d^s\log(1/\delta)+\frac{s\beta_s^2}{d^{1-s}}\log^2(1/\delta) }^{1/(1+s)}.
    \end{align*}
\end{theorem}

Note that $x_0 - \eta\nabla f(x_+) - x_+ = 0$ when $x_+ = \text{prox}_{\eta f}(x_0)$, that is, the requirement of~\cref{thm:gaussian_tilt} is that we can approximately take a proximal step on $f$ from $x_0$.

\cref{thm:gaussian_tilt} interpolates between the Lipschitz case $s=0$, which requires $\eta^{-1} \gg \beta_0^2 \log(1/\delta)$, and the smooth case $s=1$, which requires $\eta^{-1} \gg \beta_1\, d^{1/2} \log(1/\delta)$.
This recovers the result of~\citet{FanYuaChe23ImprovedProx}, except that we only use first-order queries; we spell out the implications for log-concave sampling in \cref{sec:log_concave}.

\section{Diffusion sampling}\label{sec:diffusion}

A key insight is that sampling from the backward transition kernel~\eqref{eq:backward_kernel}, given access to a score estimate $\scf_k$, is a special case of the setup of Section~\ref{sec:gaussian_tilts}, except that we only have approximately correct first-order evaluations.
This leads to the template~\cref{alg:DM} for high-accuracy diffusion sampling. Throughout this section, we adopt the following proposal distribution $\rhobar$:
\begin{align}\label{def:DDPM-proposal}
    \rhobar_k(\cdot\mid{}\xr{k+1})=\normal{ \alp_k^{-1} \xr{k+1} + \alpha_k \eta_k \scf_{k+1}(\xr{k+1}), \etabar_k \Id },
\end{align}
where $\etabar_k$ is given by $1/\etabar_k=1/\eta_k+1/\sigma_k^2$. This corresponds to appropriately applying the \emph{exponential integrator} to the backward SDE. The reason we choose \cref{def:DDPM-proposal} as the proposal is detailed in \cref{thm:DDPM-KL-exact}, where we show that this choice is almost the minimizer of the KL divergence to the true transition distribution $\rho_k$.

In the following, we show how to choose the ``corrector'' distributions $\mathcal W_x^k$ for $k\in [K]$.

\begin{algorithm}[H]
  \caption{Backward diffusion sampling}\label{alg:DM}
\begin{algorithmic}
\STATE \textbf{Input:} Score estimates $\crl{\scf_k}_{k\in [K]}$, initial distribution $\phat_K$, parameters $(\alp_k,\eta_k)_{k\in [K]}$
\STATE Sample $\xr{K}\sim \phat_K$.
\FOR{$k=K-1,\dotsc,1$}
\STATE Sample $\xr{k}\leftarrow \AlgWSC(B, q_k, (\mathcal W^k_x)_{x\in\R^d})$.
\ENDFOR
\STATE \textbf{Output:} $\xr{1}$
\end{algorithmic}
\end{algorithm}

We state all our guarantees for sampling from the early stopped distribution $p_1$, the law of $X_1 \sim \normal{\alpha_0 X_0, \sigma_{\scedit{0}}^2 \Id}$ with $X_0 \sim \pdata$; see the discussion after the theorem. 

\subsection{Intrinsic dimension}\label{ssec:intrinsic}

Our result will be based on the following notion of \emph{intrinsic dimension} of the data distribution $\pdata$~\citep{li2024adapting}.

\begin{definition}[Intrinsic dimension]\label{def:intrinsic}
For any distribution $p$ and $r\geq 0$, let $N(p;r)$ be the $r$-covering number of $\supp(p)$ under Euclidean norm $\nrm{\cdot}$. Define the \emph{intrinsic dimension} of $p$ as
\begin{align*}
    \dim_{\sigma^2}(p)\deq 1\vee \inf_{r\geq 0} \prn[\Big]{ \log N(p;r)+\frac{r^2}{\sigma^2} }\wedge d.
\end{align*}
We denote $\dstar\deq \dim_{\sigma_0^2/\alpha_0^2}(\pdata)$ to be the intrinsic dimension of the data distribution (note that $\sigma_0^2/\alpha_0^2$ is the ``real'' variance of $X_1\sim \normal{\alpha_0X_0,\sigma_0^2\Id}$), and the intrinsic dimension $\dstar$ is no larger than the embedding dimension $d$. 
\end{definition}

\begin{example}[Low-dimensional manifold]
Suppose that the data distribution $\pdata$ is supported on a compact $k$-dimensional manifold $\cX\subset \RR^d$. Then it holds that $\log N(\pdata;r)\leqsim k\log(R/r)$ (where $R$ is the diameter of $\cX$) and hence $\dstar=\wt{O}(k)$.
\end{example}

We note that $\dstar$ also captures various ``dimension-free'' settings.\footnote{Formally, in the setting of \cite{li2025dimension,gatmiry2026high}, the data distribution is assumed to be of the form $p_\star*\normal{0,\sigma_\star^2 \Id}$, where $p_\star$ is either supported on $N$ points~\citep{li2025dimension} or the ball $B(R)$~\citep{gatmiry2026high}. Our results encompass such settings as we can regard $\pdata=p_\star$, $\alp_0=1$, and $\sigma_0=\sigma_\star$, and then \cref{thm:DM-intrinsic} provides guarantees for $p_1=p_\star*\normal{0,\sigma_\star^2}$.} For example, when the support of $\pdata$ has at most $N$ elements, then $\dstar\leq \log N$, and this is exactly the setting considered in \cite{li2025dimension}. Further, when the support of $\pdata$ is contained in a ball of radius $R$, it is clear that $\dstar\leq \frac{\alpha_0^2R^2}{\sigma_0^2}$, and this captures the setting of \citet{gatmiry2026high}.

\subsection{Algorithm}

Consider instantiating the subroutine $\AlgWSC$ in \cref{alg:DM} with the following choices: $\Xbar_k=\alpha_k^{-1} X_{k+1}+\alpha_k \eta_k \scf_{k+1}(X_{k+1})$, $q_k = \normal{\Xbar_k , \etabar_k I}$, and $\mc W_x^k$ is the law of
\begin{align*}
    \widehat W_{r,z,\xhat,x} \deq \mathsf{Clip}_B\bigl(\lambda_k \langle \dot \gamma_{z,r,\xhat}(x),&~\Dn_k(\gamma_{z,r,\xhat}(x)) - \Dn_{k+1}(X_{k+1})\rangle\bigr)\,,
\end{align*}
where $\lambda_k=\frac{\alpbar_k}{\sigma_k^2}$, $z\sim \normal{0,\frac12\etabar_k I}$, $\xhat \sim \normal{\Xbar_k,\frac12\etabar_k I}$, $r\sim\Unif([0,1])$, 
and $\gamma_{z,r,\Xbar_k}$ is the path function given by
\begin{align}
    &\gamma_{z,r,\xhat}(x) \deq\alr x + (1-\alr) \xhat + \blr z\,, \label{eq:path_fn_new_1} \\
    &\alr\deq \frac13\,\prn[\Big]{1+2\cos\prn[\big]{\frac{2\pi}{3}\,(1-r)}}\,, ~~\blr\deq \frac{2}{\sqrt{3}}\sin\prn[\big]{\frac{2\pi}{3}\,(1-r)}\,.\label{eq:path_fn_new_2}
\end{align}
The crucial property is that $a_0=b_1=0$ and $a_1=b_0=1$,
\begin{align}
    \alr^2+\frac{1}{2}\,(1-\alr)^2+\frac12\,\blr^2\equiv 1\,, \qquad \forall r\in[0,1]\,.
\end{align}

\begin{theorem}\label{thm:DM-intrinsic}
Suppose that $\delta\in(0,\frac12]$, $\B=\Theta(1)$, and for any $k\in [K]$, it holds that $\alpha_k^2\eta_k\scedit{\ll} \eta_{k+1}$ and
\begin{align}\label{eq:DM-intrinsic-eta}
    \frac{\sigma_k^2}{\eta_k} \gg \dstar\log(1/\delta)+\log^2(1/\delta)\,.
\end{align}

Consider instantiating the subroutine $\AlgWSC$ in \cref{alg:DM} as above.
Let $\phat_1$ be the law of $\xr{1}$ generated by \cref{alg:DM}. Then
\begin{align*}
    \Dkl{p_1}{\phat_1}\leqsim&~ \kl{p_K}{\phat_K}+K\delta 
    +\sum_{k=1}^{K} \eta_k \epssct^2\,,
\end{align*}
\end{theorem}

Next, we describe the implications of~\cref{thm:DM-intrinsic} for diffusion model sampling.
We should always choose $\B=\Theta(1)$ %
and we denote $G\deq C(\dstar+\log(K/\delta))\log(K/\delta)$ for a sufficiently large constant $C$ from \cref{eq:DM-intrinsic-eta}.
Then, in the variance-preserving setting:%
\begin{itemize}
    \item \scedit{Recall that} $\alp_k^2=\frac{1-\sigma_{k+1}^2}{1-\sigma_k^2}$ and $\sigma_{k+1}^2=\frac{1-\sigma_{k+1}^2}{1-\sigma_k^2}\,(\sigma_k^2+\eta_k)$.
    \item The condition $\eta_k\leq \frac{\sigma_k^2}{G}$ reduces to $\frac{\sigma_{k+1}^2}{1-\sigma_{k+1}^2}\leq \frac{\sigma_k^2}{1-\sigma_k^2}\cdot \prn*{1+\frac{1}{G}}$. 
    \item This implies that as long as $K\geq O(G\log(1/(\bar\delta \sigma_0^2)))$, we can guarantee that $1-\sigma_K^2\leq \bar\delta$.
    \item From known results on the convergence of the forward process, $1-\sigma_K^2 \le \bar\delta$ and $\phat_K=\normal{0,\sigma_K^2\Id}$ imply that $\kl{p_K}{\phat_K}\lesssim \bar\delta \En\nrm{X_0}^2$.
\end{itemize}
By the above calculation, we can show the following corollary. We denote $\mathsf M_2^2 \deq \E_{X_0\sim \pdata}\|X_0\|^2$ to be the second moment of $\pdata$.

\begin{corollary}\label{cor:poly-log}
In the variance-preserving setting, for any $\delta\in(0,\frac12]$, $\sigma_0>0$, there exists a schedule $(\sigma_k)_{k\in [K]}$ such that
\begin{align*}
    K\leq O\prn[\big]{(\dstar+\log(\kappa/\delta))\log^2(\dstar\kappa/\delta)}\,,
\end{align*}
where $\kappa\deq \mathsf M_2^2/\sigma_0^2+ 1$, 
and \cref{alg:DM} can be instantiated with $\phat_K=\normal{0,\sigma_K^2\Id}$, so that
\begin{align*}
    \kl{p_1}{\phat_1}\leqsim&~ \delta^2+\sum_{k=1}^{K} \eta_k \epssct^2\,.
\end{align*}
Further, the number of queries made by~\cref{alg:DM} is $O(K)$ \whp.
\end{corollary}

This governs the complexity of sampling from the early stopped distribution $p_1$.
    To convert this into a guarantee for sampling from $\pdata$ itself, we can note that $W_2^2(\pdata, p_1)\le (1-\alpha_0)^2\,\E\|X_0\|^2 + \sigma_0^2 d$.
    This can be made at most $\delta^2$ by choosing $\sigma_0^2 \asymp \delta^2/(d+\mathsf M_2^2)$.

In terms of the bounded Lipschitz metric, this implies that $\Dbl{\pdata}{\widehat p_1}^2 \lesssim \delta^2 + \sum_{k=1}^{K} \eta_k \epssct^2$ with a total complexity of %
\begin{align*}
    \boxed{\dstar\cdot \log^3\prn[\Big]{\frac{d+\mathsf M_2^2}{\delta^2}}}.
\end{align*}
Note that this guarantee imposes no assumptions on $\pdata$ beyond a second moment bound, and depends on $\pdata$ through the \emph{intrinsic} dimension $\dstar$ of $\pdata$ instead of the embedding dimension $d$.
This improves upon prior works~\citet{Ben+24Diffusion, ConDurSil25KLDiffusion} which achieved $\widetilde O(d/\delta^2)$ under these assumptions, and~\citet{LiYan25Diffusion, JaiZha26SharpKL} which recently improved the complexity to $\widetilde O(d/\delta)$.

Additionally, we can derive KL convergence to $\pdata$ assuming that the data distribution $\pdata$ is log-smooth with parameter $L$ (in this case necessarily $\dstar=d$). In this case, we consider generating $X_0\sim \normal{ \alp_0^{-1} \xr{1} + \alpha_0 \eta_0 \scf_{1}(\xr{1}), \eta_0 \Id }$ by one \emph{extra} DDPM step.\footnote{We do not apply $\AlgWSC$ so that we do not need access to $\scf_0\approx \nabla \log \pdata$.} Then, by \cref{cor:DDPM-p0}, we can choose $\sigma_0^2\asymp \delta/(dL)$ to ensure that $\Dkl{\pdata}{\widehat p_0} \lesssim \delta^2 + \sum_{k=1}^{K} \eta_k \epssct^2$ with a total complexity of
\begin{align*}
    \boxed{d\cdot \log^3\prn[\Big]{\frac{d+L+\mathsf M_2^2}{\delta^2}}}.
\end{align*}

\subsection{Refined analysis with non-uniform Lipschitz condition}

In the following, we provide a refined upper bound, showing that our algorithm in fact achieves $\sqrt{d}$-complexity under the \emph{non-uniform Lipschitz condition}, following~\citet{jiao2024instance, jiao2025optimal}.

From now on, we focus on the \emph{variance-exploding} setting, i.e., $\alp_k\equiv 1$ for all $k\in [K]$.\footnote{This is without loss of generality, because we can always rescale $\wt{X}_k=\frac{1}{\alpbar_k} X_k$.}
To describe the non-uniform Lipschitz condition, we consider the continuous process $q_\tau\deq \pdata*\normal{0,\tau\Id}$ (i.e., $p_k=q_{\sigma_k^2}$). We define $m_\tau(y)\deq \En[Y_0\mid Y_\tau=y]$ to be the conditional mean function (i.e., $\Ds_k(\cdot)=m_{\sigma_k^2}(\cdot)$), and $\nabla m_\tau(y)=\frac{1}{\tau}\Cov(Y_0\mid Y_\tau=y)$.

\begin{assumption}[Non-uniform Lipschitz condition]\label{asmp:non-unif-Lip}
For any $\delta>0$, there is a parameter $\Lipop\geq 1$ such that for any $\tau\geq \sigma_0^2$,
\begin{align}
    \bbP_{Y_\tau\sim q_\tau}\prn[\Big]{ \nrmop{\nabla m_\tau(Y_\tau)}> \Lipop }\leq \frac{\delta}{\dstar^5}\,.
\end{align}
\end{assumption}

We note that by \cref{cor:MGF-intrinsic}, \cref{asmp:non-unif-Lip} holds unconditionally with $\Lipop=O(\dstar+\log(1/\delta))$. \sccomment{$\log(\dstar/\delta)$?}\fccomment{$\log \dstar\leq \dstar$} In other words, the score function $\scfs_k$ is smooth ``with high probability''. Further, when $\pdata$ is log-concave, \cref{asmp:non-unif-Lip} holds with $\Lipop\equiv 1$ for any $\delta\geq 0$. 
As shown in~\citet{jiao2025optimal}, when $\pdata=\sum_{h=1}^H p_h\normal{\mu_h,\sigma_h^2}$ is a mixture of $H$ Gaussian distributions, it holds that $\Lipop\leq O(\log(H)\log(d/\delta))$ (see also \cref{prop:logp-adapt}).

\paragraph{Lipschitz condition under Frobenius norm}
In addition to \cref{asmp:non-unif-Lip}, to state our result in the most unified form, we introduce the following assumption in terms of Frobenius norm. We will discuss later how it is in fact implied by \cref{asmp:non-unif-Lip}. \fcedit{In \cref{ssec:why-Frob}, we show that controlling the Frobenius norm of $\nabla m_\tau(Y_\tau)$ is in fact \emph{necessary} for any sampling scheme based on Gaussian approximations of the backward kernel~\eqref{eq:backward_kernel}. }
\begin{assumption}[Non-uniform Lipschitz condition; Frobenius norm]\label{asmp:Lip-Frob}
For any $\delta>0$, there is a parameter $\LipF\geq 1$ such that for any $\tau\geq \sigma_0^2$,
\begin{align}
    \bbP_{Y_\tau\sim q_\tau}\prn[\Big]{ \nrmF{\nabla m_\tau(Y_\tau)}> \LipF }\leq \frac{\delta}{\dstar^5}\,.
\end{align}
\end{assumption}

\newcommand{\LF}{L_{\rm F}}
\newcommand{\Lop}{L_{\rm op}}

\begin{proposition}\label{prop:Lip-op-to-Frob}
Suppose that \cref{asmp:non-unif-Lip} holds with $\Lipop$. Then \cref{asmp:Lip-Frob} also holds with
\begin{align*}
    \LipF\leq C\sqrt{\Lipop[\delta/2]\,(\dstar+\log(1/\delta))}\,, \qquad \forall \delta\in(0,1)\,.
\end{align*}
\end{proposition}

For any $k\in [K]$, we define $\wt{p}_k$ to be the distribution of $X_k'$ generated by \emph{one-step DDPM with true score function}, i.e., $X_{k+1}\sim p_{k+1}$ and $X_k'\sim \normal{ \alp_k^{-1} \xr{k+1} + \alpha_k \eta_k \scfs_{k+1}(\xr{k+1}), \etabar_k \Id }$. The following assumption requires that the denoiser function $m_{\sigma_k^2}$ is also smooth with high probability under the distribution $\wt{p}_k$. 

\begin{assumption}\label{asmp:Lip-DDPM}
For any $\delta$, the parameter $\LipF$ in \cref{asmp:Lip-Frob} also 
satisfies
\begin{align}
    \bbP_{X_k'\sim \wt{p}_k}\prn[\Big]{ \nrmF{\nabla m_{\sigma_k^2}(X_k')}> \LipF }\leq \frac{\delta}{\dstar^5}\,.
\end{align}
\end{assumption}

\newcommand{\Ldel}{L_\delta}
\newcommand{\piz}[2][\xz]{\nu_{#1}(#2)}

\begin{theorem}\label{thm:DM-Lip}
Suppose that \cref{asmp:Lip-Frob} and \cref{asmp:Lip-DDPM} hold, $\delta\in(0,\frac12]$, $\B=\Theta(1)$, and
\begin{align}
    \frac{\sigma_k^2}{\eta_k}\gg \LipF\log(\dstar/\delta)+\log^2(1/\delta)\,.
\end{align}
Consider instantiating the subroutine $\AlgWSC$ in \cref{alg:DM} as above.
Let $\phat_1$ be the law of $\xr{1}$ generated by \cref{alg:DM}. Then
\begin{align*}
    \Dkl{p_1}{\phat_1} \icml{\\
    &}\leqsim \kl{p_K}{\phat_K}+K\delta+\sum_{k=1}^{K} \eta_k \epssct^2\,.
\end{align*}
\end{theorem}

Following the reasoning of the preceding subsection, \cref{thm:DM-Lip} implies a complexity of
\begin{align*}
    \boxed{\LipF\cdot \log^3\prn[\Big]{\frac{d+\mathsf M_2^2}{\delta^2}}}
\end{align*}
to reach $\Dbl{\wh p_1}{\pdata}^2 \lesssim K\delta + \sum_{k=1}^{K} \eta_k \epssct^2$.

\begin{proposition}\label{prop:DDPM-Lip}
Suppose that there are parameters $\LF\geq \Lop\geq 1$ and $M\geq 1$ such that \cref{asmp:non-unif-Lip} and \cref{asmp:Lip-Frob} hold, and $\Lipop\leq \Lop\cdot \polylog(M/\delta)$. 
Then, \cref{asmp:Lip-DDPM} holds with $\LipF'\leq \LipF[\delta/4]$ as long as
\begin{align*}
    \frac{\sigma_k^2}{\eta_k} \gg (\LipF[\delta/4]+\min\crl{d^{1/3}\Lop^{2/3}, \dstar^{2/3}\Lop^{1/3}})\cdot \polylog(\dstar M/\delta).
\end{align*}
\end{proposition}

Therefore, the implied complexity in terms of $\Lop$ is
\begin{align*}
    \boxed{\min\crl*{\sqrt{d\Lop}, \dstar^{2/3}\Lop^{1/3}}\cdot \polylog\prn[\Big]{\frac{\dstar+M+\mathsf M_2^2}{\delta^2}}}.
\end{align*}
\sccomment{By~\cref{prop:Lip-op-to-Frob}, why can't we get $\sqrt{\dstar \Lop}$ instead of $\sqrt{d\Lop}$?}
\fccomment{This is because we have to use the above proposition, which requires $\dstar^{2/3}\Lop^{1/3}$...}

\section{Log-concave sampling}\label{sec:log_concave}
In this section, we briefly describe the implications of our results (\cref{thm:gaussian_tilt}) for the problem of sampling from a density $\mu \propto e^{-f}$ \fcedit{under log-concavity, or more generally, isoperimetric conditions}.
The key connection lies in the \emph{proximal sampler}~\citep{LeeSheTia21RGO, Chen+22ProxSampler}, which is an instance of the Gibbs sampler in which the target distribution is the augmented density
\begin{align*}
    \wt\mu(x,y)\propto \exp\prn[\Big]{-f(x) - \frac{1}{2\eta}\,\nrm{y-x}^2}\,.
\end{align*}
See~\cref{alg:prox-sampler}.

\begin{algorithm}
\caption{Proximal sampler}\label{alg:prox-sampler}
\begin{algorithmic}
\STATE \textbf{Input:} Gradient $\nabla f$, step size $\eta > 0$
\FOR{$n=1,2,3,\dotsc,N$}
\STATE Sample $Y_n \sim \normal{X_n, \eta I}$.
\STATE Sample $X_{n+1}\sim \mathsf{RGO}_{f,\eta,Y_n}$.
\ENDFOR
\STATE Output $X_N$.
\end{algorithmic}
\end{algorithm}

\cref{alg:prox-sampler} hinges on being able to implement the \emph{restricted Gaussian oracle} (RGO):
\begin{align*}
    \mathsf{RGO}_{f,\eta,y}(x)     \propto_x \exp\prn[\Big]{-f(x) - \frac{1}{2\eta}\,\nrm{y-x}^2}\,.
\end{align*}
Convergence of~\cref{alg:prox-sampler} was shown in~\citet{LeeSheTia21RGO} under log-concavity of $\mu$,
and in~\citet{Chen+22ProxSampler} under weaker isoperimetric assumptions.
On the other hand, note that the RGO is exactly a Gaussian tilt distribution, which we considered in~\cref{sec:gaussian_tilts}.
Hence, implementing the RGO step using \AlgWSC{} leads to novel high-accuracy sampling results for log-concave (and isoperimetric) distributions, \emph{without} assuming access to zeroth-order queries for $f$.
Almost all previous approaches (e.g., rejection sampling, Metropolis--Hastings) required zeroth-order queries, whereas algorithms using only first-order queries typically arise as discretizations of diffusion processes and the resulting discretization error degraded their accuracy guarantees.
A notable \textbf{exception} is the result of~\citet{LuWan22Zigzag} on the zigzag sampler, an example of a PDMP\@, which achieved a complexity of $\wt O(\kappa^2 d^{3/2} \log^{5/2}(1/\delta))$ evaluations of \emph{partial} derivatives of $f$, under strong-log-concavity and given a warm start, where $\kappa$ is the condition number. %

We only summarize some representative results here, and defer a full presentation to~\cref{appdx:log-concave}.
Let $\mu_0$, $\wh\mu$ denote the initialization and output of~\cref{alg:prox-sampler}, and in all cases we track the number of queries to a first-order and proximal oracle for $f$ in expectation.

Suppose that $f$ is smooth ($s=1$ in \cref{ass:holder}).
\begin{itemize}
    \item Under a log-Sobolev inequality, we obtain $\Dchis{\wh\mu}{\mu} \le \varepsilon^2$ in $\wt O(\kappa \,(d^{1/2} \log^{3/2}(\cR/\varepsilon^2) + \log^2(\cR/\varepsilon^2)))$ queries, where $\cR \deq \log(1+\Dchis{\mu_0}{\mu})$ and $\kappa \deq \CLSI\,\beta_1$. %
    \item Under a Poincar\'e inequality, we obtain $\Dchis{\wh\mu}{\mu} \le \varepsilon^2$ in $\wt O(\kappa\,( d^{1/2}\log^{1/2}(1/\varepsilon) + \log(1/\varepsilon)) \log(\chi^2/\varepsilon^2))$ queries, where $\chi^2 \deq \Dchis{\mu_0}{\mu}$ and $\kappa \deq \CPI\,\beta_1$.
    \item Under log-concavity, we obtain $\Dkl{\wh\mu}{\mu}\le \varepsilon^2$ in $\wt O(\beta_1 d^{1/2}\, W_2^2(\mu_0,\mu)/\varepsilon^2)$ queries. \fcedit{In addition, recent progress toward KLS~\citep{klartag2023logarithmic} implies $\CPI(\mu)\leq O(\log d)\cdot \nrmop{\En_{\mu}[XX^\top]}$ under log-concavity, and hence our previous result also implies a high-accuracy sampling guarantee for any log-concave $\mu$ (albeit relying on $\log \chi^2$).  }
\end{itemize}
Suppose now that $f$ is Lipschitz ($s=0$ in \cref{ass:holder}).
\begin{itemize}
    \item Under a Poincar\'e inequality, we obtain $\Dchis{\wh\mu}{\mu} \le \varepsilon^2$ in $\wt O(\CPI \,\beta_0^2 \log(1/\varepsilon) \log(\chi^2/\varepsilon^2))$ queries.
    \item Under log-concavity, we obtain $\Dkl{\wh\mu}{\mu}\le \varepsilon^2$ in $\wt O(\beta_0^2 \,W_2^2(\mu_0,\mu)/\varepsilon^2)$ queries.
\end{itemize}

\section{Conclusion}

In this work, we have presented high-accuracy algorithms for diffusion sampling which operate under minimal data assumptions (\cref{thm:DM-intrinsic}), with improved dimension dependence under a Lipschitz score assumption (\cref{thm:DM-Lip}).
With regards to the dependence on the target accuracy, our results improve exponentially over all prior works.
Our framework also has implications for log-concave sampling (\cref{sec:log_concave}) using only first-order queries. 

Although this is a primarily theoretical work, we are working toward implementation and experimental evaluation, which will be left for future work.

\icml{
\section{Impact Statement}

Our work provides new algorithms for diffusion sampling with potential speed-ups.
This could lead to many potential impacts, none of which we feel the need to specifically highlight here.
}

\arxiv{
  \section*{Acknowledgments}
  We thank Sam Power for bringing to our attention useful references.
  We acknowledge support from ARO through award W911NF-21-1-0328, Simons Foundation, and the NSF through awards DMS-2031883 and PHY-2019786,  the DARPA AIQ program, and AFOSR FA9550-25-1-0375. CD is supported by a Simons Investigator Award, a Simons Collaboration on Algorithmic Fairness, ONR MURI grant N00014-25-1-2116, and ONR grant N00014-25-1-2296.
}

\bibliography{ref.bib}

@inproceedings{Ben+24Diffusion,
title={Nearly {$d$}-linear convergence bounds for diffusion models via stochastic localization},
author={Joe Benton and Valentin De Bortoli and Arnaud Doucet and George Deligiannidis},
booktitle={The Twelfth International Conference on Learning Representations},
year={2024},
}

@article{li2024adapting,
  title={Adapting to unknown low-dimensional structures in score-based diffusion models},
  author={Li, Gen and Yan, Yuling},
  journal={Advances in Neural Information Processing Systems},
  volume={37},
  pages={126297--126331},
  year={2024}
}

@article{azangulov2024convergence,
  title={Convergence of diffusion models under the manifold hypothesis in high-dimensions},
  author={Azangulov, Iskander and Deligiannidis, George and Rousseau, Judith},
  journal={arXiv preprint arXiv:2409.18804},
  year={2024}
}

@InProceedings{PotAzaDel25Manifold,
  title = 	 {Linear convergence of diffusion models under the manifold hypothesis},
  author =       {Potaptchik, Peter and Azangulov, Iskander and Deligiannidis, George},
  booktitle = 	 {Proceedings of Thirty Eighth Conference on Learning Theory},
  pages = 	 {4668--4685},
  year = 	 {2025},
  editor = 	 {Haghtalab, Nika and Moitra, Ankur},
  volume = 	 {291},
  series = 	 {Proceedings of Machine Learning Research},
  month = 	 {7},
  publisher =    {PMLR},
}

@article{HuaWeiChe26LowDim,
  title={Denoising diffusion probabilistic models are optimally adaptive to unknown low dimensionality},
  author={Huang, Zhihan and Wei, Yuting and Chen, Yuxin},
  journal={Mathematics of Operations Research},
  year={2026}
}

@InProceedings{liang2025low,
  title = 	 {Low-dimensional adaptation of diffusion models: convergence in total variation (extended abstract)},
  author =       {Liang, Jiadong and Huang, Zhihan and Chen, Yuxin},
  booktitle = 	 {Proceedings of Thirty Eighth Conference on Learning Theory},
  pages = 	 {3723--3729},
  year = 	 {2025},
  editor = 	 {Haghtalab, Nika and Moitra, Ankur},
  volume = 	 {291},
  series = 	 {Proceedings of Machine Learning Research},
  month = 	 {7},
  publisher =    {PMLR},
}

@article{li2025dimension,
  title={Dimension-free convergence of diffusion models for approximate {G}aussian mixtures},
  author={Li, Gen and Cai, Changxiao and Wei, Yuting},
  journal={arXiv preprint arXiv:2504.05300},
  year={2025}
}

@article{tang2025adaptivity,
  title={Adaptivity and convergence of probability flow {ODEs} in diffusion generative models},
  author={Tang, Jiaqi and Yan, Yuling},
  journal={arXiv preprint arXiv:2501.18863},
  year={2025}
}

@article{zhang2025sublinear,
  title={Sublinear iterations can suffice even for {DDPMs}},
  author={Zhang, Matthew S. and Huan, Stephen and Huang, Jerry and Boffi, Nicholas M. and Chen, Sitan and Chewi, Sinho},
  journal={arXiv preprint arXiv:2511.04844},
  year={2025}
}

@article{jiao2024instance,
  title={Instance-dependent convergence theory for diffusion models},
  author={Jiao, Yuchen and Li, Gen},
  journal={arXiv preprint arXiv:2410.13738},
  year={2024}
}

@article{jiao2025optimal,
  title={Optimal convergence analysis of {DDPM} for general distributions},
  author={Jiao, Yuchen and Zhou, Yuchen and Li, Gen},
  journal={arXiv preprint arXiv:2510.27562},
  year={2025}
}

@InProceedings{li2024accelerating,
  title = 	 {Accelerating convergence of score-based diffusion models, provably},
  author =       {Li, Gen and Huang, Yu and Efimov, Timofey and Wei, Yuting and Chi, Yuejie and Chen, Yuxin},
  booktitle = 	 {Proceedings of the 41st International Conference on Machine Learning},
  pages = 	 {27942--27954},
  year = 	 {2024},
  editor = 	 {Salakhutdinov, Ruslan and Kolter, Zico and Heller, Katherine and Weller, Adrian and Oliver, Nuria and Scarlett, Jonathan and Berkenkamp, Felix},
  volume = 	 {235},
  series = 	 {Proceedings of Machine Learning Research},
  month = 	 {7},
  publisher =    {PMLR},
}

@article{li2024provable,
  title={Provable acceleration for diffusion models under minimal assumptions},
  author={Li, Gen and Cai, Changxiao},
  journal={arXiv preprint arXiv:2410.23285},
  year={2024}
}

@article{huang2025convergence,
  title={Convergence analysis of probability flow {ODE} for score-based generative models},
  author={Huang, Daniel Zhengyu and Huang, Jiaoyang and Lin, Zhengjiang},
  journal={IEEE Transactions on Information Theory},
  year={2025},
  publisher={IEEE}
}

@article{li2025faster,
  title={Faster Diffusion Models via Higher-Order Approximation},
  author={Li, Gen and Zhou, Yuchen and Wei, Yuting and Chen, Yuxin},
  journal={arXiv preprint arXiv:2506.24042},
  year={2025}
}

@article{huang2025fast,
  title={Fast convergence for high-order {ODE} solvers in diffusion probabilistic models},
  author={Huang, Daniel Zhengyu and Huang, Jiaoyang and Lin, Zhengjiang},
  journal={arXiv preprint arXiv:2506.13061},
  year={2025}
}

@book {BGL14,
    AUTHOR = {Bakry, Dominique and Gentil, Ivan and Ledoux, Michel},
     TITLE = {Analysis and geometry of {M}arkov diffusion operators},
    SERIES = {Grundlehren der Mathematischen Wissenschaften [Fundamental
              Principles of Mathematical Sciences]},
    VOLUME = {348},
 PUBLISHER = {Springer, Cham},
      YEAR = {2014},
     PAGES = {xx+552},
}

@book {Chewi26Book,
    AUTHOR = {Chewi, Sinho},
     TITLE = {Log-concave sampling},
 PUBLISHER = {Forthcoming},
      YEAR = {2026},
     NOTE = {Available online at \url{https://chewisinho.github.io/}}
}

@article {Che+24LMC,
    AUTHOR = {Chewi, Sinho and Erdogdu, Murat A. and Li, Mufan (Bill) and Shen, Ruoqi and Zhang, Matthew S.},
     TITLE = {Analysis of {L}angevin {M}onte {C}arlo from {P}oincar\'{e} to log-{S}obolev},
   JOURNAL = {Found. Comput. Math.},
  FJOURNAL = {Foundations of Computational Mathematics. The Journal of the
              Society for the Foundations of Computational Mathematics},
      YEAR = {2024},
     PAGES = {},
}

@InProceedings{FanYuaChe23ImprovedProx,
  title = 	 {Improved dimension dependence of a proximal algorithm for sampling},
  author =       {Fan, Jiaojiao and Yuan, Bo and Chen, Yongxin},
  booktitle = 	 {Proceedings of Thirty Sixth Conference on Learning Theory},
  pages = 	 {1473--1521},
  year = 	 {2023},
  editor = 	 {Neu, Gergely and Rosasco, Lorenzo},
  volume = 	 {195},
  series = 	 {Proceedings of Machine Learning Research},
  month = 	 {7},
  publisher =    {PMLR},
}

@article {ConDurSil25KLDiffusion,
    AUTHOR = {Conforti, Giovanni and Durmus, Alain and Gentiloni Silveri,
              Marta},
     TITLE = {K{L} convergence guarantees for score diffusion models under minimal data assumptions},
   JOURNAL = {SIAM J. Math. Data Sci.},
  FJOURNAL = {SIAM Journal on Mathematics of Data Science},
    VOLUME = {7},
      YEAR = {2025},
    NUMBER = {1},
     PAGES = {86--109},
}

@inproceedings{LiYan25Diffusion,
title={{$O(d/T)$} convergence theory for diffusion probabilistic models under minimal assumptions},
author={Gen Li and Yuling Yan},
booktitle={The Thirteenth International Conference on Learning Representations},
year={2025},
}

@inproceedings{JaiZha26SharpKL,
title={A sharp {KL} convergence analysis for diffusion models under minimal assumptions},
author={Nishant Jain and Tong Zhang},
booktitle={The Fourteenth International Conference on Learning Representations},
year={2026},
}

@InProceedings{Chen+22ProxSampler,
  title = 	 {Improved analysis for a proximal algorithm for sampling},
  author =       {Chen, Yongxin and Chewi, Sinho and Salim, Adil and Wibisono, Andre},
  booktitle = 	 {Proceedings of Thirty Fifth Conference on Learning Theory},
  pages = 	 {2984--3014},
  year = 	 {2022},
  editor = 	 {Loh, Po-Ling and Raginsky, Maxim},
  volume = 	 {178},
  series = 	 {Proceedings of Machine Learning Research},
  month = 	 {7},
  publisher =    {PMLR},
}

@InProceedings{LeeSheTia21RGO,
  title = 	 {Structured logconcave sampling with a restricted {G}aussian oracle},
  author =       {Lee, Yin Tat and Shen, Ruoqi and Tian, Kevin},
  booktitle = 	 {Proceedings of Thirty Fourth Conference on Learning Theory},
  pages = 	 {2993--3050},
  year = 	 {2021},
  editor = 	 {Belkin, Mikhail and Kpotufe, Samory},
  volume = 	 {134},
  series = 	 {Proceedings of Machine Learning Research},
  month = 	 {8},
  publisher =    {PMLR},
}

@article{gatmiry2026high,
  title={High-accuracy and dimension-free sampling with diffusions},
  author={Gatmiry, Khashayar and Chen, Sitan and Salim, Adil},
  journal={arXiv preprint arXiv:2601.10708},
  year={2026}
}

@article{wainwright2025score,
  title={Score-based sampling without diffusions: guidance from a simple and modular scheme},
  author={Wainwright, Martin J.},
  journal={arXiv preprint arXiv:2512.24152},
  year={2025}
}

@article{huang2024reverse,
  title={Reverse transition kernel: a flexible framework to accelerate diffusion inference},
  author={Huang, Xunpeng and Zou, Difan and Dong, Hanze and Zhang, Zhang and Ma, Yian and Zhang, Tong},
  journal={Advances in Neural Information Processing Systems},
  volume={37},
  pages={95515--95578},
  year={2024}
}

@article{keane1994bernoulli,
  title={A {B}ernoulli factory},
  author={Keane, MS and O'Brien, George L.},
  journal={ACM Transactions on Modeling and Computer Simulation (TOMACS)},
  volume={4},
  number={2},
  pages={213--219},
  year={1994},
  publisher={ACM New York, NY, USA}
}

@article {nacu2005fast,
    AUTHOR = {Nacu, Serban and Peres, Yuval},
     TITLE = {Fast simulation of new coins from old},
   JOURNAL = {Ann. Appl. Probab.},
  FJOURNAL = {The Annals of Applied Probability},
    VOLUME = {15},
      YEAR = {2005},
    NUMBER = {1A},
     PAGES = {93--115},
}

@article{AltChe24Warm,
author = {Altschuler, Jason M. and Chewi, Sinho},
title = {Faster high-accuracy log-concave sampling via algorithmic warm starts},
year = {2024},
issue_date = {June 2024},
publisher = {Association for Computing Machinery},
address = {New York, NY, USA},
volume = {71},
number = {3},
journal = {J. ACM},
month = {6},
articleno = {24},
numpages = {55},
}

@article {LuWan22Zigzag,
    AUTHOR = {Lu, Jianfeng and Wang, Lihan},
     TITLE = {Complexity of zigzag sampling algorithm for strongly log-concave distributions},
   JOURNAL = {Stat. Comput.},
  FJOURNAL = {Statistics and Computing},
    VOLUME = {32},
      YEAR = {2022},
    NUMBER = {3},
     PAGES = {Paper No. 48, 12},
}

@inproceedings{Chen+23SGM,
title={Sampling is as easy as learning the score: theory for diffusion models with minimal data assumptions},
author={Sitan Chen and Sinho Chewi and Jerry Li and Yuanzhi Li and Adil Salim and Anru R. Zhang},
booktitle={The Eleventh International Conference on Learning Representations },
year={2023},
}

@InProceedings{LeeLuTan23SGMGeneral,
  title = 	 {Convergence of score-based generative modeling for general data distributions},
  author =       {Lee, Holden and Lu, Jianfeng and Tan, Yixin},
  booktitle = 	 {Proceedings of the 34th International Conference on Algorithmic Learning Theory},
  pages = 	 {946--985},
  year = 	 {2023},
  editor = 	 {Agrawal, Shipra and Orabona, Francesco},
  volume = 	 {201},
  series = 	 {Proceedings of Machine Learning Research},
  month = 	 {2},
  publisher =    {PMLR},
}

@inproceedings{Chen+23FlowODE,
 author = {Chen, Sitan and Chewi, Sinho and Lee, Holden and Li, Yuanzhi and Lu, Jianfeng and Salim, Adil},
 booktitle = {Advances in Neural Information Processing Systems},
 editor = {A. Oh and T. Neumann and A. Globerson and K. Saenko and M. Hardt and S. Levine},
 pages = {68552--68575},
 publisher = {Curran Associates, Inc.},
 title = {The probability flow {ODE} is provably fast},
 volume = {36},
 year = {2023}
}

@article{gao2025wasserstein,
  title={Wasserstein convergence guarantees for a general class of score-based generative models},
  author={Gao, Xuefeng and Nguyen, Hoang M and Zhu, Lingjiong},
  journal={Journal of Machine Learning Research},
  volume={26},
  number={43},
  pages={1--54},
  year={2025}
}

@inproceedings{chen2023improved,
  title={Improved analysis of score-based generative modeling: user-friendly bounds under minimal smoothness assumptions},
  author={Chen, Hongrui and Lee, Holden and Lu, Jianfeng},
  booktitle={International Conference on Machine Learning},
  pages={4735--4763},
  year={2023},
  organization={PMLR}
}

@article{li2023towards,
  title={Towards faster non-asymptotic convergence for diffusion-based generative models},
  author={Li, Gen and Wei, Yuting and Chen, Yuxin and Chi, Yuejie},
  journal={arXiv preprint arXiv:2306.09251},
  year={2023}
}

@article{li2024sharp,
  title={A sharp convergence theory for the probability flow {ODEs} of diffusion models},
  author={Li, Gen and Wei, Yuting and Chi, Yuejie and Chen, Yuxin},
  journal={arXiv preprint arXiv:2408.02320},
  year={2024}
}

@InProceedings{GaoZhu25PFODE,
  title = 	 {Convergence analysis for general probability flow {ODEs} of diffusion models in {W}asserstein distances},
  author =       {Gao, Xuefeng and Zhu, Lingjiong},
  booktitle = 	 {Proceedings of the 28th International Conference on Artificial Intelligence and Statistics},
  pages = 	 {1009--1017},
  year = 	 {2025},
  editor = 	 {Li, Yingzhen and Mandt, Stephan and Agrawal, Shipra and Khan, Emtiyaz},
  volume = 	 {258},
  series = 	 {Proceedings of Machine Learning Research},
  month = 	 {5},
  publisher =    {PMLR},
}

@incollection {Pap11Diffusion,
    AUTHOR = {Papaspiliopoulos, Omiros},
     TITLE = {Monte {C}arlo probabilistic inference for diffusion processes:
              a methodological framework},
 BOOKTITLE = {Bayesian time series models},
     PAGES = {82--103},
 PUBLISHER = {Cambridge Univ. Press, Cambridge},
      YEAR = {2011},
}

@article {Wag1988SDE,
    AUTHOR = {Wagner, Wolfgang},
     TITLE = {Monte {C}arlo evaluation of functionals of solutions of
              stochastic differential equations. {V}ariance reduction and
              numerical examples},
   JOURNAL = {Stochastic Anal. Appl.},
  FJOURNAL = {Stochastic Analysis and Applications},
    VOLUME = {6},
      YEAR = {1988},
    NUMBER = {4},
     PAGES = {447--468},
}

@misc{Hua+25QuantizedDiffusion,
      title={Almost linear convergence under minimal score assumptions: quantized transition diffusion}, 
      author={Xunpeng Huang and Yingyu Lin and Nikki L. Kuang and Hanze Dong and Difan Zou and Yian Ma and Tong Zhang},
      year={2025},
      journal={arXiv preprint 2505.21892},
}

@article {beskos2005exact,
    AUTHOR = {Beskos, Alexandros and Roberts, Gareth O.},
     TITLE = {Exact simulation of diffusions},
   JOURNAL = {Ann. Appl. Probab.},
  FJOURNAL = {The Annals of Applied Probability},
    VOLUME = {15},
      YEAR = {2005},
    NUMBER = {4},
     PAGES = {2422--2444},
}

@article{beskos2006retrospective,
  title={Retrospective exact simulation of diffusion sample paths with applications},
  author={Beskos, Alexandros and Papaspiliopoulos, Omiros and Roberts, Gareth O},
  journal={Bernoulli},
  volume={12},
  number={6},
  pages={1077--1098},
  year={2006},
  publisher={Bernoulli Society for Mathematical Statistics and Probability}
}

@article{klartag2023logarithmic,
  title={Logarithmic bounds for isoperimetry and slices of convex sets},
  author={Klartag, Bo'az},
  journal={arXiv preprint arXiv:2303.14938},
  year={2023}
}

@book{polyanskiy2025information,
  title={Information theory: From coding to learning},
  author={Polyanskiy, Yury and Wu, Yihong},
  year={2025},
  publisher={Cambridge university press}
}
\icml{
\bibliographystyle{icml2026}
}

\appendix
\newpage
\icml{
\onecolumn
}

\section{Discussion}

\subsection{Concurrent work}\label{sec:concurrent}
In independent and concurrent work, \citet{gatmiry2026high} study an accelerated ODE flow and \scedit{remarkably} obtain a \emph{high-accuracy} guarantee under a set of structural assumptions. \scedit{This is highly non-trivial since, as discussed in the introduction, prior works based on higher-order methods incur implicit exponential dependencies on problem parameters and the order of the method.} Their results rely on the following conditions:
\begin{enumerate}[leftmargin=3em]
\item[(A1)] The data distribution satisfies $\pdata = p_\star * \normal{0,\sigma_\star^2 I}$, where $p_\star$ is supported on a ball of radius $R$.
\item[(A2)] \emph{Sub-exponential score error}: for some $\epssc = \tO(\delta)$,
\[
\PP_{X \sim p_t}\!\left(\nrm{\scf_t(X) - \scfs_t(X)} \ge u\right) \le \exp(-u/\epssc)\,.
\]
\item[(A3)] The score estimates $(\scf_t)_{t\in [T]}$ are Lipschitz.
\end{enumerate}
Under these assumptions, \citet{gatmiry2026high} propose a method with query complexity $\tO\left((R/\sigma_\star)^2 \log^2(1/\delta)\right)$.

\scedit{In comparison with our results}, we make the following remarks:
\begin{itemize}
\item \fcedit{The query complexity is recovered by \cref{cor:poly-log}, as the intrinsic dimension $\dstar$ can always be upper bounded\footnote{Formally, we instantiate \cref{thm:DM-intrinsic} and \cref{cor:poly-log} by regarding $\pdata=p_\star$, $\alp_0=1$, and $\sigma_0=\sigma_\star$ so that $p_1=p_\star*\normal{0,\sigma_\star^2}$, and then \cref{thm:DM-intrinsic} provides sampling guarantees for $p_1$.} by $(R/\sigma_\star)^2$. On one hand, when $R/\sigma_\star$ is constant, this guarantee is \emph{dimension-free}. On the other hand, $R/\sigma_\star$ could be large in many settings where $\dstar$ remains bounded.}

\item As noted by \citet{gatmiry2026high}, (A1) indeed holds in the \emph{early stopping} regime, i.e., the data distribution actually corresponds to the early stopped distribution $p_1$. However, in this case, achieving distributional closeness to the real data distribution $p_\star$ typically requires $\sigma_\star \ll \delta$, and then we should expect \scedit{$\poly(1/\delta)$ dependence, instead of a high-accuracy $\polylog(1/\delta)$ dependence}. Further, the bound $R$ may also scale with $\sqrt{d}$ in some settings.
\item Assumption (A2) is substantially stronger than our average error condition (\cref{def:score-error}). It is unclear whether standard statistical learning procedures can achieve convergence guarantees expressed in terms of sub-exponential tail bounds.
\item Finally, while the analysis of \citet{gatmiry2026high} is specific to accelerated ODE-based diffusion models, our approach applies more broadly to general first-order sampling methods, including log-concave sampling (\cref{sec:log_concave}).
\end{itemize}

\section{Additional notation and technical tools}
For any function $f:\RR^d\to\RR$ such that $Z_f\deq \int_{\RR^d} e^{-f(x)}dx<+\infty$, we define $\mu_f$ to be the distribution over $\RR^d$ with density $\mu_f(x)=\frac{1}{Z_f}e^{-f(x)}$.

For $\nu \ll \mu$, let $\Ren[\lambda]{\nu}{\mu} \deq \frac{1}{\lambda-1} \log \En_\mu[(\frac{d\nu}{d\mu})^\lambda]$ denote the R\'enyi divergence of order $\lambda > 1$. In addition, we define
\begin{align}
    \Dren{p}{q} \deq \E_p\prn[\big]{\frac{dp}{dq}}^\lambda - 1\,,\qquad
    \Dsys{p}{q}=\max\crl{\Dren{p}{q},\Dren{q}{p}}\,.
\end{align}

\subsection{Functional inequalities}

\begin{definition}[PI]\label{def:PI}
A probability distribution $\mu$ on $\R^d$ satisfies a Poincar\'e inequality (PI) with constant $C$ if for all smooth and compactly supported functions $\phi:\R^d\to \R$,
\begin{align*}
    \var_\mu(\phi)\leq C\En_\mu \nrm{\nabla \phi}^2\,.
\end{align*}
We let $\CPI(\mu)$ denote the smallest constant $C$ such that $\mu$ satisfies PI with constant $C$.
\end{definition}

\begin{definition}[LSI]\label{def:LSI}
A probability distribution $\mu$ on $\R^d$ satisfies a log-Sobolev inequality (LSI) with constant $C$ if for all smooth and compactly supported functions $\phi:\R^d\to \R$,
\begin{align*}
    \En_\mu\brk[\Big]{ \phi^2 \log \frac{\phi^2}{\En_\mu[\phi^2]}}\leq 2C\En_\mu \nrm{\nabla \phi}^2\,.
\end{align*}
We let $\CLSI(\mu)$ denote the smallest constant $C$ such that $\mu$ satisfies LSI with constant $C$.
\end{definition}

The following facts are standard, see~\citet{BGL14} or~\citet[Chapter 2]{Chewi26Book}.

\begin{lemma}[Bakry--\'Emery]\label{lem:LSI-SC}
If $\mu$ is $\alpha$-strongly log-concave, then $\CLSI(\mu)\leq \frac{1}{\alpha}$.
\end{lemma}

\begin{lemma}[Holley--Stroock perturbation principle]\label{lem:HS-perturb}
Let $\mu$ and $\nu$ be probability measures such that $\frac{1}{b}\leq \frac{d\nu}{d\mu}\leq b$ for some parameter $b>0$. Then $\CLSI(\nu)\leq b^2\,\CLSI(\mu)$.
\end{lemma}

\subsection{Technical lemmas}

\subsubsection{Tail estimates}

\begin{lemma}\label{lem:Gaussian-vMGF}
\begin{align*}
    \En_{W\sim \normal{0,\eta\Id}} \exp\prn*{\abs{\tri{W,v}}}\leq 2\exp\prn[\Big]{\frac{\eta\nrm{v}^2}{2}}\,.
\end{align*}
For $\lambda\leq \frac{1}{2\eta}$,
\begin{align*}
    \En_{W\sim \normal{0,\eta\Id}} \exp\prn[\Big]{\frac12\lambda\nrm{W+v}^2}\leq \exp\prn[\big]{d\lambda \eta+\lambda\nrm{v}^2}\,.
\end{align*}
Further, for $t\geq 0$,
\begin{align*}
    \PP_{x\sim \normal{v,\Sigma}}\prn*{\abs{x\tp Ax-v\tp A v-\tr(\Sigma A)}\geq t}&\leq 2\exp\prn*{ -\frac{1}{16\nrmop{\Sigma}}\min\crl*{\frac{t^2}{\nrmop{\Sigma}\nrmF{A}^2+\nrm{Av}^2}, \frac{t}{\nrmop{A}} }}\,.
\end{align*}
\end{lemma}
\begin{proof}
    \newcommand{\tSigma}{\wt{\Sigma}}
The first two statements are standard. For the third, for any $A\preceq \frac12\Sigma^{-1}$, we can calculate
\begin{align*}
    \En_{x\sim \normal{v,\Sigma}}\exp\prn[\Big]{\frac12\,x\tp Ax-\frac12\,v\tp A v}
    =\En_{W\sim \normal{0,\Sigma}} \exp\prn[\Big]{ v\tp A W+\frac12\, W\tp A W }
    =\sqrt{\frac{\det(\tSigma)}{\det(\Sigma)}}\exp\prn[\Big]{ \frac12\nrm{Av}_{\tSigma}^2 }\,,
\end{align*}
where $\tSigma^{-1}=\Sigma^{-1}-A$. Noting that $\Sigma^{1/2}\tSigma^{-1}\Sigma^{1/2}=\Id-\Sigma^{1/2}A\Sigma^{1/2}\succeq \frac12\Id$, we can bound
\begin{align*}
    \frac{\det(\tSigma)}{\det(\Sigma)}=\frac{1}{\det\prn*{\Id-\Sigma^{1/2}A\Sigma^{1/2}}}\leq \exp\prn[\big]{\tr(\Sigma^{1/2}A\Sigma^{1/2})+\nrmF{\Sigma^{1/2}A\Sigma^{1/2}}^2}\,,
\end{align*}
where we use $-\log(1-\lambda)\leq \lambda+\lambda^2$ for $\lambda\leq \frac12$. Therefore, we have shown
\begin{align*}
    \En_{x\sim \normal{v,\Sigma}}\exp\prn[\Big]{\frac12\,\prn*{x\tp Ax-v\tp A v-\tr(\Sigma A)}}
    \leq\exp\prn[\big]{\nrmF{\Sigma^{1/2}A\Sigma^{1/2}}^2+\nrm{Av}_{\Sigma}^2}\,, 
\end{align*}
as long as $A\preceq \frac12\Sigma^{-1}$. Then, by rescaling $A\leftarrow \lambda A$ and requiring $\abs{\lambda}\leq \frac{1}{2\nrmop{\Sigma^{1/2}A\Sigma^{1/2}}}$, we have
\begin{align*}
    \PP\prn*{\abs{x\tp Ax-v\tp A v-\tr(\Sigma A)}\geq t}\leq 2\exp\prn[\Big]{ \lambda^2\,\prn[\big]{\nrmF{\Sigma^{1/2}A\Sigma^{1/2}}^2+\nrm{Av}_{\Sigma}^2}-\frac12\lambda t }\,, \qquad \forall t>0\,.
\end{align*}
Suitably choosing $\lambda$ gives, for all $t\ge 0$,
\begin{align*}
    \PP\prn*{\abs{x\tp Ax-v\tp A v-\tr(\Sigma A)}\geq t}\leq&~ 2\exp\prn*{ -\min\crl*{\frac{t^2}{16(\nrmF{\Sigma^{1/2}A\Sigma^{1/2}}^2+\nrm{Av}_{\Sigma}^2)},\, \frac{t}{8\nrmop{\Sigma^{1/2}A\Sigma^{1/2}}} }} \\
    \leq&~ 2\exp\prn*{ -\frac{1}{16\nrmop{\Sigma}}\min\crl*{\frac{t^2}{\nrmop{\Sigma}\nrmF{A}^2+\nrm{Av}^2},\, \frac{t}{\nrmop{A}} }}\,.
\end{align*}
\end{proof}

\begin{lemma}\label{lem:Gaussian-s}
Suppose that $\eta>0$ and $0\leq \lambda\leq \frac{d^{1-s}}{4s\eta^s}$. Then, it holds that
\begin{align*}
    \En_{W\sim \normal{0,\eta\Id}} \exp\prn*{\lambda\nrm{W}^{2s}}\leq \exp\prn*{2(\eta d)^s \lambda}\,.
\end{align*}
\end{lemma}

\begin{proof}
We use the inequality $w^s\leq s\cdot\frac{w}{(\eta d)^{1-s}}+(1-s)\cdot (\eta d)^s$ for $w\geq 0$. Therefore, as long as $\frac{s\lambda}{(\eta d)^{1-s}}\leq \frac{1}{4\eta}$, it holds that
\begin{align*}
    \En\exp\prn*{\lambda\nrm{W}^{2s}}
    \leq \En\exp\prn[\Big]{\frac{s\lambda}{(\eta d)^{1-s}}\nrm{W}^{2}+(1-s)(\eta d)^s\lambda }
    \leq \exp\prn*{2s(\eta d)^s\lambda+(1-s)(\eta d)^s\lambda }
    \leq \exp\prn*{2(\eta d)^s \lambda}\,,
\end{align*}
where the second inequality follows from \cref{lem:Gaussian-vMGF}.
\end{proof}

\newcommand{\nrmexp}[1]{\nrm{#1}_{\psi_1}}
\begin{lemma}\label{lem:sub-exp-trunc}
For a random variable $Y$, we define the \emph{sub-exponential norm} of $Y$ as
\begin{align*}
    \nrmexp{Y}\deq \inf\crl{M>0: \En\exp(\abs{Y}/M)\leq 2}\,.
\end{align*} 
Then for any $s\geq 1$ and $t> 2s\nrmexp{Y}$, we have $\En(\abs{Y}^s-t^s)_+\leq 4t^se^{-t/\nrmexp{Y}}$. 
\end{lemma}
\begin{proof}
Fix any $M>\nrmexp{Y}$, and we only need to show $\En(\abs{Y}^s-t^s)_+\leq 4t^s e^{-t/M}$ for $t\geq 2sM$. Without loss of generality we assume $M=1$.  Note that we have $\PP(\abs{Y}\geq y)\leq 2e^{-y/M}$ for $y\geq 0$. Therefore, for $A\geq 2s$, 
\begin{align*}
    \En(\abs{Y}^s-A^s)_+
    \leq \En\abs{Y}^s \indic\crl{\abs{Y}\geq A}
    =s\int_A^{\infty} \PP(\abs{Y}\geq y)\,y^{s-1}\,dy
    \leq 2s\int_A^{\infty} e^{-y}\,y^{s-1}\,dy\leq 4sA^{s-1}e^{-A}\,.
\end{align*}
This is the desired upper bound. 
\end{proof}

\begin{lemma}\label{lem:Gaussian-Dcov}
Suppose $p=\normal{u,\eta\Id}$ and $q=\normal{v,\eta\Id}$. Then for $\delta\in(0,1)$, as long as $\eta\geq 8\log(1/\delta)\nrm{u-v}^2$, it holds that $p\prn[\big]{\frac{dp}{dq}\geq e}\leq \delta$.
\end{lemma}

\begin{lemma}\label{lem:Ito-Freedman}
    Suppose that $(B_t)_{t\ge 0}$ is the $d$-dimensional Brownian motion and $(J_t)_{t\ge 0}$ is a process of symmetric random matrices adapted to the filtration of $(B_t)_{t\ge 0}$.
Denote $Z_t\deq\int_0^t J_s\, dB_s$ and $M_t\deq \int_0^t J_s^2\,ds$. Then for any $R>0$, and distribution $u$ over $[0,T]$, 
\begin{align*}
    \bbP(\En_{t\sim u}\nrm{Z_t}^2\geq 4R\log(e/\delta))\leq 2\bbP(\tr(M_T)\geq R)+\delta\,.
\end{align*}
\end{lemma}

\begin{proof}
For any vector $v\in\RR^d$ and $\lambda\in\RR$, we know
\begin{align*}
    \En \exp\prn[\Big]{ \lambda \tri{v, Z_t}-\frac{\lambda^2}{2} \nrm{v}^2_{M_t} }\leq 1\,.
\end{align*}
Taking expectation over $v\sim \normal{0,\Id}$, it holds that
\begin{align*}
    \En \exp\prn[\Big]{ \frac{1}{2}\lambda^2 \nrm{Z_t}^2_{(\Id+\lambda^2 M_t)^{-1}}-\frac12\log\det(\Id+\lambda^2 M_t) }\leq 1\,.
\end{align*}
Note that $\nrm{Z_t}^2_{(\Id+\lambda^2 M_t)^{-1}}\geq \frac{1}{1+\lambda^2\nrmop{M_t}}\nrm{Z_t}^2$ and $\log\det(\Id+\lambda^2 M_t)\leq \lambda^2 \tr(M_t)$, we know
\begin{align*}
    \En \exp\prn[\Big]{ \frac{\lambda^2}{2(1+\lambda^2\nrmop{M_T})}\En_t\nrm{Z_t}^2-\frac{\lambda^2}2\tr(M_T) }
    \leq \En \En_{t}\exp\prn[\Big]{ \frac{\lambda^2}{2(1+\lambda^2\nrmop{M_t})}\nrm{Z_t}^2-\frac{\lambda^2}2\tr(M_t) }\leq 1\,.
\end{align*}
Therefore, by Markov's inequality, it holds that for any $\delta\in(0,1)$ and $\lambda>0$,
\begin{align*}
    \bbP\prn*{ \En_t\nrm{Z_t}^2\geq 2(1+\lambda^2\nrmop{M_T})\prn*{\lambda^{-2}\log(1/\delta)+\tr(M_T)} }\leq \delta\,.
\end{align*}
Choosing $\lambda^{-2}=R$ gives the desired upper bound.
\end{proof}

\subsubsection{Change of measure and tilts}

\begin{lemma}\label{lem:KL-triangle}

\begin{align*}
    \Dkl{\nu}{\muhat}\leq 2\Dkl{\nu}{\mu}+\log\prn*{1+\Dchis{\mu}{\muhat}}\,.
\end{align*}
\end{lemma}
\begin{proof}
We define $h=\log \frac{d\mu}{d\muhat}$. Then, we know that $\En_{\mu}[e^h]=1+\Dchis{\mu}{\muhat}$, and hence
\begin{align*}
    \Dkl{\nu}{\muhat}-\Dkl{\nu}{\mu}=&~\En_{\nu}\brk[\Big]{\log \frac{d\nu}{d\muhat}-\log \frac{d\nu}{d\mu}} 
    = \En_\nu[h] \\
    \leq&~ \Dkl{\nu}{\mu}+\log \En_{\mu}[e^{h}]
    =\Dkl{\nu}{\mu}+\log\prn*{1+\Dchis{\mu}{\muhat}}\,,
\end{align*}
where we use the Donsker--Varadhan variational inequality.
\end{proof}

\begin{lemma}\label{lem:Dchis-perturb}
Suppose that $\mu\propto pe^{h}$ and $\muhat\propto qe^{h}$, such that $\abs{h}\leq B$. Then it holds that
\begin{align*}
    \Dchis{\mu}{\muhat}\leq e^{4B}\Dchis{p}{q}\,.
\end{align*} 
\end{lemma}
\begin{proof}
By definition,
\begin{align*}
    1+\Dchis{\mu}{\muhat}
    =&~ \En_{x\sim \mu}\brk[\Big]{\frac{\mu(x)}{\muhat(x)}}
    =\frac{\En_q[e^h]}{(\En_p[e^h])^2}\cdot \En_{x\sim q}\brk[\Big]{ \frac{p(x)^2}{q(x)^2}\cdot e^{h(x)} }\,.
\end{align*}
Note that $\frac{p(x)^2}{q(x)^2}=\prn[\big]{\frac{p(x)}{q(x)}-1}^2+2\,\frac{p(x)}{q(x)}-1$, and hence
\begin{align*}
    \En_{x\sim q}\brk[\Big]{ \frac{p(x)^2}{q(x)^2}\cdot e^{h(x)} }
    =\En_{x\sim q}\brk[\Big]{ \prn[\Big]{\frac{p(x)}{q(x)}-1}^2 \cdot e^{h(x)} }+2\En_p[e^h]-\En_q[e^h]\,.
\end{align*}
Hence, we can rewrite
\begin{align*}
    \Dchis{\mu}{\muhat}
    =&~
    \frac{\En_q[e^h]}{(\En_p[e^h])^2}\cdot \En_{x\sim q}\brk[\Big]{ \prn[\Big]{\frac{p(x)}{q(x)}-1}^2 \cdot e^{h(x)} } - \prn[\Big]{\frac{\En_q[e^h]}{\En_p[e^h]}-1}^2 \\
    \leq&~ e^{4B}\En_{x\sim q}\brk[\Big]{ \prn[\Big]{\frac{p(x)}{q(x)}-1}^2  }
    =e^{4B}\Dchis{p}{q}\,.
\end{align*}
\end{proof}

\begin{lemma}\label{lem:Hels-var}
Suppose that $P,Q\in\DZ$. Then, for $h:\cZ\to[0,1]$,
\begin{align*}
    \En_P[h]\leq 3\En_Q[h]+4\Dhels{P}{Q}\,.
\end{align*}
Further, for any $f:\cZ\to [-1,1]$, it holds that
\begin{align*}
    \abs{\En_P[f]-\En_Q[f]}\leq 4\sqrt{\En_Q[f^2]\cdot \Dhels{P}{Q}}+4\Dhels{P}{Q}\,.
\end{align*}
\end{lemma}

\begin{proof}
We denote $P(\cdot)$ (resp.\ $Q(\cdot)$) to be the density function of $P$ (resp.\ $Q$). Then for any function $f:\cZ\to \RR$,
  \begin{align*}
    \abs{\En_P[f]-\En_Q[f]}^2=&~ \prn[\Big]{\int_{\cZ} f(z)\,(P(z)-Q(z))\, dz}^2 \\
    \leq&~\int_{\cZ} f(z)^2\,(\sqrt{P(z)}+\sqrt{Q(z)})^2 \, dz \cdot \int_{\cZ} (\sqrt{P(z)}-\sqrt{Q(z)})^2 \, dz \\
    \leq&~ 4\Dhels{P}{Q} \cdot \prn*{ \En_Q[f^2]+ \En_{P}[f^2]}\,.
  \end{align*}
  In particular, when $h:\cZ\to [0,1]$, the inequality above implies that 
  \begin{align*}
    \abs{\En_P[h]-\En_Q[h]}\leq 2\Dhel{P}{Q} \sqrt{ (\En_P[h]+\En_Q[h])}
    \leq \frac{1}{2}\,(\En_P[h]+\En_Q[h])+2\Dhels{P}{Q}\,,
  \end{align*}
  and hence it holds that $\En_P[h]\leq 3\En_Q[h]+4\Dhels{P}{Q}$. 

  Now, we can bound
  \begin{align*}
    \abs{\En_P[f]-\En_Q[f]}^2
    \leq&~ 4\Dhels{P}{Q} \cdot \prn*{ \En_Q[f^2]+ \En_{P}[f^2]} \\
    \leq&~ 16\Dhels{P}{Q} \cdot \prn*{ \En_Q[f^2]+ \Dhels{P}{Q}}\,.
  \end{align*}
   This gives the desired upper bound.
\end{proof}

\begin{lemma}\label{lem:KL-perturb}
Suppose that $\muhat\propto \mu e^{-h}$, where $\abs{h}\leq \B$. Then for any distribution $\nu$, it holds that
\begin{align*}
    \Dkl{\nu}{\muhat}\leq 4e^B\Dkl{\nu}{\mu}+4e^B\En_\mu[h^2]\,.
\end{align*} 
\end{lemma}
\begin{proof}
By definition,
\begin{align*}
    \Dkl{\nu}{\muhat}-\Dkl{\nu}{\mu}=&~\En_{\nu}[\log \mu(x)-\log \muhat(x)] \\
    =&~ \En_\nu[h]+\log \En_{\mu}[e^{-h}]
    \leq \En_\nu[h]-\En_\mu[h]+c_B\En_\mu[h^2]\,,
\end{align*}
where $c_B=\frac{e^B-B-1}{B^2}$. By \cref{lem:Hels-var}, it holds that
\begin{align*}
    \abs{\En_\nu[h]-\En_\mu[h]}\leq 4\sqrt{\En_\mu[h^2]\Dhels{\nu}{\mu}} + 4B\Dhels{\nu}{\mu} 
    \le 2\En_\nu[h^2] + (4B+2)\Dhels{\nu}{\mu}\,.
\end{align*}
Combining the inequalities above with $\Dhels{\nu}{\mu}\leq \frac12 \Dkl{\nu}{\mu}$ completes the proof.
\end{proof}

\begin{lemma}\label{lem:Renyi-to-diff}
For any $\ell>1$, it holds that
\begin{align}
    \Dsys[\ell]{\mu_f}{\mu_g}
    \leq \En_{x\sim \mu_f}\brk{e^{2\ell\,\abs{f(x)-g(x)}}-1}\,.
\end{align}
\end{lemma}
\begin{proof}
By definition, we can write
\begin{align*}
    \frac{\mu_f(x)}{\mu_g(x)}=e^{g(x)-f(x)}\En_{\mu_f}[e^{f-g}]\,.
\end{align*}
Therefore, we have
\begin{align*}
    1+\Dren[\ell]{\mu_f}{\mu_g}=&~\En_{\mu_f}\prn[\big]{\frac{\mu_f}{\mu_g}}^{\ell-1}
    = \prn*{\En_{\mu_f}[e^{f-g}]}^{\ell-1}\cdot \En_{\mu_f}[e^{(\ell-1)(g-f)}] \\
    \leq&~ \prn[\big]{\En_{\mu_f}e^{(\ell-1)\abs{f-g}}}^2
    \leq \En_{\mu_f}[e^{2\ell\abs{f-g}}]\,.
\end{align*}
Similarly,
\begin{align*}
    1+\Dren[\ell]{\mu_g}{\mu_f}=&~\En_{\mu_f}\prn[\big]{\frac{\mu_g}{\mu_f}}^{\ell}= \prn*{\En_{\mu_f}[e^{f-g}]}^{-\ell}\cdot \En_{\mu_f}[e^{\ell(f-g)}] 
    \leq \En_{\mu_f}[e^{-\ell(f-g)}]\cdot \En_{\mu_f}[e^{\ell(f-g)}]\\
    \leq&~ \prn[\big]{\En_{\mu_f}e^{\ell\abs{f-g}}}^2
    \leq \En_{\mu_f}[e^{2\ell\abs{f-g}}]\,.
\end{align*}
Combining both inequalities completes the proof.
\end{proof}

\section{Analysis of FORS}

\subsection{Proof of Theorem~\ref{thm:fors}}

\begin{proof}[Proof of~\cref{thm:fors}]
    Let $x_{\rm out}$ denote the output of~\cref{alg:fors}.
    On a given iteration of the algorithm, conditional on the draw $x\sim q$, the probability that the algorithm terminates on that iteration is
    \begin{align*}
        a(x)
        &= \E\Bigl[\prod_{j=1}^J \frac{B+W_j}{2B}\Bigm\vert x\Bigr]
        =\E\Bigl[\Bigl(\frac{B+\E[W_1\mid x]}{2B}\Bigr)^J\Bigr]\\
        &= \sum_{J\geq 0} \frac{e^{-2B}\,(2B)^J}{J!} \,\Bigl(\frac{B+\E[W_1\mid x]}{2B}\Bigr)^J
        = e^{\E[W_1\mid x]-B}\,.
    \end{align*}
    Thus, the acceptance probability of a given iteration is $A \deq \int a(x)\,q(dx)$.
    Then,
    \begin{align*}
        \Pr(x_{\rm out} \in dx)
        &= \sum_{i=1}^\infty \Pr(x_{\rm out} \in dx, \, \text{\cref{alg:fors} succeeds on iteration}~i) \\
        &= \sum_{i=1}^\infty (1-A)^{i-1}\,q(dx)\,a(x)
        = \frac{q(dx)\,a(x)}{A}\,.
    \end{align*}
    For the second statement, let $J_1,J_2,J_3,\dotsc$ be an i.i.d.\ sequence of $\Poi(2B)$ random variables.
    The probability that the number $N$ of draws from $\mathcal W$ exceeds $m$ is $\Pr(N\geq m)=\Pr(\sum_{i=1}^I J_i \ge m)$, where $I$ is the number of iterations of the algorithm.
    For any $i_0$, we take $m = (2+c)\,i_0 B$ and bound
    \begin{align*}
        \Pr\Bigl(\sum_{i=1}^I J_i \ge m\Bigr)
        &\le \Pr(I > i_0) + \Pr\Bigl(\sum_{i=1}^{i_0} J_i \ge (2+c)\,i_0 B\Bigr)\,.
    \end{align*}
    Note that if any $J_i = 0$, then $I \le i$, so the first term is bounded by $(1-e^{-2B})^{i_0} \le \exp(-e^{-2B}\,i_0)$.
    The second term, using a concentration bound for the Poisson random variable, is bounded by $\exp(-\frac{c^2}{c+2}\,i_0 B)$.
    Setting $i_0 = e^{2B} \log(2/\delta)$ and $c=e^{-2B}$ suffices to bound both terms by $\delta/2$.

    Note that this also implies $\En[\exp\prn{N/(CBe^{2B})}]\leq 2$ for an absolute constant $C>0$, where $N=\sum_{i=1}^I J_i$ is the total number of draws from $\cW$.
    So, the last statement follows immediately.
\end{proof}

\scedit{
\subsection{Clipping error}

In order to control the error incurred by FORS, we often use the following argument.

\begin{lemma}\label{lem:fors_clip}
    Let $p(x) \propto q(x)\,e^{\E W_x}$ and $\phat(x) \propto q(x)\,e^{\E\pclip{W_x}}$.
    Then, for any $\ell > 1$,
    \begin{align*}
        \Dsys[\ell]{p}{\phat}
        \le e^{2B}\En_{x\sim q}[e^{2\ell\,(|W_x| - B)_+} - 1]\,.
    \end{align*}
\end{lemma}
\begin{proof}
    By~\cref{lem:Renyi-to-diff},
    \begin{align*}
        \Dsys[\ell]{p}{\phat}
        &\le \En_{x\sim \phat}[e^{2\ell\En \trunc[\B]{W_x}} - 1]
        \le e^{2B} \En_{x\sim q}[e^{2\ell\En \trunc[\B]{W_x}} - 1]
    \end{align*}
    where we recall $\trunc[\B]{y}\deq(\abs{y}-\B)_+=|y-\pclip{y}|$ and we used $\frac{d\phat}{dq} \le e^{2B}$.
    The result follows from the convexity of $w\mapsto e^w$.
\end{proof}
}

\section{Proofs for Section~\ref{sec:gaussian_tilts}}\label{appdx:proof-gaussian-tilts}

In the following, we prove a slightly stronger version of \cref{thm:gaussian_tilt}.

\begin{theorem}\label{thm:gaussian_tilt_full}
Suppose that \cref{ass:holder} holds, and \cref{alg:fors} is instantiated as in \cref{thm:gaussian_tilt}. Let $B=\Theta(1)$, $\ell\geq 2$, $\delta\in(0,\frac12]$, and
\begin{align*}
    \frac{1}{\eta^{1+s}}\gg \beta_s^2\,\prn[\Big]{d^s\,(\ell+\log(1/\delta))+\frac{s}{d^{1-s}}(\ell^2+\log^2(1/\delta))}\,.
\end{align*}
Then, the law $\wh\nu$ of~\cref{alg:fors} satisfies
\begin{align*}
    \Dsys[\ell]{\nuhat}{\nu} \leq \delta\,.
\end{align*}
\end{theorem}

\begin{proof}[Proof of~\cref{thm:gaussian_tilt_full}]
We recall that $q=\normal{\xhat,\eta\Id}$, where we denote $\widehat x \deq x_0 - \eta\nabla f(x_+)$.

By~\cref{thm:fors}, the output of~\cref{alg:fors} with the specified choices samples from $\widehat\nu$, such that
\begin{align*}
    \log \wh\nu(x)-\log q(x)=\const+\En_{z\sim P,\,r\sim\unif([0,1])}\pclip{\tri{\gamp{\lr}, \nabla f(\xp)-\nabla f(\gam{\lr})}}\,.
\end{align*}

In the following, we denote $\En_{z,r}[\cdot]\ldef \En_{z\sim P,\,r\sim\unif([0,1])}[\cdot]$ and define
\begin{align*}
    W_{r,z,x}
    \ldef\tri{\gamp{\lr}, \nabla f(\xp) - \nabla f(\gam{\lr})}\,.
\end{align*}
\scedit{By~\cref{lem:fors_clip}, for $\ell > 1$,}
\begin{align}\label{pfeq:RGO-Renyi}
\begin{aligned}
    \Dsys[\ell]{\nuhat}{\nu}
    \leq&~ e^{2B}\En_{x\sim q}\En_{z,r}\brk{ e^{2\ell\trunc{W_{r,z,x}}}-1 }\,.
\end{aligned}
\end{align}

In the following, we proceed to prove the following claims.

\paragraph{Claim 1}
It holds that for any fixed $\lr\in[0,1]$ and $\lambda^2\leq \frac{d^{1-s}}{12s\beta_s^2\eta^{1+s}}$,
\begin{align*}
    \log \En_{x\sim q,\,z\sim P}\exp\prn*{\lambda \abs{W_{r,z,x}}}\leq 10 d^s \eta^{1+s}\lambda^2\beta_s^2\,.
\end{align*}

\paragraph{Proof of Claim 1}
Recall that $\alr=\sin(\pi\lr/2), ~~\blr=\cos(\pi\lr/2)$, 
\begin{align*}
    \gam{\lr}=\xhat+\alr (x-\xhat) + \blr z\,, \qquad
    \gamp{\lr}=\alrp (x -\xhat) + \blrp z\,.
\end{align*}
Hence, under $x\sim q=\normal{\xhat,\eta\Id}$ and $z\sim P=\normal{0,\eta\Id}$, $[\gam{\lr};\gamp{\lr}]$ are jointly distributed as
\begin{align*}
    [\gam{\lr};\gamp{\lr}]\sim \normal{\begin{bmatrix} \xhat \\ 0 \end{bmatrix}, \begin{bmatrix} \eta\Id &  \\ & (\pi/2)^2\eta\Id \end{bmatrix}
    }\,.
\end{align*}
Hence, as long as $3\eta\lambda^2\beta_s^2\leq \frac{d^{1-s}}{4s\eta^s}$,
\begin{align*}
    \En_{x\sim q,\,z\sim P}\exp\prn*{\lambda \abs{W_{r,z,x}}}
    =&~\En_{Z_1,Z_2\sim \normal{0,\eta\Id}}\exp\prn*{\lambda\abs{\tri{Z_1\pi/2, \nabla f(\xhat+Z_2)-\nabla f(\xp)}}} \\
    \leq&~ 2\En_{Z_2\sim \normal{0,\eta\Id}}\exp\prn[\Big]{\frac{3}{2}\eta\lambda^2\,\nrm{\nabla f(\xhat+Z_2)-\nabla f(\xp)}^2} \\
    \leq&~ 2\En_{Z_2\sim \normal{0,\eta\Id}}\exp\prn[\Big]{\frac{3}{2}\eta\lambda^2\beta_s^2\,\nrm{\xhat-\xp+Z_2}^{2s}} \\
    \leq&~ 2\exp\prn*{3\eta\lambda^2\beta_s^2\nrm{\xhat-\xp}^{2s}}\En_{W\sim \normal{0,\eta\Id}}\exp\prn*{3\eta\lambda^2\beta_s^2\,\nrm{W}^{2s}} \\
    \leq&~ 2\exp\prn[\big]{3\eta\lambda^2\beta_s^2\,\nrm{\xhat-\xp}^{2s}+6d^s\eta^{1+s}\lambda^2\beta_s^2}
    \leq 2\exp\prn[\big]{10d^s\eta^{1+s}\lambda^2\beta_s^2}\,,
\end{align*}
where the second line uses \cref{lem:Gaussian-vMGF}, and the last line uses \cref{lem:Gaussian-s}. This completes the proof of Claim 1.%

Finally, we use Claim 1 to prove the following claim, from which \cref{thm:gaussian_tilt_full} follows immediately.

\paragraph{Claim 2}
Suppose that 
\begin{align*}
    \frac{1}{\eta^{1+s}}\geq 64\beta_s^2\,\prn[\Big]{\ell B^{-1} d^s+\frac{s\ell^2}{d^{1-s}}}\,.
\end{align*}
Then for any $\lr\in[0,1]$, it holds that
\begin{align*}
    \En_{x\sim q,\, z\sim P}\brk{e^{2\ell\trunc{W_{r,z,x}}}-1}
    \leq&~ 2\exp\prn[\Big]{-\min\crl[\Big]{\frac{\B^2}{40\beta_s^2d^s\eta^{1+s}},\, \frac{Bd^{(1-s)/2}}{8\sqrt{s}\beta_s\eta^{(1+s)/2}}}}\,.
\end{align*}

\paragraph{Proof of Claim 2}
Using Claim 1, for any fixed $\lr\in[0,1]$ and $2\ell\leq \lambda\leq \frac{d^{(1-s)/2}}{4\sqrt{s}\beta_s\eta^{(1+s)/2}}$, we can upper bound
\begin{align*}
    \En_{x\sim q,\, z\sim P}\brk{e^{2\ell\trunc{W_{r,z,x}}}-1}
    \leq&~ \En_{x\sim q,\, z\sim P}\brk{e^{\lambda \trunc{W_{r,z,x}}}-1}
    \leq e^{-\lambda B} \En_{x\sim q,\, z\sim P}\brk{e^{\lambda\abs{W_{r,z,x}}}} \\
    \leq&~ 2\exp\prn*{10d^s\eta^{1+s}\lambda^2\beta_s^2-B\lambda}\,.
\end{align*}
Therefore, the desired upper bound follows from setting
\begin{align*}
    \lambda=\min\crl[\Big]{\frac{d^{(1-s)/2}}{4\sqrt{s}\beta_s\eta^{(1+s)/2}},\, \frac{B}{20\beta_s^2d^s\eta^{1+s}}}\geq 2\ell\,.
\end{align*}

Combining Claim 2 above with \cref{pfeq:RGO-Renyi}, we can deduce the desired upper bound.
\end{proof}

\section{Structural properties of the diffusion process}\label{appdx:structure}
In this section, we establish some crucial properties of DDPM\@.
In the following, we consider the forward SDE $dY_t=dB_t$ with $Y_0\sim \pbar$, and denote $q_t\deq \pbar*\normal{0,t\Id}$ to be the marginal density of $Y_t$. Then, $\nabla \log q_t(x)=\frac{\En[Y_0\mid Y_t=x]-x}{t}$, and we denote $m_t(x)\deq \En[Y_0\mid Y_t=x]$, $\cov_t(x)=\En[(Y_0-m_t(x))\,(Y_0-m_t(x))^\top\mid Y_t=x]$ be the posterior mean and covariance function. We know $\nabla_x m_t(x)=\frac{1}{t}\,\cov_t(x)$.

We establish two main types of results.
The first type of result concerns the coverage of the DDPM kernel w.r.t.\ the reverse SDE and vice versa.
Namely, we are interested in understanding the distribution $q_{\tau,\eta}(\cdot\mid y)$ of $Y_{\tau-\eta}\mid Y_\tau=y$ and the distribution $\wb{q}_{\tau,\eta}(\cdot\mid y)=\normal{y+\eta \nabla \log q_\tau(y), \etabar\Id}$, where $\etabar=\eta-\frac{\eta^2}{\tau}$.

For distributions $P,Q$, we define the following divergence for $C\geq 1$:
\begin{align*}
    \Dcov[C]{P}{Q}\deq \En_Q\prn[\Big]{\frac{dP}{dQ}-C}_+\leq P\prn[\Big]{\frac{dP}{dQ}\geq C}\,.
\end{align*}
Note that $\Dcov[C]{\cdot}{\cdot}$ is a $f$-divergence with $f(x)=(x-C)_+$.
The reason for introducing $\Dcov[C]{\cdot}{\cdot}$ is the following lemma. 
\begin{lemma}\label{lem:Dcov}
Suppose that $P,Q$ are distribution over $\cX$, and $C\geq 1$. For any function $F:\cX\to [0,1]$, it holds that $\En_P[F]\leq C\En_Q[F]+\Dcov[C]{P}{Q}$.
\end{lemma}

We establish the following coverage estimates.

\begin{proposition}\label{prop:DDPM-coverage}
Suppose that $\eta\leq \frac{\tau}{4}$. Then
\begin{align*}
    \En\Dcov[e]{q_{\tau,\eta}(\cdot\mid Y_\tau)}{\wb{q}_{\tau,\eta}(\cdot\mid Y_\tau)}
    \leq&~ 2\delta+2\max_{t\in[\tau-\eta, \tau]}\En\prn*{ M\nrmF{\nabla m_{t}(Y_t)}^2-1 }_+\,,
\end{align*}
where $M=\frac{32\eta^2\log^2(e/\delta)}{\tau^2}$. Further, for any $L\geq 1$, $t\in[\tau-\eta,\tau]$, it holds that
\begin{align*}
    \PP\prn*{ \nrm{m_t(Y_t)-m_\tau(Y_\tau)}^2 \geq 8L^2\eta\log(e/\delta) }
    \leq&~ \delta+2\max_{s\in[t,\tau]} \En_P \prn*{L^{-2}\nrmF{\nabla m_{s}(Y_s)}^2-1}_+\,.
\end{align*}
\end{proposition}

Next, we prove the following proposition provides the reverse bound of \cref{prop:DDPM-coverage}. The proof is inspired by \citet{jiao2025optimal}. Note that
\begin{align*}
    \partial_t m_t(x)=&~\frac{1}{2t^2} \En[ (Y_0-m_t(x))\,\nrm{Y_0-x}^2 \mid Y_t=x ] \\
    =&~\frac{1}{2t^2}\,\prn*{ \En[ (Y_0-m_t(x))\,\nrm{Y_0-m_t(x)}^2 \mid Y_t=x ]+2\cov_t(x)\,(m_t(x)-x) },
\end{align*}
and we define $k_t(x)\deq \En[ (Y_0-m_t(x))\,\nrm{Y_0-m_t(x)}^2 \mid Y_t=x ]$. 

\begin{proposition}\label{prop:coverage-backward}
We denote $\wt{q}_{\tau,\eta}$ to be the distribution of $Y'\sim \wb{q}_{\tau,\eta}(\cdot\mid Y_\tau)$ under $Y_\tau\sim q_\tau$. Suppose that $\tau\geq 2\eta$. Then
\begin{align*}
    \Dcov[e]{\wt{q}_{\tau,\eta}}{q_{\tau-\eta}}\leq \delta+\max_{t\in[\tau-\eta,\tau]}\En\prn*{M \nrm{k_t(Y_t)}-1 }_++\En\prn*{M \nrm{\cov_t(Y_t)\,(m_t(Y_t)-Y_t)}-1 }_+\,,
\end{align*}
where $M=\frac{16\sqrt{\log(1/\delta)}\eta^{3/2}}{(\tau-\eta)^{3}}$.
\end{proposition}

Bounding the terms that appear in these results requires estimates along the diffusion process, and our second main type of result is to prove such estimates which scale with the intrinsic dimension.

\begin{corollary}\label{cor:MGF-intrinsic}
There is an absolute constant $C>0$ such that the following holds for any $\tau>0$.

(1) 
\begin{align*}
    \En\brk[\Big]{ \exp\prn[\Big]{ \frac{\tr(\cov_\tau(Y_\tau))}{C\tau} }}\leq \En\brk[\Big]{ \exp\prn[\Big]{ \frac{\nrm{Y_0-m_\tau(Y_\tau)}^2}{C\tau} }}\leq e^{\dim_\tau(\pbar)}\,.
\end{align*}

(2)
\begin{align*}
    \En\brk[\Big]{ \exp\prn[\Big]{ \frac{1}{C\tau}\,\nrm{Y_\tau-m_\tau(Y_\tau)}_{\cov_\tau(Y_\tau)} }}\leq e^{\dim_\tau(\pbar)}\,.
\end{align*}

(3) For any $0<s\leq \tau$,
\begin{align*}
    \En\brk[\Big]{ \exp\prn[\Big]{ \frac{1}{C\tau}\nrm{m_s(Y_s)-m_\tau(Y_\tau)}^2 } }
    \leq&~ e^{\dim_\tau(\pbar)}\,.
\end{align*}
\end{corollary}

The following two subsections are devoted to proving these results (among others).
The remainder of the section focuses on applications of these results.

\subsection{Coverage of the DDPM distribution}

Combining \cref{prop:DDPM-coverage} with the fact that the sub-exponential norm of $\nrmF{\nabla m_t(Y_t)}\leq \tr(\nabla m_t(Y_t))$ is bounded by $O(\dim_\tau(\pbar))$ for any $t\in[\tau-\eta,\tau]$, we have the following proposition.
\begin{corollary}\label{cor:DDPM-coverage}
Suppose that $\LF\geq 1$.
Then as long as 
\begin{align*}
    \frac{\tau}{\eta}\gg \LF\log(1/\delta)+\log^2(1/\delta),
\end{align*}
it holds that
\begin{align*}
    \En\Dcov[e]{q_{\tau,\eta}(\cdot\mid Y_\tau)}{\wb{q}_{\tau,\eta}(\cdot\mid Y_\tau)}\leqsim \dim_\tau(\pbar)^2(\delta+\max_{t\in[\tau-\eta,\tau]}\PP\prn*{\nrmF{\nabla m_t(Y_t)}\geq \LF}),
\end{align*}
and for any $t\in[\tau-\eta,\tau]$,
\begin{align*}
    \PP\prn*{ \nrm{m_t(Y_t)-m_\tau(Y_\tau)}^2 \geq 8L^2\eta\log(e/\delta) }
    \leqsim \dim_\tau(\pbar)^2(\delta+\max_{t\in[\tau-\eta,\tau]}\PP\prn*{\nrmF{\nabla m_t(Y_t)}\geq \LF}).
\end{align*}
\end{corollary}

\begin{corollary}\label{cor:coverage-backward}
For any parameter $\Lop\geq 1$ and $\delta\in(0,1)$, as long as 
\begin{align}\label{eq:coverage-eta}
    \frac{\tau}{\eta}\gg \min\crl{ \Lop^{2/3}d^{1/3}, \Lop^{1/3}\dim_\tau(\pbar)^{2/3} } \log^{1/3}(1/\delta)+\log^2(1/\delta),
\end{align}
it holds that
\begin{align*}
    \Dcov[e]{\wt{q}_{\tau,\eta}}{q_{\tau-\eta}}\leqsim \dim_\tau(\pbar)^2(\delta+\max_{t\in[\tau-\eta,\tau]}\PP(\nrmop{\nabla m_t(Y_t)}\geq \Lop)).
\end{align*}
\end{corollary}

\begin{proof}[\pfref{prop:DDPM-coverage}]
 We consider the backward SDE.
Starting from a point $Y_\tau\sim q_\tau$, consider the following SDE (with $\Ybar_0=Y_\tau$):
\begin{align}
    d\Ybar_s=\nabla \log q_{\tau-s}(\Ybar_s)\,ds+d\beta_s\,, \qquad s\in[0,\eta]\,.
\end{align}
Let $P$ be the law of the above SDE. We consider $\mu_s\deq \frac{1}{\tau-s}\,(m_{\tau-s}(\Ybar_s)-m_\tau(\Ybar_0))$ and $\beta'_t\deq \beta_t+\int_0^t \mu_s\,ds$. Then, by Girsanov's theorem,\footnote{Formally, we need to check Novikov's condition that $\En_P \exp\prn*{ \int_0^\eta \nrm{\mu_t}^2 dt }<+\infty$ holds. This is guaranteed by \eqref{eq:MGF-finiteness} (\cref{lem:MGF-trivial}) as long as $\tau\geq 4\eta$. } $(\beta'_t)_{t\in [0,\eta]}$ is a Brownian motion under $Q$, where
\begin{align*}
    \frac{dP}{dQ}=\exp\prn[\Big]{ \int_0^\eta \mu_t \,d\beta_t+\frac12\int_0^\eta \nrm{\mu_t}^2\, dt }\,.
\end{align*}
Note that under $P$, we have $\En_P\exp\prn[\big]{ \lambda \int_0^\eta \mu_t\, d\beta_t-\frac{\lambda^2}{2}\int_0^\eta \nrm{\mu_t}^2\, dt }\leq 1$ for any $\lambda\in\RR$, and hence we can set $\lambda=\frac{\log(1/\delta)}{2}$ to prove
\begin{align*}
    \Dcov[e]{P}{Q}=&~\En_Q\prn[\Big]{ \frac{dP}{dQ}-e }_+\leq P\prn[\Big]{ \frac{dP}{dQ}\geq e }
    \leq P\prn[\Big]{ \int_0^\eta \nrm{\mu_t}^2\, dt\geq \frac{1}{1+2\log(1/\delta)} }+\delta\,.
\end{align*}
Note that $Z_t\deq m_{\tau-t}(\Ybar_t)-m_\tau(\Ybar_0)$ is a martingale, and hence $dZ_t=\nabla m_{\tau-t}(\Ybar_t)\, d\beta_t$, i.e., $Z_t=\int_0^t \nabla m_{\tau-s}(\Ybar_s) \, d\beta_s$. Applying \cref{lem:Ito-Freedman} to $(Z_t)_{t\in [0,\eta]}$ gives that for any $R>0$,
\begin{align*}
    P\prn*{ \En_{t\sim \unif([0,\eta])} \nrm{Z_t}^2\leq 4R\log(e/\delta)} 
    \leq \delta+2P\prn[\Big]{ \int_0^\eta \nrmF{\nabla m_{\tau-s}(\Ybar_s)}^2\,ds\geq R }\,.
\end{align*}
Choosing $R=\frac{(\tau-\eta)^2}{16\eta\log^2(e/\delta)}$ and combining the inequalities above gives
\begin{align*}
    \Dcov[e]{P}{Q}\leq&~ 
    \delta+ P\prn[\Big]{ \frac{1}{(\tau-\eta)^2}\int_0^\eta \nrm{Z_t}^2\, dt\geq \frac{1}{1+2\log(1/\delta)} }\\
    \leq&~ 2\delta+ 2P\prn[\Big]{ \int_0^\eta \nrmF{\nabla m_{\tau-s}(\Ybar_s)}^2\,ds\geq R } \\
    \leq&~ 2\delta+2\En_P\prn[\Big]{ R^{-1}\int_0^\eta \nrmF{\nabla m_{\tau-s}(\Ybar_s)}^2\,ds-1 }_+\\
    \leq&~ 2\delta+2\max_{t\in[\tau-\eta, \tau]}\En_P\prn*{ R^{-1}\eta\, \nrmF{\nabla m_{t}(Y_t)}^2-1 }_+
\end{align*}
where 
the last line uses convexity of $w\mapsto (w-1)_+$ and $\Ybar_t \eqd Y_{\tau-t}$.

Finally, we note that under $P$, marginally $\Ybar_{\eta}\mid Y_\tau\sim q_{\tau,\eta}(\cdot\mid Y_\tau)$; Under $Q$, marginally $\Ybar_\eta\mid Y_\tau\sim \normal{Y_\tau+\eta\nabla \log q_\tau(Y_\tau), \etabar\Id}$. This completes the proof of the first inequality.

Similarly, our argument also implies that for any $R'>0$,
\begin{align*}
    P\prn*{ \nrm{Z_t}^2 \geq 4R'\log(e/\delta) }
    \leq&~ \delta+2P\prn[\Big]{ \int_0^t \nrmF{\nabla m_{\tau-s}(\Ybar_s)}^2\,ds\geq R' } \\
    \leq&~ \delta+2\max_{s\in[t,\tau]} \En_P \prn*{(2t/R')\,\nrmF{\nabla m_{s}(Y_s)}^2-1}_+\,.
\end{align*}
\end{proof}

\begin{proof}[\pfref{prop:coverage-backward}]
We consider the backward ODE.
Starting from a point $Y_\tau\sim q_\tau$, consider the following ODE with $\Ybar_0=Y_\tau$:
\begin{align}
    \frac{d}{ds} \Ybar_s=\frac12\nabla \log q_{\tau-s}(\Ybar_s).
\end{align}
Under $Y_\tau\sim q_\tau$, marginally $\Ybar_s\sim q_{\tau-s}$. In particular, we consider $r=2\eta-\frac{\eta^2}{\tau}$. Because $Y_{\tau-\eta}\mid Y_{\tau-r}\sim \normal{Y_{\tau-r}, \etabar\Id}$ and $Y_{\tau-r}\eqd \Ybar_{r}$, it holds that
\begin{align*}
    q_{\tau-\eta}=&~ \En_{Y_{\tau-r}} \normal{Y_{\tau-r}, \etabar\Id}
    = \En_{\Ybar_r} \normal{\Ybar_r, \etabar\Id}.
\end{align*}
On the other hand, by definition, $\wt{q}_{\tau,\eta}=\En_{Y_\tau}  \normal{Y_\tau+\eta \nabla \log q_\tau(Y_\tau), \etabar\Id}$. Therefore, we can bound
\begin{align*}
    \Dcov[e]{\wt{q}_{\tau,\eta}}{q_{\tau-\eta}}
    \leq&~ \En_{\Ybar_0=Y_\tau\sim q_\tau} \Dcov[e]{\normal{\Ybar_0+\eta \nabla \log q_\tau(\Ybar_0), \etabar\Id}}{\normal{\Ybar_r, \etabar\Id}} \\
    \leq&~ \delta+ \bbP_{\Ybar_0=Y_\tau\sim q_\tau}\prn*{ 8\log(1/\delta)\nrm{ \Ybar_r-\Ybar_0-\eta \nabla \log q_\tau(\Ybar_0) }^2\geq \etabar }
\end{align*}
where the first inequality uses the joint convexity of $\Dcov[e]{\cdot}{\cdot}$ and the second inequality uses \cref{lem:Gaussian-Dcov}.

Note that we can rewrite
\begin{align*}
    \frac{d}{ds} \frac{\Ybar_s}{\sqrt{\tau-s}}=\frac1{2(\tau-s)^{3/2}}m_{\tau-s}(\Ybar_s),
\end{align*}
and hence
\begin{align*}
    \frac{\Ybar_{r}}{\sqrt{\tau-r}}=\frac{\Ybar_0}{\sqrt{\tau}}+\int_0^r \frac1{2(\tau-s)^{3/2}}m_{\tau-s}(\Ybar_s)ds.
\end{align*}
Then, it holds that
\begin{align*}
    \Ybar_r-\Ybar_0-\eta \nabla \log q_\tau(\Ybar_0)
    =&~\frac{\tau-\eta}{2\sqrt{\tau}}\int_0^r \frac1{(\tau-s)^{3/2}}(m_{\tau-s}(\Ybar_s)-m_\tau(\Ybar_0))ds \\
    =&~\frac{\tau-\eta}{2\sqrt{\tau}}\int_{0\leq s\leq s'\leq r} \frac1{(\tau-s')^{3/2}}\partial_s (m_{\tau-s}(\Ybar_s))dsds' \\
    =&~\frac{\tau-\eta}{\sqrt{\tau}}\int_{0}^r \brk*{\frac{1}{\sqrt{\tau-r}}-\frac{1}{\sqrt{\tau-s}}} \partial_s (m_{\tau-s}(\Ybar_s))ds.
\end{align*}
Therefore,
\begin{align*}
    \nrm{\Ybar_r-\Ybar_0-\eta \nabla \log q_\tau(\Ybar_0)}
    \leq \frac{r}{\tau-\eta}\int_0^r \nrm{\partial_s (m_{\tau-s}(\Ybar_s))} ds.
\end{align*}
A direct calculation yields
\begin{align*}
    \partial_t (m_t(\Ybar_{\tau-t}))
    =&~(\partial_t m_t)(\Ybar_{\tau-t})-\nabla m_t(\Ybar_{\tau-t})\cdot \frac{d\Ybar_s}{ds}\Big|_{s=\tau-t} \\
    =&~ \frac{1}{2t^2}\brk*{ k_t(\Ybar_{\tau-t})+\cov_t(\Ybar_{\tau-t})(m_t(\Ybar_{\tau-t})-\Ybar_{\tau-t}) }.
\end{align*}
Therefore,
\begin{align*}
    \nrm{\Ybar_r-\Ybar_0-\eta \nabla \log q_\tau(\Ybar_0)}
    \leq \frac{r}{2(\tau-\eta)^3}\int_0^r (\nrm{k_t(\Ybar_{\tau-t})}+\nrm{\cov_t(\Ybar_{\tau-t})(m_t(\Ybar_{\tau-t})-\Ybar_{\tau-t})}) dt.
\end{align*}
Then, by combining the inequalities above and applying Markov's inequality, we can bound
\begin{align*}
    &~\Dcov[e]{\wt{q}_{\tau,\eta}}{q_{\tau-\eta}}
    \leq \delta+ \bbP_{\Ybar_0=Y_\tau\sim q_\tau}\prn*{ \sqrt{8\log(1/\delta)/\etabar}\nrm{ \Ybar_r-\Ybar_0-\eta \nabla \log q_\tau(\Ybar_0) }\geq 1 } \\
    \leq&~ \delta+ \bbP_{\Ybar_0\sim q_\tau}\prn*{ \sqrt{8\log(1/\delta)/\etabar}\cdot \frac{r}{2(\tau-\eta)^3}\int_0^r (\nrm{k_t(\Ybar_{\tau-t})}+\nrm{\cov_t(\Ybar_{\tau-t})(m_t(\Ybar_{\tau-t})-\Ybar_{\tau-t})}) dt \geq 1 } \\
    \leq&~ \delta+\En_{\Ybar_0\sim q_\tau}\prn*{C \int_0^r \nrm{k_t(\Ybar_{\tau-t})}dt-1 }_++\En_{\Ybar_0\sim q_\tau}\prn*{C \int_0^r \nrm{\cov_t(\Ybar_{\tau-t})(m_t(\Ybar_{\tau-t})-\Ybar_{\tau-t})}dt-1 }_+.
\end{align*}
where we denote $C=\sqrt{8\log(1/\delta)/\etabar}\cdot \frac{r}{(\tau-\eta)^3}$. The proof is then completed by the convexity of $w\mapsto (w-1)_+$.
\end{proof}

\begin{proof}[\pfref{cor:coverage-backward}]
We first note that by \cref{cor:MGF-intrinsic}, the sub-exponential norm of $\nrm{m_t(Y)-Y_t}_{\cov_t(Y_t)}$ and $\nrmop{\cov_t(Y_t)}$ are bounded by $O(\dstar t)$. Hence, by \cref{lem:sub-exp-trunc}, we can choose $K_1=C_1(\tau\dstar\log(1/\delta))^{3/2}$ to  upper bound 
\begin{align*}
    \En\prn*{M \nrm{\cov_t(Y_t)(m_t(Y_t)-Y_t)}-1 }_+
    \leq K_1M(\delta+\PP\prn*{M \nrm{\cov_t(Y_t)(m_t(Y_t)-Y_t)}\geq 1}).
\end{align*}
Then, using \cref{cor:MGF-intrinsic} again, we also have
\begin{align*}
    \PP\prn*{ \nrm{\cov_t(Y_t)(m_t(Y_t)-Y_t)}\geq c_1t\sqrt{\nrmop{\cov_t(Y_t)}}(\dstar+\log(1/\delta)) }\leq \delta.
\end{align*}
This immediately implies that
\begin{align*}
    \PP\prn*{M \nrm{\cov_t(Y_t)(m_t(Y_t)-Y_t)}\geq 1}
    \leq \delta+\PP\prn*{ t^{-1}\nrmop{\cov_t(Y_t)}\geq \frac{c_2\tau^4}{\eta^3(\dstar+\log(1/\delta))^2\log(1/\delta)} }.
\end{align*}
Alternatively, by \cref{lem:MGF-trivial}, it holds that $\nrm{m_t(Y)-Y_t}^2\leq O(t(d+\log(1/\delta)))$ with probability at least $1-\delta$. Therefore, it holds that
\begin{align}\label{pfeq:cov-inner}
    \PP\prn*{M \nrm{\cov_t(Y_t)(m_t(Y_t)-Y_t)}\geq 1}
    \leq \delta+\PP\prn*{ t^{-1}\nrmop{\cov_t(Y_t)}\geq c_3\sqrt{\frac{\tau^3}{\eta^3(d+\log(1/\delta))\log(1/\delta)}} }.
\end{align}
Therefore, for any parameter $\Lop\geq 1$, as long as 
\begin{align}\label{pfeq:coverage-eta}
    \frac{\tau}{\eta}\gg \min\crl{ \Lop^{2/3}d^{1/3}, \dstar^{2/3}\Lop^{1/3} } \log^{1/3}(1/\delta)+\log(1/\delta),
\end{align}
it holds that for any $t\in[\tau-r,\tau]$,
\begin{align*}
    \En\prn*{M \nrm{\cov_t(Y_t)(m_t(Y_t)-Y_t)}-1 }_+\leqsim \dstar^{3/2}(\delta+\PP\prn*{ t^{-1}\nrmop{\cov_t(Y_t)}\geq \Lop}).
\end{align*}

Next, it remains to bound $k_t(x)$. We invoke the following lemma.
\begin{lemma}\label{lem:kt-to-variance}
    It holds that
\begin{align*}
    \nrm{k_t(x)}^2
    \leq \nrmop{\cov_t(x)} \Var[ \nrm{Y_0-m_t(x)}^2 \mid Y_t=x ].
\end{align*}
Alternatively, it holds that
\begin{align*}
    \nrm{k_t(x)}
    \leq 2\nrm{\cov_t(x)(m_t(x)-x)}+\sqrt{\nrmop{\cov_t(x)} \Var[ \nrm{Y_t-Y_0}^2 \mid Y_t=x ]}.
\end{align*}
\end{lemma}

In the following, we define $Q_t(x)\deq \Var[ \nrm{Y_0-m_t(x)}^2 \mid Y_t=x ]$. Note that $Q_t(Y_t)\leq \En[ \nrm{Y_0-m_t(\scedit{Y_t})}^4 \mid Y_t ]$, and hence by \cref{cor:MGF-intrinsic}, we can bound $\En \exp\prn*{ c_4\sqrt{Q_t(Y_t)}/t }\leq e^{\dstar}$ for a sufficiently small constant $c_4>0$.
\sccomment{Can't we apply this argument to $k_t$ directly, by saying that $k_t(Y_t) \le \E[\|Y_0 - m_t(Y_t)\|^3\mid Y_t]$? Perhaps bypasses the need for the lemma above?}\fccomment{we need this lemma for sublinear dependence on $\dstar$ or $d$...}
Therefore, the sub-exponential norm of $\sqrt{Q_t(Y_t)}$ is bounded by $O(t\dstar)$, and hence 
by \cref{lem:sub-exp-trunc}, we can choose $K_2=C_2(\tau\dstar\log(1/\delta))^{3/2}$ to  upper bound 
\begin{align*}
    \En\prn*{M \nrm{k_t(Y_t)}-1 }_+
    \leq K_2M(\delta+\PP\prn*{M \nrm{k_t(Y_t)}\geq 1}).
\end{align*}
Furthermore, we know that $\PP(\sqrt{Q_t(Y_t)}\geq c_5t(\dstar+\log(1/\delta)))\leq \delta$, and hence
\begin{align*}
    \PP\prn*{M \nrm{k_t(Y_t)}\geq 1}
    \leq \delta+\PP\prn*{ t^{-1}\nrmop{\cov_t(Y_t)}\geq \frac{c_2\tau^4}{\eta^3(\dstar+\log(1/\delta))^2\log(1/\delta)} }.
\end{align*}
In addition, by \cref{lem:Gaussian-vMGF}, it holds that
\begin{align*}
    \En\exp\prn*{ \frac{1}{4t}\abs*{\nrm{Y_t-Y_0}^2-td} }\leq e^{d/4}.
\end{align*}
For random variable $A$, it holds that $\exp(\sqrt{\En A^2})\leq e^2\En \exp(\abs{A})$, and hence
\begin{align*}
    \En\exp\prn*{ \frac{1}{4t}\sqrt{\Var[ \nrm{Y_t-Y_0}^2 \mid Y_t]} }
    \leq&~ \En\exp\prn*{ \frac{1}{4t}\sqrt{\En[\prn{\nrm{Y_t-Y_0}^2-td}^2\mid Y_t]} } \\
    \leq&~ e^2\En\exp\prn*{ \frac{1}{4t}\abs*{\nrm{Y_t-Y_0}^2-td} }\leq e^{2+d}.
\end{align*}
Hence, using the second inequality of \cref{lem:kt-to-variance} and \cref{pfeq:cov-inner}, we can show that
\begin{align*}
    \PP\prn*{M \nrm{k_t(Y_t)}\geq 1}
    \leq \delta+\PP\prn*{ t^{-1}\nrmop{\cov_t(Y_t)}\geq c_6\sqrt{\frac{\tau^3}{\eta^3(d+\log(1/\delta))\log(1/\delta)}} }.
\end{align*}
Therefore, under \cref{pfeq:coverage-eta}, for any $t\in[\tau-r,\tau]$,
\begin{align*}
    \En\prn*{M \nrm{k_t(Y_t)}-1 }_+\leqsim \dstar^{3/2}(\delta+\PP\prn*{ t^{-1}\nrmop{\cov_t(Y_t)}\geq \Lop}).
\end{align*}
\end{proof}

\begin{proof}[\pfref{lem:kt-to-variance}]
Fix $x\in\RR^d$. We define the random variable
\begin{align*}
    \Delta\deq  \nrm{Y_0-m_t(x)}^2-\En[ \nrm{Y_0-m_t(x)}^2 \mid Y_t=x ],
\end{align*}
and then $k_t(x)=\En[ (Y_0-m_t(x))\Delta\mid Y_t=x ]$.
For any vector $v\in\RR^d$, we can bound
\begin{align*}
    \tri{v,k_t(x)}^2
    =&~\prn*{ \En[ \tri{v,Y_0-m_t(x)}\Delta \mid Y_t=x ] }^2 \\
    \leq&~ \En[ \tri{v,Y_0-m_t(x)}^2\mid Y_t=x ] \cdot \En[ \Delta^2 \mid Y_t=x ] \\
    \leq&~ \nrmop{\cov_t(x)} \nrm{v}^2 \cdot \En[ \Delta^2 \mid Y_t=x ].
\end{align*}
Noting that $\En[ \Delta^2 \mid Y_t=x ]=\Var[ \nrm{Y_0-m_t(x)}^2 \mid Y_t=x ]$ completes the proof of the first inequality.

On the other hand, we can denote $Z_t=Y_t-Y_0$ and write
\begin{align*}
    k_t(x)=&~\En[ (Y_0-m_t(x))\nrm{Y_0-m_t(x)}^2 \mid Y_t=x ] \\
    =&~ \En[ (Y_0-m_t(x))(\nrm{Z_t}^2+2\tri{Z_t,m_t(x)-Y_t}+\nrm{m_t(x)-Y_t}^2) \mid Y_t=x ] \\
    =&~ \En[ (Y_0-m_t(x))\nrm{Z_t}^2 \mid Y_t=x ]
    +2\En[ (Y_0-m_t(x))\tri{Z_t,m_t(x)-Y_t} \mid Y_t=x ] \\
    =&~ \En[ (Y_0-m_t(x))\nrm{Z_t}^2 \mid Y_t=x ]
    +2\En[ (Y_0-m_t(x))\tri{m_t(x)-Y_0,m_t(x)-Y_t} \mid Y_t=x ] \\
    =&~ \En[ (Y_0-m_t(x))\nrm{Z_t}^2 \mid Y_t=x ]-2\cov_t(x)(m_t(x)-x).
\end{align*}
Repeating our argument above gives the second inequality.
\end{proof}

\subsection{Upper bounds with low intrinsic dimension}

The following proposition generalizes the argument of \citet{HuaWeiChe26LowDim} that bounds the posterior covariance matrix in terms of the intrinsic dimension. 

\begin{proposition}\label{prop:logp-adapt}
Recall that $(Y_0,Y_\tau)$ is jointly distributed as $Y_0\sim \pbar$, $Y_\tau\sim \normal{Y_0,\tau\Id}$. It holds that
\begin{align*}
    \En\brk*{ \exp\prn*{ \frac{\nrm{Y_0-m_\tau(Y_\tau)}^2}{160\tau} }}\leq 4e^{\dim_\tau(\pbar)}.
\end{align*}
\end{proposition}

\begin{proof}
We write $Q(\cdot\mid x)$ be the conditional distribution of $Y_0\mid Y_\tau=x$, i.e.,
\begin{align*}
    Q(x_0\mid x)=\frac{\exp\prn*{-\frac{\nrm{\xbar+V-Y_0}^2}{2\tau}} }{\En_{Y_0\sim \pbar}\brk*{ \exp\prn*{-\frac{\nrm{\xbar+V-Y_0}^2}{2\tau}} }}
    =\frac{\exp\prn*{-\frac{\nrm{\xbar-Y_0}^2+2\tri{V,\xbar-Y_0}}{2\tau}} }{\En_{Y_0\sim \pbar}\brk*{ \exp\prn*{-\frac{\nrm{\xbar-Y_0}^2+2\tri{V,\xbar-Y_0}}{2\tau}} }}\,.
\end{align*}
Our goal is to upper bound the moment
\begin{align*}
    \MM_{c}(x)\deq \En_{Y_0\sim Q(\cdot\mid x)}\exp\prn*{\frac{\nrm{Y_0-m_\tau(x)}^2}{c\tau}},
\end{align*}
where $m_\tau(x)\deq \En_{Y_0\sim Q(\cdot\mid x)}[Y_0]$ is the conditional mean. By triangle inequality, we know that for any $\xbar$ and
\begin{align}\label{pfeq:intrinsic-mgf-decomp}
    \MM_{8}(x)
    \leq \En_{Y_0\sim Q(\cdot\mid x)} \exp\prn*{\frac{\nrm{Y_0-\xbar}^2+\nrm{m_\tau(x)-\xbar}^2}{4\tau}}
    \leq \prn*{ \En_{Y_0\sim Q(\cdot\mid x)}\exp\prn*{\frac{\nrm{Y_0-\xbar}^2}{4\tau}} }^2.
\end{align}

Fix a $r$-covering of $\cXbar \deq \scedit{\supp(\pbar)}$:
\begin{align*}
    \cXbar\subseteq \bigcup_{i=1}^N B_i\,,
\end{align*}
where $N=N(\pbar,r)$ and $B_1,\dotsc, B_N$ are balls of radius $r$ and centers $z_1,\dotsc,z_N$. 

Fix any $\xbar$ and an index $j=j(\xbar)$ such that $\xbar\in B_j$.

For any $i\in[N]$, we consider
\begin{align*}
    g_i(V)\deq \En_{Y_0\sim \pbar} \indic\crl{Y_0\in B_i} \exp\prn*{-\frac{\nrm{\xbar-Y_0}^2}{4\tau}+\frac{\tri{V,z_i-Y_0}}{\tau}}.
\end{align*}
Note that 
\begin{align*}
    \En[g_i(V)]
    =&~\En_{Y_0\sim \pbar} \indic\crl{Y_0\in B_i} \exp\prn*{-\frac{\nrm{\xbar-Y_0}^2}{4\tau}+\frac{\nrm{z_i-Y_0}^2}{2\tau}} \\
    \leq&~ \En_{Y_0\sim \pbar} \indic\crl{Y_0\in B_i} \exp\prn*{\frac{2 r^2-\nrm{\xbar-Y_0}^2}{4\tau}} 
    \leq \exp\prn*{ \frac{2 r^2-(\nrm{z_i-\xbar}-r)_+^2}{4\tau} }.
\end{align*}

In addition, we define (recall $j=j(\xbar)$ is an index such that $\xbar\in B_j$)
\begin{align*}
    u(V)\deq \En_{Y_0\sim \pbar} \brk*{ \exp\prn*{-\frac{\tri{V,\xbar-Y_0}}{\tau}} \Big| Y_0\in B_j}.
\end{align*}
Note that
\begin{align*}
    \En_{Y_0\sim \pbar}\brk*{ \exp\prn*{-\frac{\nrm{\xbar-Y_0}^2+2\tri{V,\xbar-Y_0}}{2\tau}} }
    \geq&~ \En_{Y_0\sim \pbar} \indic\crl{Y_0\in B_j} \exp\prn*{-\frac{\nrm{\xbar-Y_0}^2+2\tri{V,\xbar-Y_0}}{2\tau}} \\
    \geq&~ \pbar(B_j)e^{-2r^2/\tau} u(V),
\end{align*}
and we also have
\begin{align*}
    \En_V[u(V)^{-1}]
    \leq&~ \En_V \En_{Y_0\sim \pbar} \brk*{ \exp\prn*{\frac{\tri{V,\xbar-Y_0}}{\tau}} \Big| Y_0\in B_j} \\
    =&~ \En_{Y_0\sim \pbar} \brk*{ \exp\prn*{\frac{\nrm{\xbar-Y_0}^2}{2\tau}} \Big| Y_0\in B_j}
    \leq \exp\prn*{\frac{2 r^2}{\tau}}.
\end{align*}

By definition and \eqref{pfeq:intrinsic-mgf-decomp}, for any $V$, we can decompose
\begin{align*}
    \MM_{16}(\xbar+V)\leq\sqrt{\MM_{8}(\xbar+V)}
    \leq \frac{e^{2r^2/\tau}}{\scedit{\pbar}(B_j)\, u(V)}\sum_{i\in\cI} g_i(V)\exp\prn*{ \frac{\tri{V,\xbar-z_i}}{\tau} }.
\end{align*}
Therefore, by union bound, we know that $\PP_V(\cV_1\cap \cV_2)\geq 1-\delta$, where
\begin{align*}
    \cV_1 &\deq \crl[\Big]{ u(V)^{-1}\leq e^{2 r^2/\tau}\cdot \frac{3N}{\delta} } \cap \scedit{\bigcap_{i\in [N]}} \crl[\Big]{ g_i(V)\leq \exp\prn[\Big]{ \frac{2 r^2-(\nrm{z_i-\xbar}-r)_+^2}{4\tau} }\cdot \frac{3N}{\delta} }\,, \\
    \cV_2 &\deq \scedit{\bigcap_{i\in [N]}}\crl[\Big]{ \tri{V, \xbar-z_i}\leq \nrm{\xbar-z_i}\sqrt{2\tau\log(3N/\delta)}}\,.
\end{align*}
Note that under $V\in\cV_1\cap \cV_2$, we can upper bound
\begin{align*}
    \log \MM_{16}(\xbar+V)-\log \frac{1}{\scedit{\pbar}(B_j)}
    \leq&~ 2\log(3N/\delta)+\frac{5 r^2}{\tau}+\max_i -\frac{(\nrm{z_i-\xbar}-r)_+^2}{4\tau}+\nrm{\xbar-z_i}\sqrt{2 \log(3N/\delta)/\tau} \\
    \leq&~ 2\log(3N/\delta)+\frac{5 r^2}{\tau}+r\sqrt{2 \log(3N/\delta)/\tau} +2 \log(3N/\delta) \\
    \leq&~ 5 \log(3N/\delta)+\frac{6 r^2}{\tau}.
\end{align*}
Therefore, we denote $w=\frac{1}{10}$, and we know that
\begin{align*}
    \PP_V\prn*{ \MM_{16}(\xbar+V)^{w}
    \geq \frac{1}{\scedit{\pbar}(B_j)^w}\sqrt{3N/\delta} e^{r^2/\tau} }\leq \delta, \qquad\forall \delta\in(0,1).
\end{align*}
Integration gives
\begin{align*}
    \En_V \MM_{16}(\xbar+V)^{w} \leq \frac{2}{\scedit{\pbar}(B_{j})^w} \sqrt{3N}e^{r^2/\tau}.
\end{align*}
Finally, we can take expectation over $\xbar\sim \pbar$, and using the fact that $\En[\scedit{\pbar}(B_{j(\xbar)})^w]\leq \sum_i \scedit{\pbar}(B_i)^{1-w}\leq \sqrt{N}$ gives
\begin{align}
    \En_{\xbar\sim \pbar, V\sim\normal{0,\tau\Id}} \MM_{16}(\xbar+V)^{w}\leq 4Ne^{r^2/\tau}.
\end{align}
This gives the desired upper bound by taking infimum over $r>0$ (and combining \cref{lem:MGF-trivial} when $\dim_\tau(\pbar)=d$).
\end{proof}

\begin{proof}[\pfref{cor:MGF-intrinsic}]
The first inequality follows from \cref{prop:logp-adapt}. In addition, with the first inequality, we know $m_s(Y_s)=\En[Y_0\mid Y_s]=\En[Y_0\mid Y_s,Y_t]$, and hence
\begin{align*}
    \En\exp\prn[\Big]{ \frac{1}{C\tau}\,\nrm{m_s(Y_s)-m_\tau(Y_\tau)}^2 }
    \leq \En\exp\prn[\Big]{ \frac{1}{C\tau}\,\En[\nrm{Y_\tau-Y_0}^2\mid Y_s,Y_\tau] }
    &\leq \En\exp\prn[\Big]{\frac{1}{C\tau} \,\nrm{Y_\tau-Y_0}^2}\\
    &\leq e^{\dim_\tau(\pbar)}\,.
\end{align*}
This gives the third inequality. In the following, we prove the second inequality. 
By our proof of \cref{prop:logp-adapt}, we can show the following fact: There is a constant $C_0>0$ such that 
\begin{align*}
    \En_{\xbar\sim \pbar, V\sim \normal{0,\tau\Id}} \En_{Y_0\sim Q(\cdot\mid \xbar+V)}\brk*{ \exp\prn*{ \frac{\abs{\tri{Y_0-\xbar, V}}}{C_0\tau} } }\leq e^{\dim_\tau(\pbar)}
\end{align*}
Note that for random variable $A$, it holds that $\sqrt{\En A^2}\leq 2+\log\En \exp(\abs{A})$, i.e., $\exp(\sqrt{\En A^2})\leq e^2\En \exp(\abs{A})$. Further, $\nrm{V}_{\cov_\tau(\xbar+V)}=\sqrt{\En_{Y_0\sim Q(\cdot\mid \xbar+V)}\abs{\tri{Y_0-\xbar, V}}^2}$. Therefore, it holds that
\begin{align*}
    \En_{\xbar\sim \pbar, V\sim \normal{0,\tau\Id}}\brk*{ \exp\prn*{ \frac{1}{C_0\tau}\nrm{V}_{\cov_\tau(\xbar+V)} } }\leq e^{\dim_\tau(\pbar)+2}.
\end{align*}
Consider $x=\xbar+V$. Under $\xbar\sim \pbar, V\sim \normal{0,\tau\Id}$, it holds that $\En[V\mid x]=x-\En[\xbar\mid x]=x-m_\tau(x)$. By convexity, we can then conclude that
\begin{align*}
    \En_{x\sim \pbar*\normal{0,\tau\Id}}\brk*{ \exp\prn*{ \frac{1}{C_0\tau}\nrm{x-m_\tau(x)}_{\cov(\xbar+V)} } }
    \leq\En_{\xbar\sim \pbar, V\sim \normal{0,\tau\Id}}\brk*{ \exp\prn*{ \frac{1}{C_0\tau}\nrm{V}_{\cov(\xbar+V)} } }\leq e^{\dim_\tau(\pbar)+2}.
\end{align*}
This is the desired upper bound.
\end{proof}

\begin{lemma}\label{lem:MGF-trivial}
The following holds for any $t>0$.

(1) It holds that
\begin{align*}
\En\exp\prn*{ \frac{1}{10t}\nrm{Y_0-m_t(Y_t)}^2 }
    \leq e^d.
\end{align*}

(2) It holds that
\begin{align*}
    \En\exp\prn*{ \frac{1}{3t}\nrm{Y_t-m_t(Y_t)}^2 }
    \leq e^d.
\end{align*}

(3) It holds that
\begin{align*}
    \En\exp\prn*{ \frac{1}{3t}\tr(\cov_t(Y_t)) }
    \leq e^d.
\end{align*}

(4) For any $0\leq s\leq t$, it holds that
\begin{align*}
    \En\exp\prn*{ \frac{1}{3t}\nrm{m_s(Y_s)-m_t(Y_t)}^2 }
    \leq e^d.
\end{align*}
\end{lemma}

\begin{proof}
By definition, we know $Y_t-m_t(Y_t)=\En[Y_t-Y_0\mid Y_t]$, and hence for any $\lambda<\frac{1}{2t}$,
\begin{align*}
    \En\exp\prn*{ \lambda\nrm{Y_t-m_t(Y_t)}^2 }
    \leq&~ \En\exp\prn*{ \lambda\En[\nrm{Y_t-Y_0}^2\mid Y_t] }
    \leq \En\exp\prn*{\lambda\nrm{Y_t-Y_0}^2}=(1-2\lambda t)^{-d/2},
\end{align*}
where we use $Y_t-Y_0\sim \normal{0,t\Id}$. Choosing $\lambda=\frac{1}{3t}$ completes the proof of (2) and (3), because $\tr(\cov_t(Y_t))=\En[\nrm{Y_0-m_t(Y_t)}^2\mid Y_t]\leq \En[\nrm{Y_t-Y_0}^2\mid Y_t]$. In addition, we know $m_s(Y_s)=\En[Y_0\mid Y_s]=\En[Y_0\mid Y_s,Y_t]$, and hence
\begin{align}\label{eq:MGF-finiteness}
    \En\exp\prn*{ \lambda\nrm{m_s(Y_s)-m_t(Y_t)}^2 }
    \leq&~ \En\exp\prn*{ \lambda\En[\nrm{Y_t-Y_0}^2\mid Y_s,Y_t] }
    \leq \En\exp\prn*{\lambda\nrm{Y_t-Y_0}^2}=(1-2\lambda t)^{-d/2}.
\end{align}
This gives (4).

Further, using $\nrm{Y_0-m_t(Y_t)}\leq \nrm{Y_0-Y_t}+\nrm{Y_t-m_t(Y_t)}$, we know that for any $\lambda<\frac{1}{2t}$,
\begin{align*}
    \En\exp\prn*{ \frac{\lambda}{4}\nrm{Y_0-m_t(Y_t)}^2 }
    \leq&~ \En\exp\prn*{ \frac{\lambda}{2}\nrm{Y_0-Y_t}^2+\frac{\lambda}{2}\nrm{Y_t-m_t(Y_t)}^2 } \\
    \leq&~ \sqrt{\En\exp\prn*{ \lambda\nrm{Y_t-m_t(Y_t)}^2 }\En\exp\prn*{ \lambda\nrm{Y_t-m_t(Y_t)}^2 }}
    \leq (1-2\lambda t)^{-d/2}.
\end{align*}
This gives (1) by choosing $\lambda=\frac{2}{5t}$. 
\end{proof}

\subsection{\pfref{prop:Lip-op-to-Frob}}

By \cref{cor:MGF-intrinsic}, it holds that $\PP_{Y_\tau\sim q_\tau}(\tr(\cov_t(Y_t))\geq C(\dstar+\log(1/\delta)))\leq \frac{\delta}{2\dstar^5}$ for any $\delta\in(0,1)$. Then, using $\nrmF{\cov_t(Y_t)}^2\leq \nrmop{\cov_t(Y_t)}\cdot \tr(\cov_t(Y_t))$, we can upper bound
\begin{align*}
    \PP_{Y_\tau\sim q_\tau}(t^{-2}\nrmF{\cov_t(Y_t)}^2\geq C(\dstar+\log(1/\delta))\Lipop[\delta/2])
    \leq&~ \PP_{Y_\tau\sim q_\tau}(t^{-1}\nrmop{\cov_t(Y_t)}\geq \Lipop[\delta/2]) \\
    &~ +\PP_{Y_\tau\sim q_\tau}(\tr(\cov_t(Y_t))\geq C(\dstar+\log(1/\delta)))\\
    \leq&~ \frac{\delta}{\dstar^5}.
\end{align*}
\qed

\subsection{\pfref{prop:DDPM-Lip}}

Fix any $k\in[K]$.
In \cref{cor:coverage-backward}, we choose $\eta=\eta_k$, $\tau=\sigma_{k+1}^2$, and hence $\tau-\eta=\sigma_k^2$, $q_{\tau-\eta}=p_k$ and $\wt{q}_{\tau,\eta}=\wt{p}_k$. Then, suppose that \cref{asmp:non-unif-Lip} holds with $\Lipop\leq \Lop\cdot \polylog(M/\delta)$. As long as 
\begin{align*}
    \frac{\sigma_k^2}{\eta_k}\gg \min\crl{ \Lop^{2/3}d^{1/3}, \Lop^{1/3}\dstar^{2/3} } \log^{1/3}(M\dstar/\delta)+\log^2(M\dstar/\delta),
\end{align*}
by \cref{cor:coverage-backward}, it holds that
\begin{align*}
    \Dcov[e]{\wt{p}_k}{p_k}\leq \frac{\delta}{100\dstar^5}.
\end{align*}
Then, \cref{asmp:Lip-Frob} implies that $\bbP_{X_k\sim p_k}\prn[\Big]{ \nrmF{\nabla m_{\sigma_k^2}(X_k)}> \LipF[\delta/3] }\leq \frac{\delta}{3\dstar^5}$, and hence
\begin{align}
    \bbP_{X_k'\sim \wt{p}_k}\prn[\Big]{ \nrmF{\nabla m_{\sigma_k^2}(X_k')}> \LipF[\delta/3] }\leq&~
    e \bbP_{X_k\sim p_k}\prn[\Big]{ \nrmF{\nabla m_{\sigma_k^2}(X_k)}> \LipF[\delta/3] } + \Dcov[e]{\wt{p}_k}{p_k}\leq 
     \frac{\delta}{\dstar^5}.
\end{align}
\qed

\subsection{Why Lipschitz condition under the Frobenius norm?}\label{ssec:why-Frob}

\sccomment{I suggest moving this subsection to the appendix. It adds a lot of new notation to parse and I don't think it adds enough to the main story}

\newcommand{\step}{\eta'}
\newcommand{\Comp}{\mathsf{Lip}}

We now argue that the Frobenius-norm Lipschitz condition (\cref{asmp:Lip-Frob}) is not merely an artifact of our analysis, but rather an instance-specific complexity measure for any sampling scheme based on Gaussian approximation of \eqref{eq:backward_kernel}. The argument proceeds through an exact characterization of the KL error of one-step Gaussian approximation.

To make this precise, consider the one-step KL divergence from the true backward transition $\rho_k(\cdot\mid X_{k+1})$ to any Gaussian approximation:
\begin{align*}
    U_k(\eta, v)\deq \En_{X_{k+1}\sim p_{k+1}}\Dkl{\rho_k(\cdot\mid X_{k+1})}{\normal{v(X_{k+1}),\eta\Id}}\,,
\end{align*}
where $\eta>0$ is the step size and $v:\RR^d\to\RR^d$ is an arbitrary mean function. Intuitively, $U_k(\eta, v)$ measures the best possible performance of any one-step scheme of the form $X_k\sim \normal{v(X_{k+1}),\eta\Id}$. More specifically, for rejection sampling with proposal $\normal{v(X_{k+1}),\eta\Id}$ (e.g., via FORS) to succeed, it is at least necessary that $U_k(\eta, v)=O(1)$.

The following theorem provides an \emph{exact} characterization of $\min_v U_k(\eta, v)$, revealing that the minimum one-step KL decomposes into a score estimation error and an irreducible discretization term governed by $\nabla m_\tau(Y_\tau)$. To state the result, we define
\begin{align*}
    \Comp_k(\lambda)\deq \int_{\sigma_k^2}^{\sigma_{k+1}^2} \frac{(\lambda+\tau)(\tau-\sigma_k^2)}{\tau^2(\lambda+\sigma_k^2)} \En\nrmF{\nabla m_\tau(Y_\tau)-\lambda/(\lambda+\tau)\Id}^2\, d\tau\,, \qquad \forall \lambda\geq 0\,,
\end{align*}
and $\Comp_k(\infty)\deq \lim_{\lambda\to\infty} \Comp_k(\lambda)$.

\begin{theorem}\label{thm:DDPM-KL-exact}
For $k\in[K]$, it holds that
\begin{align}\label{eq:KL-decomp}
    U_k(\eta, v)=\frac{1}{2\eta}\En\nrm{v(X_{k+1})-X_{k+1}-\eta\scfs_{k+1}(X_{k+1})}^2+\wb{U}_k(\eta),
\end{align}
where $\wb{U}_k(\eta)\geq 0$ satisfies:

(1) When $\eta<\etabar_k$, it holds that $\wb{U}_k>\Comp_k(0)$. When $\eta>\eta_k$, it holds that $\wb{U}_k>\Comp_k(\infty)$. 

(2) When $\eta\in[\etabar_k,\eta_k]$, then there is $\lambda\in[0,\infty]$ such that $\frac{1}{\eta}=\frac{1}{\eta_k}+\frac{1}{\lambda+\sigma_k^2}$ and $\wb{U}_k(\eta)=\Comp_k(\lambda)$.
\end{theorem}

We interpret this result and its implications below.

\paragraph{Optimal mean function}
The first term in the decomposition~\eqref{eq:KL-decomp}, $\frac{1}{2\eta}\En\nrm{v(X_{k+1})-X_{k+1}-\eta\scfs_{k+1}(X_{k+1})}^2$, vanishes iff $v(X)=X+\eta \scfs_{k+1}(X)$. Given an estimated score function $\scf_{k+1}\approx \scfs_{k+1}$, the optimal choice of mean function is therefore $v(X)=X+\eta \scf_{k+1}(X)$, which matches the DDPM proposal we consider.

\paragraph{Discretization error}
The second term $\wb{U}_k(\eta)\geq 0$ cannot be eliminated by any choice of mean function $v$; it provides a \emph{lower bound} on the one-step KL divergence that is intrinsic to the data distribution.
By \cref{thm:DDPM-KL-exact} (2), it holds that
\[
\min_{\eta,v} U_k(\eta, v) = \min_{\lambda\geq 0} \Comp_k(\lambda)\,,
\]
where the minimum over $\eta$ selects the optimal step size for a given $\lambda$. This identity shows that $\Comp_k(\lambda)$ \emph{exactly} characterizes the best possible one-step performance of any Gaussian transition scheme. Further, to sample from the backward kernel~\eqref{eq:backward_kernel}, we may expect that rejection sampling with proposal $\normal{v(X_{k+1}),\eta\Id}$ succeeds only when $ U_k(\eta, v)=O(1)$, i.e., $\Comp_k(\lambda)=O(1)$ for the corresponding $\lambda$. 

In principle, we should choose $\lambda_k^\star$ that minimizes $\Comp_k(\cdot)$ and select $\eta_k^\star$ accordingly (by $\frac{1}{\eta_k^\star}=\frac{1}{\eta_k}+\frac{1}{\lambda^\star_k+\sigma_k^2}$). However, this requires prior knowledge of the data distribution, and in the absence of such knowledge, the natural choice is $\lambda=0$, for which $\frac{1}{\eta}=\frac{1}{\eta_k}+\frac{1}{\sigma_k^2}$, i.e., $\eta=\etabar_k$. This choice is justified by \cref{cor:MGF-intrinsic}, which implies $\Comp_k(0)\leq O(\dstar\eta_k/\sigma_k^2)$. Further, we can relate $\Comp_k(0)$ to the non-uniform Lipschitz condition (\cref{asmp:Lip-Frob}) directly.

\subsection{\pfref{thm:DDPM-KL-exact}}
We work under the additional notation introduced in \cref{appdx:structure}. We write $\tau=\sigma_{k+1}^2$, $\beta=\sigma_k^2$, so that $\tau-\beta=\eta_k$. We also denote $\wt{\eta}=\eta_k$.

Recall that
\begin{align*}
    q_{\tau,\tau-\beta}(y\mid Y_\tau)=\frac{q_{\beta}(y)}{q_\tau(Y_\tau)}\cdot \frac{1}{(2\pi\wt{\eta})^{d/2}}\exp\prn*{ -\frac{1}{2\wt{\eta}}\nrm{y-Y_\tau}^2 }
\end{align*}
is the conditional distribution of $Y_{\beta}\mid Y_\tau$. 
Then, we can express
\begin{align*}
    U(v,\eta)=&~ \Dkl{q_{\tau,\tau-\beta}(y\mid Y_\tau)}{\normal{v(Y_\tau),\eta\Id}} \\
    =&~\En_{Y\sim q_{\tau,\tau-\beta}(\cdot\mid Y_\tau)}\brk*{ \log q_{\beta}(Y)-\log q_\tau(Y_\tau)-\frac{1}{2\wt{\eta}}\nrm{Y-Y_\tau}^2+\frac{1}{2\eta}\nrm{Y-v(Y_\tau)}^2 }+\frac{d}{2}\log(\eta/\wt{\eta}).
\end{align*}
Taking expectation over $Y_\tau\sim q_\tau$, we know
\begin{align*}
    &~\En_{Y_\tau\sim q_\tau}\Dkl{q_{\tau,\tau-\beta}(y\mid Y_\tau)}{\normal{v(Y_\tau),\eta\Id}} \\
    =&~\En_{(Y_{\beta},Y_\tau)}\brk*{ \log q_{\beta}(Y_{\beta})-\log q_\tau(Y_\tau)-\frac{1}{2\wt{\eta}}\nrm{Y_{\beta}-Y_\tau}^2+\frac{1}{2\eta}\nrm{Y_{\beta}-v(Y_\tau)}^2 }+\frac{d}{2}\log(\eta/\wt{\eta}) \\
    =&~ H(q_\tau)-H(q_{\beta})-\frac{d}{2}+\frac{1}{2\eta}\En\brk*{ \nrm{Y_{\beta}-\En[Y_{\beta}\mid Y_\tau]}^2+\nrm{\En[Y_{\beta}\mid Y_\tau]-v(Y_\tau)}^2 }+\frac{d}{2}\log(\eta/\wt{\eta}),
\end{align*}
where $H(q)=-\En_{Y\sim q}[\log q(Y)]$ is the differential entropy of $q$. 

Note that $\En[Y_{\beta}\mid Y_\tau]=Y_\tau+\frac{\wt{\eta}}{\tau}(\En[Y_0\mid Y_\tau]-Y_\tau)=Y_\tau+\wt{\eta}\nabla \log q_\tau(Y_\tau)$. Hence,
\begin{align*}
    \En\brk*{ \nrm{Y_{\beta}-\En[Y_{\beta}\mid Y_\tau]}^2 }
    =&~\En\brk*{ \nrm{Y_{\beta}-Y_\tau}^2 }-\En\brk*{ \nrm{Y_{\tau}-\En[Y_{\beta}\mid Y_\tau]}^2 } \\
    =&~ d\wt{\eta} - \frac{\wt{\eta}^2}{\tau^2}\En\nrm{Y_\tau-m_\tau(Y_\tau)}^2 \\
    =&~ d\prn*{\wt{\eta}-\frac{\wt{\eta}^2}{\tau}}+\frac{\wt{\eta}^2}{\tau^2}\En\nrm{Y_0-m_\tau(Y_\tau)}^2.
\end{align*}
It is also straightforward to verify that (see, e.g., \citep[Theorem 3.14, I-MMSE]{polyanskiy2025information})
\begin{align*}
    \partial_\tau H(q_\tau)=\frac{d}{2\tau}-\frac{1}{2\tau^2}\En\nrm{Y_0-m_\tau(Y_\tau)}^2.
\end{align*}
Therefore, in the following, we denote $M_t\deq \En\nrm{Y_0-m_t(Y_t)}^2$, and then
\begin{align*}
    2\wb{U}(\eta)\deq &~2\En_{Y_\tau\sim q_\tau}\Dkl{q_{\tau,\tau-\beta}(y\mid Y_\tau)}{\normal{v(Y_\tau),\eta\Id}} -\frac{1}{\eta}\En\nrm{v(Y_\tau)-Y_\tau-\wt{\eta}\nabla \log q_\tau(Y_\tau)}^2 \\
    =&~ \frac{\wt{\eta}^2}{\eta\tau^2}M_\tau-\int_{\beta}^\tau \frac{M_t}{t^2}dt+d\brk*{ \frac{\wt{\eta}\beta}{\eta\tau}-1 -\log\frac{\wt{\eta}\beta}{\eta\tau} }.
\end{align*}
Note that from our proof of \cref{prop:DDPM-coverage}, it holds that
\begin{align*}
    M_\tau-M_t=\En\nrm{m_t(Y_t)-m_\tau(Y_\tau)}^2=\int_t^\tau \En\nrmF{\nabla m_s(Y_s)}^2 ds.
\end{align*}
Therefore, $\partial_t M_t=\frac{1}{t^2}\En\nrmF{\cov_t(Y_t)}^2$. We denote $c_t=\frac{(\lambda+t)^2}{t^2}$ and $b_t=\frac{\lambda t}{\lambda+t}$, and then
\begin{align*}
    \partial_t (c_tM_t)=\frac{c_t}{t^2}\En\nrmF{\cov_t(Y_t)}^2-\frac{2\lambda c_t}{t(\lambda+t)}M_t
    =\frac{c_t}{t^2} \En\nrmF{\cov_t(Y_t)-b_t\Id}^2-\frac{\lambda^2d}{t^2}.
\end{align*}
Integrating from $t$ to $\tau$ gives
\begin{align*}
    c_\tau M_\tau-c_t M_t=\int_t^\tau \frac{c_s}{s^2} \En\nrmF{\cov_s(Y_s)-b_s\Id}^2 ds - \frac{\lambda^2d}{t}+\frac{\lambda^2d}{\tau}.
\end{align*}
Integrating $\frac{c_\tau}{(\lambda+t)^2}M_\tau-\frac{1}{t^2}M_t$ from $\beta$ to $\tau$ gives
\begin{align*}
    \frac{\wt{\eta}^2 M_\tau}{\tau^2}\prn*{\frac{1}{\wt{\eta}}+\frac{1}{\lambda+\beta}}-\int_{\beta}^\tau \frac{M_t}{t^2}dt
    =&~ I_\lambda +d\prn*{ \frac{\wt{\eta} \lambda}{\tau(\lambda+\beta)}-\log\frac{\tau(\lambda+\beta)}{(\lambda+\tau)\beta} },
\end{align*}
where
\begin{align*}
    I_\lambda \deq \int_{\beta}^\tau \int_t^\tau \frac{c_s}{s^2(\lambda+t)^2} \En\nrmF{\cov_s(Y_s)-b_s\Id}^2 dsdt
    =\int_{\beta}^\tau \frac{(\lambda+s)(s-\tau+\wt{\eta})}{s^2(\lambda+\beta)} \En\nrmF{s^{-1}(\cov_s(Y_s)-b_s\Id)}^2 ds.
\end{align*}

In the following, we consider three cases.

(1) Suppose that $\frac{1}{\eta}>\frac{1}{\wt{\eta}}+\frac{1}{\beta}$. Then we can set $\lambda=0$ to see $2\wb{U}(\eta)>I_0$. 

(2) Suppose that $\frac{1}{\wt{\eta}}<\frac{1}{\eta}\leq \frac{1}{\wt{\eta}}+\frac{1}{\beta}$. Then there is $\lambda\geq 0$ such that $\frac{1}{\eta}=\frac{1}{\wt{\eta}}+\frac{1}{\lambda+\beta}$, and our calculation above shows $2\wb{U}(\eta)=I_\lambda$.

(3) Suppose that $\frac{1}{\eta}\leq \frac{1}{\wt{\eta}}$. In this case, we can let $\lambda\to \infty$, and it is clear that
\begin{align*}
    \lim_{\lambda\to\infty} I_\lambda=I_\infty\deq \int_{\beta}^\tau \frac{s-\tau+\wt{\eta}}{s^2} \En\nrmF{s^{-1}\cov_s(Y_s)-\Id}^2 ds.
\end{align*}
In this case, we have
\begin{align*}
    2\wb{U}(\eta)=I_\infty+\frac{\wt{\eta}^2M_\tau}{\tau^2}\prn*{\frac{1}{\eta}-\frac{1}{\wt{\eta}}}+d\brk*{ \frac{\wt{\eta}\beta}{\eta\tau}+\frac{\wt{\eta}}{\tau}-1-\log \frac{\wt{\eta}}{\eta} }.
\end{align*}
Note that $M_\tau=d\tau-\En\nrm{m_\tau(Y_\tau)-Y_\tau}^2\leq d\tau$, and hence $2\wb{U}(\eta)\geq I_\infty+d\prn*{\frac{\wt{\eta}}{\eta}-1-\log \frac{\wt{\eta}}{\eta}}\geq I_\infty$, with equality iff $\eta=\wt{\eta}$. 

Combining all the cases completes the proof.
\qed

\begin{corollary}\label{cor:DDPM-p0}
Suppose that $\nabla \log \pdata$ is $L$-Lipschitz and $\sigma_0^2\leq \frac{1}{2L}$.
Then it holds that 
\begin{align*}
    \En_{X_{1}\sim p_{1}}\Dkl{\rho_1(\cdot\mid X_{1})}{\normal{ X_1 + \eta_0 \scf_{1}(X_1), \eta_0 \Id }}
    \leq \frac{\eta_0}{2}\En\nrm{\scf_1(X_1)-\scfs_1(X_1)}^2
    +2dL^2\sigma_0^4.
\end{align*}
\end{corollary}

\begin{proof}
Using the proof above, we know for $\eta<\tau$, score function $s$, 
\begin{align*}
    &~ \Dkl{q_{\tau,\eta}(y\mid Y_\tau)}{\normal{Y_\tau+\eta s(Y_\tau),\eta\Id}} \\
    =&~ \frac{\eta}{2}\En\nrm{ s(Y_\tau)-\nabla \log q_\tau(Y_\tau) }^2+\frac12\int_{\tau-\eta}^\tau \frac{t-\tau+\eta}{t^2} \En\nrmF{t^{-1}\cov_t(Y_t)-\Id}^2 dt.
\end{align*}
Note that under our assumption, the conditional distribution $Y_0\mid Y_t$ is $(t^{-1}-L)$-strongly log-concave and $(t^{-1}+L)$-log-smooth, and hence
\begin{align*}
    \frac{1}{t^{-1}+L}\Id\preceq \cov_t(Y_t)\preceq \frac{1}{t^{-1}-L}\Id,
\end{align*}
and hence $\frac{Lt}{1+Lt}\Id\preceq t^{-1}\cov_t(Y_t)-\Id\preceq \frac{Lt}{1-Lt}\Id$. This immediately implies
\begin{align*}
    \int_{\tau-\eta}^\tau \frac{t-\tau+\eta}{t^2} \En\nrmF{t^{-1}\cov_t(Y_t)-\Id}^2 dt
    \leq \int_{\tau-\eta}^\tau \frac{t-\tau+\eta}{t^2} \cdot \frac{(Lt)^2d}{(1-Lt)^2} dt \leq \frac{(L\eta)^2d}{(1-L\tau)^2}.
\end{align*}
Sending $\eta\to \tau$ completes the proof.
\end{proof}

\section{Proofs for Section~\ref{sec:diffusion}}

\subsection{Single-step analysis}

In the following, we analyze the $k$-th step of \cref{alg:DM} for a fixed $k\in[K]$. Without loss of generality, we assume $\alpha_k=1$.
Recall that in the $k$-th step, our goal is to sample from the tilt measure
\begin{align}\label{eq:def-nu-xz}
    \rho_k(x\mid x_+)\propto p_k(x) \exp\prn[\Big]{-\frac{\nrm{x-x_+}^2}{2\eta}}\,,
\end{align}
given estimated score function $\scf_k\approx \nabla \log p_k$ and $\scf_{k+1}\approx \nabla \log p_{k+1}$.

\paragraph{Notation}
Recall that we define %
$\frac{1}{\etabar_k} = \frac{1}{\eta_k} +\frac{1}{\sigma_k^2}$, 
\begin{align*}
    \Dn_k(x)=\sigma_k^2 \scf_k(x)+x\,, \qquad
    \Ds_k(x)=\sigma^2 \scfs_k(x)+x=\En[\Xbar\mid X_k=x]\,,
\end{align*}
where the expectation is taken over $\Xbar=\alpbar_k X_0, X_0\sim \pdata, X\sim \normal{\Xbar,\sigma^2\Id}$.
Similarly, we define
\begin{align*}
    \Dn_{k+1}(x)=\sigma_{k+1}^2 \scf_{k+1}(x)+x\,, \qquad
    \Ds_{k+1}(x)=\sigma_{k+1}^2 \scfs_{k+1}(x)+x=\En[\Xbar\mid X_{k+1}=x]\,,
\end{align*}
where the expectation is taken over $\Xbar=\alpbar_k X_0, X_0\sim \pdata, X\sim \normal{\Xbar,\sigma^2\Id}, X_{k+1}\sim \normal{X_k,\eta\Id}$.

Denote $\gd(x_+)\deq x_++\eta\scf_{k+1}(x_+)$.
In the following, we abbreviate
\begin{align*}
    \Enxp[\cdot]\ldef \En_{\lr\sim \unif([0,1]),\,z\sim\Netabar,\, \xhat\sim \normal{\gd(x_+),\frac12\etabar\Id}}[\cdot]\,.
\end{align*}
We will also frequently omit the subscript $k$ and denote $\eta=\eta_k$, $\sigma=\sigma_k$, and $\etabar=\etabar_k$.

\paragraph{Distributions}
By~\cref{thm:fors}, \cref{alg:fors} instantiated with the proposal distribution $\normal{\gd(x_+),\etabar I}$ and estimator $\widehat W_{z,r,\xhat,x} = \pclip{ W_{z,r,\xhat,x}}$,
\begin{align}\label{eq:def-DM-path-What}
    W_{z,r,\xhat,x}\deq \sigma^{-2}\langle \dot\gamma_{z,r,\xhat}(x),\, \Dn_k(\gamma_{z,r,\xhat}(x)) - \Dn_{k+1}(x_+)\rangle
\end{align}
samples from the distribution $\rhohat_k(\cdot\mid x_+)$, where
\begin{align}\label{eq:rhohz}
    \log \rhohat_k(x\mid x_+)=&~\const_{x_+}+\Enxp\pclip{W_{z,r,\xhat,x}}-\frac{\nrm{x-\gd(x_+)}^2}{2\etabar}.
\end{align}
For analysis, we introduce $\gds(x_+)\deq x_++\eta\scfs_{k+1}(x_+)$ and
\begin{align}
    \Wstar_{z,r,\xhat,x} = \sigma^{-2}\langle \dot\gamma_{z,r,\xhat}(x),\, \Ds_k(\gamma_{z,r,\xhat}(x)) - \Ds_{k+1}(x_+)\rangle,
\end{align}
and write $\rhostar_k(\cdot\mid x_+)$ for the density corresponding to 
\begin{align}\label{eq:rhosz}
    \log \rhostar_k(x\mid x_+)=\const_{x_+}+\Enxps\pclip{\Wstar_{z,r,\xhat,x}}-\frac{\nrm{x-\gds(x_+)}^2}{2\etabar}\,.
\end{align}
Using \cref{eq:path-integral}, we can also write
\begin{align*}
    \log \rho_k(x\mid x_+)+\frac{\nrm{x-\gds(x_+)}^2}{2\etabar}
    =&~\const_{x_+}+ \log p_k(x)+\frac{\nrm{x}^2}{2\sigma^2}-\tri{\Ds_{k+1}(x_+),x-x_+}\\
    =&~ \const_{x_+}'+\Enxps \sigma^{-2}\tri{\dot \gamma_{z,r,\xhat}(x),\Ds_k(\gamma_{z,r,\xhat}(x))-\Ds_{k+1}(x_+)}\,,
\end{align*}
where we recall $\nabla \log p_k=\sigma^{-2}(\Ds_k(x)-x)$, and hence
\begin{align}
    \log \rho_k(x\mid x_+)=&~\const_{x_+}+\Enxps \Wstar_{z,r,\xhat,x}-\frac{\nrm{x-\gds(x_+)}^2}{2\etabar}\,.
\end{align}

The following property is crucial for our analysis.
\begin{lemma}\label{lem:ind-Gaussian}
For any given vector $g\in \RR^d$ and $r\in[0,1]$, for independent random vectors $x\sim \normal{g, \etabar\Id}$, $\xhat\sim \normal{g,\frac12\etabar\Id}$, $z\sim \normal{0,\frac12\etabar\Id}$, it holds that $(\dot\gamma_{z,r,\xhat}(x),\gamma_{z,r,\xhat}(x))$ are independent Gaussian vectors distributed as
\begin{align*}
    \gamma_{z,r,\xhat}(x)\sim \normal{g,\etabar\Id}, \qquad \dot\gamma_{z,r,\xhat}(x)\sim \normal{0,c\etabar\Id},
\end{align*}
where $3(\alrp)^2/2+(\blrp)^2/2\equiv c=\frac{8}{27}\pi^2\leq 3$.
\end{lemma}

We analyze the relationship between distributions $\rho_k$, $\rhohat_k$ and $\rhostar_k$ in the following propositions. 

\begin{proposition}\label{prop:sc-err-new}
Suppose that $B\leq O(1)$. It holds that
\begin{align}\label{pfeq:sc-err-KL-decomp}
\begin{aligned}
    \MoveEqLeft \En_{x_+\sim p_{k+1}}\Dkl{\rho_k(\cdot\mid x_+)}{\rhohat_k(\cdot\mid x_+)}\\
    \leqsim &~ \En_{x_+\sim p_{k+1}}\Dkl{\rho_k(\cdot\mid x_+)}{\rhostar_k(\cdot\mid x_+)} \\
    &~+\eta_k \sigma_{k+1}^2/\sigma_k^2 \cdot \En_{x_+\sim p_{k+1}}\nrm{\scfs_{k+1}(x_+)-\scf_{\scedit{k+1}}(x_+)}^2
    +\eta_k \En_{x\sim p}\nrm{\scfs_k(x)-\scf_k(x)}^2\,.
\end{aligned}
\end{align}
\end{proposition}

\newcommand{\dbar}{\dstar}

\begin{proposition}\label{prop:intrinsic-chi-sq}
Suppose that $B=\Theta(1)$. For any $\delta\in(0,1)$, as long as
\begin{align}
    \frac{\sigma_k^2}{\eta_k}\gg \dbar\log(1/\delta)+\log^2(1/\delta),
\end{align}
it holds that $\En_{x_+\sim p_{k+1}}\Dchis{\rho_k(\cdot\mid x_+)}{\rhostar_k(\cdot\mid x_+)}\leq \delta$. 
\end{proposition}

\begin{proposition}\label{prop:sqrt-d-simple}
Suppose that $B=\Theta(1)$. For any $\delta\in(0,1)$, $\LF\geq 1$, as long as
\begin{align}\label{eq:DM-Lip-Frob-eta}
        \frac{\sigma_k^2}{\eta_k}\gg \LF \log(1/\delta)+\log^2(1/\delta),
    \end{align}
it holds that
\begin{align*}
    \MoveEqLeft \En_{x_+\sim p_{k+1}}\Dkl{\rho_k(\cdot\mid x_+)}{\rhostar_k(\cdot\mid x_+)} \\
    \leqsim&~ \dstar^5\brk*{ \delta+\max_{\tau\in[\sigma_k^2,\sigma_{k+1}^2]}\PP_{Y_\tau\sim q_\tau}\prn*{ \nrmF{\nabla m_\tau(Y_\tau)}\geq \LF }+ \PP_{X_k'\sim \wt{p}_k}\prn*{ \nrmF{\nabla m_{\sigma_k^2}(X_k')}\geq \LF }}.
\end{align*}
\end{proposition}

\subsection{Proof of \cref{thm:DM-intrinsic} and \cref{thm:DM-Lip}}

The process $\xr{K}\to\cdots\to\xr{1}$ generated by \cref{alg:DM} is a Markov chain such that $\xr{K}\sim \phat_K$, and $\xr{k}\mid \xr{k+1}\sim \wh\rho_t(\cdot\mid \xr{k+1})$.
The data-processing inequality and chain rule for the KL divergence yield
\begin{align*}
    \Dkl{p_1}{\wh p_1}
    &\leq \Dkl{p_K}{\wh p_K} + \sum_{k=1}^{K-1} \E_{\xr{k+1} \sim p_{k+1}}\Dkl{\rho_k(\cdot \mid \xr{k+1})}{\wh\rho_k(\cdot \mid \xr{k+1})}\,.
\end{align*}
Then, \cref{thm:DM-intrinsic} follows from \cref{prop:sc-err-new} and \cref{prop:intrinsic-chi-sq}; \cref{thm:DM-Lip} follows from \cref{prop:sc-err-new} and \cref{prop:sqrt-d-simple}.
\qed

\subsection{\pfref{prop:sc-err-new}}

\scedit{
\textbf{Claim.} Let $\nu_i(x) \propto \exp(h_i(x)-\frac{\|x-m_i\|^2}{2\etabar})$ for $i=1,2$, where $|h_i| \le B$.
Then, for any $\mu$,
\begin{align*}
    \Dkl{\mu}{\nu_1} \le 8e^{4B}\,\brk[\Big]{\Dkl{\mu}{\nu_0} + \En_{x\sim\nu_0}[(h_1(x) - h_0(x))^2] + \frac{\|m_1-m_0\|^2}{\etabar}}\,.
\end{align*}
\textbf{Proof of Claim.}
Introduce $\wt\nu(x) \propto \exp(h_1(x)-\frac{\|x-m_0\|^2}{2\etabar})$.
Note that
\begin{align*}
    \log(1+\Dchis{\normal{m_0,\etabar\Id}}{\normal{m_1,\etabar\Id}})=\frac{\nrm{m_1-m_0}^2}{\etabar}\,.
\end{align*}
Therefore, using \cref{lem:KL-triangle}, \cref{lem:Dchis-perturb}, and the fact that $\log(1+Ct)\leq C\log(1+t)$ for $t\geq 0$ and $C\geq 1$, we know that for any $\xz$, 
\begin{align*}
    \Dkl{\mu}{\nu_1}
    &\le 2\Dkl{\mu}{\wt{\nu}} + \log(1+\Dchis{\wt\nu}{\nu_1})
    \le 2\Dkl{\mu}{\wt{\nu}} + e^{4B}\,\frac{\|m_1-m_0\|^2}{\etabar}\,.
\end{align*}
Next, we note that we can write $\wt\nu \propto \nu_0 e^{h_1-h_0}$.
Then, using \cref{lem:KL-perturb}, we can bound
\begin{align*}
    \Dkl{\mu}{\wt \nu}
    &\le 4e^{2B}\,\prn[\big]{\Dkl{\mu}{\nu_0} + \En_{\nu_0}[(h_1-h_0)^2]}\,,
\end{align*}
which proves the claim. \qed

We now apply the claim with $\mu = \rho_k(\cdot \mid x_+)$, $\nu_0 = \rhostar_k(\cdot\mid x_+)$, and $\nu_1 = \rhohat_k(\cdot\mid x_+)$, which yields
\begin{align*}
    \Dkl{\rho_k(\cdot\mid x_+)}{\rhohat_k(\cdot\mid x_+)}
    &\lesssim_B\Dkl{\rho_k(\cdot\mid x_+)}{\rhostar_k(\cdot\mid x_+)} + \etabar\,\nrm{\scf_{k+1}(x_+)-\scfs_{k+1}(x_+)}^2 \\
    &\qquad{} + \En_{x\sim  \rhostar_k(\cdot\mid x_+)}[h(x\mid x_+)^2]\,,
\end{align*}
}
where
\begin{align*}
    h(x\mid x_+)=\Enxp\pclip{W_{z,r,\xhat,x}}-\Enxps\pclip{\Wstar_{z,r,\xhat,x}} \in [-2B,2B]\,.
\end{align*}

Next, we upper bound $\abs{h(x\mid x_+)}$. By triangle inequality,
\begin{align*}
    \abs{h(x\mid x_+)}
    \leq&~ \abs{\Enxp\pclip{W_{z,r,\xhat,x}}-\Enxps\pclip{W_{z,r,\xhat,x}}} \\
    &~+\abs{\Enxps\pclip{W_{z,r,\xhat,x}}-\Enxps\pclip{\Wstar_{z,r,\xhat,x}}} \\
    \leq&~ 2B\Dtv{\nor\prn[\Big]{\gds(x_+),\frac12\etabar\Id}}{\nor\prn[\Big]{\gd(x_+),\frac12\etabar\Id}}
    + \Enxps\pclip[2B]{\abs{W_{z,r,\xhat,x}-\Wstar_{z,r,\xhat,x}}}\,.
\end{align*}
We note that $\Dtv{\normal{\gds(x_+),\frac12\etabar\Id}}{\normal{\gd(x_+),\frac12\etabar\Id}}^2\leq 2\etabar\, \nrm{\scfs_{k+1}(x_+)-\scf_{k+1}(x_+)}^2$, and hence
\begin{align*}
    \En_{x\sim \rhostar_k(\cdot\mid x_+)}[h(x\mid x_+)^2]
    \leq 4\etabar\, \nrm{\scfs_{k+1}(x_+)-\scf_{k+1}(x_+)}^2
    +2\En_{x\sim \rhostar_k(\cdot\mid x_+)}\Enxps\pclip[2B]{W_{z,r,\xhat,x}-\Wstar_{z,r,\xhat,x}}^2\,.
\end{align*}
Next, we note that
\begin{align*}
    &~ \En_{x\sim \rhostar_k(\cdot\mid x_+)}\Enxps\pclip[2B]{W_{z,r,\xhat,x}-\Wstar_{z,r,\xhat,x}}^2 \\
    \leq &~ e^{4B}\En_{r\sim \unif([0,1])}\En_{x\sim \normal{\gds(x_+),\etabar\Id},\, \xhat\sim \normal{\gds(x_+),\frac12\etabar\Id},\,z\sim \Netabar} \\
    &\qquad\qquad\qquad\qquad\qquad{} \pclip[2B]{\sigma^{-2}\,\langle \dot\gamma_{z,r,\xhat}(x),\, [\Dn_k-\Ds_k](\gamma_{z,r,\xhat}(x)) - [\Dn_{k+1}-\Ds_{k+1}](x_+)\rangle}^2 \\
    \leq&~ 20e^{4B} \En_{x\sim \normal{\gds(x_+),\etabar\Id}} \min\crl*{\etabar\sigma^{-4}\,\nrm{ [\Dn_k-\Ds_k](x) - [\Dn_{k+1}-\Ds_{k+1}](x_+)}^2,4B^2} \\
    \leq&~ 20e^{8B} \En_{x\sim \rhostar_k(\cdot\mid x_+)} \min\crl*{\etabar\sigma^{-4}\,\nrm{ [\Dn_k-\Ds_k](x) - [\Dn_{k+1}-\Ds_{k+1}](x_+)}^2,4B^2}\,,
\end{align*}
where the second line uses the \cref{lem:ind-Gaussian} that for any fixed $r\in[0,1]$, under the distribution of consideration, $(\dot\gamma_{z,r,\xhat}(x),\gamma_{z,r,\xhat}(x))$ are independent Gaussian such that $\gamma_{z,r,\xhat}(x)\sim \normal{\gds(x_+),\etabar\Id}$ and
\begin{align*}
    \dot\gamma_{z,r,\xhat}(x)\sim \normal{0,c_r\etabar\Id}\,.
\end{align*}
Further, we can apply \cref{lem:Hels-var} to get
\begin{align*}
    &~\En_{x\sim \rhostar_k(\cdot\mid x_+)} \min\crl*{\scedit{\etabar\sigma^{-4}}\, \nrm{ [\Dn_k-\Ds_k](x) - [\Dn_{k+1}-\Ds_{k+1}](x_+)}^2,4B^2} \\
    \leq&~ 3\En_{x\sim \rho_k(\cdot\mid x_+)} \min\crl*{\etabar\sigma^{-4}\,\nrm{ [\Dn_k-\Ds_k](x) - [\Dn_{k+1}-\Ds_{k+1}](x_+)}^2,4B^2}+8B^2\Dkl{\rho_k(\cdot\mid x_+)}{\rhostar_k(\cdot\mid x_+)}\,.
\end{align*}
Combining the inequalities above immediately implies the desired upper bound.
\qed

\subsection{\pfref{prop:intrinsic-chi-sq}}

Fix any $\ell\geq 2$, $x_+\in\RR^d$.
\scedit{By~\cref{lem:fors_clip},}
\begin{align*}
    &\Dsys[\ell]{\rho_k(\cdot\mid x_+)}{\rhostar_k(\cdot\mid x_+)}
    \leq E_{\ell,x_+} \\
    &\qquad \deq e^{4B}\En_{r\sim \unif([0,1])} \En_{x\sim \normal{\gds(x_+),\etabar\Id},\, \xhat\sim \normal{\gds(x_+),\frac12\etabar\Id},\,z\sim \Netabar}[e^{2\ell(\abs{\Wstar_{z,r,\xhat,x}}-B)_+}-1] \\
    &\qquad \leq e^{4B-2\ell B} \En_{x'\sim \normal{\gds(x_+),\etabar\Id},\, w\sim \normal{0, c_r\etabar\Id}} e^{2\ell \abs{\tri{w,\Ds_k(x')-\Ds_{k+1}(x_+)}}} \\
    &\qquad \leq 2e^{4B-2\ell B} \En_{x\sim \normal{\gds(x_+),\etabar\Id}}\exp\prn*{ 6\etabar\ell^2 \nrm{\Ds_k(x)-\Ds_{k+1}(x_+)}^2}\,.
\end{align*}
Now, we can proceed
\begin{align*}
    E_{\ell,x_+}
    \leq&~  2e^{8B-2\ell B} \En_{x\sim \rhostar_k(\cdot\mid x_+)}\exp\prn*{ 20\etabar\ell^2 \nrm{\Ds_k(x)-\Ds_{k+1}(x_+)}^2} \\
    \leq&~ 2e^{8B-2\ell B} \sqrt{\prn*{1+\Dchis{\rho_k(\cdot\mid x_+)}{\rhostar_k(\cdot\mid x_+)}} \En_{x\sim \rho_k(\cdot\mid x_+)}\exp\prn*{ 12\etabar\ell^2 \nrm{\Ds_k(x)-\Ds_{k+1}(x_+)}^2}},
\end{align*}
and we also note that $\Dchis{\rho_k(\cdot\mid x_+)}{\rhostar_k(\cdot\mid x_+)}\leq E_{2,x_+}\leq E_{\ell,x_+}$. This immediately gives $E_{\ell,x_+}\leq \sqrt{E_{\ell,x_+}'}+E_{\ell,x_+}'$, where we define
\begin{align*}
    E_{\ell,x_+}'\deq 8e^{16B-4\ell B}\En_{x\sim \rho_k(\cdot\mid x_+)}\exp\prn*{ 12\etabar\ell^2 \nrm{\Ds_k(x)-\Ds_{k+1}(x_+)}^2}.
\end{align*}
By definition, we know
\begin{align*}
    \En_{x_+\sim p_{k+1}}[E_{\ell, x_+}']
    =8e^{16B-4\ell B}\En_{x\sim p, x_+\sim \normal{x,\eta\Id}}\exp\prn*{ 12\etabar\ell^2 \nrm{\Ds_k(x)-\Ds_{k+1}(x_+)}^2}.
\end{align*}
By \cref{cor:MGF-intrinsic}, we have the following bound with an absolute constant $C>0$:
\begin{align*}
    \En_{x\sim p, x_+\sim \normal{x,\eta\Id}} \exp\prn*{ \frac{\nrm{\Ds_k(x)-\Dn_{k+1}(x_+)}^2}{C\sigma_{k+1}^2} }
    \leq e^{\dbar}.
\end{align*}
Therefore, as long as $\ell^2\leq \frac{\sigma_{k+1}^2}{12C\etabar}$, we can upper bound 
\begin{align*}
    \En_{x_+\sim p_{k+1}}[E_{\ell, x_+}']
    \leq 8e^{16B}\exp\prn*{ -4\ell B + 12C\etabar \ell^2 \dbar/\sigma_{k+1}^2 }.
\end{align*}
We then choose $\ell=\min\crl{ \frac{\sigma_{k+1}}{\sqrt{12C\etabar}}, \frac{B\sigma_{k+1}^2}{6C \etabar \dbar} }$. Note that as long as $\ell\geq 8+\frac{3+\log(1/\delta)}{B}$, we have shown that $\En_{x_+\sim p_{k+1}}[E_{\ell, x_+}']\leq 8e^{16B-2B \ell}\leq 1$ and hence
\begin{align*}
    \En_{x_+\sim p_{k+1}}\Dchis{\rho_k(\cdot\mid x_+)}{\rhostar_k(\cdot\mid x_+)}
    \leq 8e^{8B-B\ell}\leq \delta.
\end{align*}
\qed

\begin{corollary}\label{cor:DDPM-chi-sq}
There is a constant $c>0$ such that as long as $c\etabar\leq \sigma_{k+1}^2$,
\begin{align*}
    \En_{x_+\sim p_{k+1}}\Dsys[2]{\rho_k(\cdot\mid x_+)}{\normal{\gds(x_+),\etabar\Id}}\leq ce^{c\dstar},
\end{align*}
\end{corollary}

\begin{proof}
Note that when $B=0$, we have $\rhostar(\cdot\mid x_+)=\normal{\gds(x_+),\etabar\Id}$. Therefore, from our proof of \cref{prop:intrinsic-chi-sq} above, we can extract the following fact (by setting $B=0$ and $\ell=2$):
\begin{align*}
    \Dsys[2]{\rho_k(\cdot\mid x_+)}{\normal{\gds(x_+),\etabar\Id}}\leq \sqrt{E_{x_+}}+E_{x_+},
\end{align*}
where we define 
\begin{align*}
    E_{x_+}\deq 8\En_{x\sim \rho_k(\cdot\mid x_+)}\exp\prn*{ 48\etabar\nrm{\Ds_k(x)-\Ds_{k+1}(x_+)}^2}.
\end{align*}
Therefore, as long as $\etabar\leq \frac{\sigma_{k+1}^2}{48C}$, $\En_{x_+\sim p_{k+1}}[E_{x_+}]\leq 8e^{\dbar}$. This is the desired upper bound.
\end{proof}

\subsection{\pfref{prop:sqrt-d-simple}}

In the following, to make the presentation clearer, we write
\begin{align*}
    \Wstar_{z,r,\xhat,x_+}(x)=\sigma^{-2}\tri{\dot\gamma_{z,r,\xhat}(x),\Ds_k(\gamma_{z,r,\xhat}(x))-\Ds_{k+1}(x_+)}\,.
\end{align*}
Then,
\begin{align*}
    \log \rho_k(x\mid x_+)-\log \rhostar_k(x\mid x_+)=\const+\En_{\lr,z}\trunc[\B]{\Wstar_{z,r,\xhat,x_+}(x)}\,.
\end{align*}
Therefore,
\begin{align*}
    \nabla_x \log \frac{\rho_k(x\mid x_+)}{\rhostar_k(x\mid x_+)}
    =&~ \En_{\lr,z}\brk*{ \nabla_x \Wstar_{z,r,\xhat,x_+}(x) \cdot h_B(\Wstar_{z,r,\xhat,x_+}(x))}\,, 
\end{align*}
where
\begin{align*}
h_B(y)=\begin{cases}
    1\,, & y>B\,,\\
    0\,, &y\in[-B,B]\,,\\
    -1\,, &y<-B\,,
\end{cases}
\end{align*}
is the derivative of $\trunc{y}$.
By elementary calculation,
\begin{align*}
    \sigma^2 \nabla \Wstar_{z,r,\xhat,x_+}(x)=\alrp\, (\Ds_k(\gamma_{z,r,\xhat}(x))-\Ds_{k+1}(x_+))+\alr\, \nabla \Ds_k(\gamma_{z,r,\xhat}(x)) \cdot \dot\gamma_{z,r,\xhat}(x)\,.
\end{align*}

By~\cref{lem:LSI-SC} and \cref{lem:HS-perturb}, we know that
$\CLSI(\rhostar_k(\cdot\mid x_+))\leq \etabar e^{4B}$. Hence,
\begin{align*}
    \Dkl{\rho_k(\cdot\mid x_+)}{\rhostar_k(\cdot\mid x_+)}\leqsim \eta\En_{x\sim \rho_k(\cdot\mid x_+)} \nrm[\Big]{\nabla \log \frac{\rho_k(x\mid x_+)}{\rhostar_k(x\mid x_+)}}^2\,.
\end{align*}
Taking expectation over $x_+ \sim p_{k+1}$ gives
\begin{align*}
    &\En_{x_+ \sim p_{k+1}} \Dkl{\rho_k(\cdot\mid x_+)}{\rhostar_k(\cdot\mid x_+)} \\
    &\qquad \leqsim \eta \En_{r\sim \unif([0,1])}\En_{z,r,x,x_+,\xhat}\brk[\big]{ \nrm{\nabla \Wstar_{z,r,\xhat,x_+}(x)}^2 \cdot \indic\crl{\abs{\Wstar_{z,r,\xhat,x_+}(x)}>\B} }\,, 
\end{align*}
where $\En_{x_+,x,z,\xhat}$ is the expectation over $x_+\sim p_{k+1}$, $x\sim \rho_k(\cdot\mid x_+)$, $\xhat\sim \normal{\gds(x_+),\frac12\etabar\Id}$, $z\sim \normal{0,\frac12\etabar\Id}$.

Fix any $r\in[0,1]$.
In \cref{lem:sub-exponential-bounds}, we show that under the distribution of $(z,x,x_+,\xhat)$,
\begin{itemize}
    \item the sub-exponential norm of $\nrm{\Ds_k(\gamma_{z,r,\xhat}(x))-\Ds_{k+1}(x_+)}^2$ is bounded by $O(\dbar\sigma^2)$,
    \item and the sub-exponential norm of $ \nrm{\nabla \Ds_k(\gamma_{z,r,\xhat}(x))\,\dot\gamma_{z,r,\xhat}(x)}^{2/3}$ is bounded by $O(\etabar^{1/3}\dbar)$.
\end{itemize}
Therefore, for a sufficiently large constant $C_1$, we can set $M_1=C_1\sigma^{-4}\,(\dbar\sigma^2\log^2(1/\delta)+\dbar^{3}\etabar\log^3(1/\delta))$ and bound
\begin{align*}
    \En (\nrm{\nabla \Wstar_{z,r,\xhat,x_+}(x)}^2-M_1)_+ \leq M_1\delta\,.
\end{align*}
Therefore, it holds that
\begin{align*}
    \En_{x_+\sim p_{k+1}} \Dkl{\rho_k(\cdot\mid x_+)}{\rhostar_k(\cdot\mid x_+)} 
    \leqsim &~ \eta M_1\,\prn[\big]{\delta+\bbP_{x_+,x,z,\xhat}\prn[\big]{\abs{\Wstar_{z,r,\xhat,x_+}(x)}>\B}} \\
    \leqsim &~ \dstar^3\,\prn[\big]{\delta+\bbP_{x_+,x,z,\xhat}\prn[\big]{\abs{\Wstar_{z,r,\xhat,x_+}(x)}>\B}}\,.
\end{align*}

Next, we bound $\bbP_{x_+,x,z,\xhat}\prn*{\abs{\Wstar_{z,r,\xhat,x_+}(x)}>\B}$. Using \cref{lem:Dcov}, we know that for any event $E$,
\begin{align*}
    \bbP_{x\sim \rho_k(\cdot\mid x_+)}(E)\leq  e\bbP_{x\sim \normal{\gds(x_+), \etabar\Id}}(E)+\Dcov[e]{\rho_k(\cdot\mid x_+)}{\normal{\gds(x_+), \etabar\Id}}.
\end{align*}
In the following, we denote $\eps_{x_+}\deq \Dcov[e]{\rho_k(\cdot\mid x_+)}{\normal{\gds(x_+), \etabar\Id}}$. Note that by \cref{cor:DDPM-coverage} and \cref{asmp:Lip-Frob}, we can bound
\begin{align}
    \En_{x_+\sim p_{k+1}}[\eps_{x_+}]\leqsim \dstar^2(\delta+\max_{t\in[\sigma_k^2,\sigma_{k+1}^2]}\PP_{Y_t\sim q_t}\prn*{\nrmF{\nabla m_t(Y_t)}\geq \LF}).
\end{align}

Hence, we can bound
\begin{align*}
    &~ \bbP_{x_+,x,z,\xhat}\prn[\big]{\abs{\Wstar_{z,r,\xhat,x_+}(x)}>\B} \\
    =&~ \En_{x_+\sim p_{k+1},\,\xhat\sim \normal{\gds(x_+),\frac12\etabar\Id},\, z\sim \normal{0,\frac12\etabar\Id}} \bbP_{x\sim \rho_k(\cdot\mid x_+)}(\abs{\Wstar_{z,r,\xhat,x_+}(x)}>\B) \\
    \leqsim&~  \En_{x_+\sim p_{k+1}} \bbP_{x\sim \normal{\gds(x_+), \etabar\Id}, \,\xhat\sim \normal{\gds(x_+),\frac12\etabar\Id},\, z\sim \normal{0,\frac12\etabar\Id}}(\abs{\Wstar_{z,r,\xhat,x_+}(x)}>\B)+\En_{x_+\sim p_{k+1}}[\eps_{x_+}] \\
    =&~\En_{x_+\sim p_{k+1}} \bbP_{x'\sim \normal{\gds(x_+), \etabar\Id},\, w\sim \normal{0,c\etabar\Id}}(\sigma^{-2}\abs{\tri{w,\Ds_k(x')-\Ds_{k+1}(x_+)}}\geq \B)+\En_{x_+\sim p_{k+1}}[\eps_{x_+}]\,,
\end{align*}
where the last line uses \cref{lem:ind-Gaussian}. Next, by Gaussian concentration, we know that under $w\sim \normal{0,c\etabar\Id}$, it holds that 
\begin{align*}
    \PP_w(\abs{\tri{w,\Ds_k(x)-\Ds_{k+1}(x_+)}}\geq \sqrt{2c\etabar\log(1/\delta)}\nrm{\Ds_k(x)-\Ds_{k+1}(x_+)})\leq 2\delta\,.
\end{align*}
Therefore, we can denote $M_2=8\sigma^{-2}B^{-1}\sqrt{\etabar\log(1/\delta)}$ and bound
\begin{align*}
    &~ \bbP_{x_+,x,z,\xhat}\prn[\big]{\abs{\Wstar_{z,r,\xhat,x_+}(x)}>\B} \\
    \leqsim &~ \PP_{x_+\sim p_{k+1},\, x\sim \normal{\gds(x_+), \etabar\Id}}\prn*{M_2\nrm{\Ds_k(x)-\Ds_{k+1}(x_+)}\geq 1}+\delta+\En_{x_+\sim p_{k+1}}[\eps_{x_+}] \\
    \leq&~ \PP_{x_+\sim p_{k+1},\, x\sim \normal{\gds(x_+), \etabar\Id},\, x'\sim \rho_k(\cdot\mid x_+)}\prn*{M_2\nrm{\Ds_k(x)-\Ds_k(x')}+M_2\nrm{\Ds_k(x')-\Ds_{k+1}(x_+)}\geq 1}\\
    &~+\delta+\En_{x_+\sim p_{k+1}}[\eps_{x_+}] \\
    \leq&~ \PP_{x_+\sim p_{k+1},\, x\sim \normal{\gds(x_+), \etabar\Id},\, x'\sim \rho_k(\cdot\mid x_+)}\prn*{2M_2\nrm{\Ds_k(x)-\Ds_k(x')}\geq 1} \\
    &~+\PP_{x_+\sim p_{k+1},\, x'\sim \rho_k(\cdot\mid x_+)}\prn*{2M_2\nrm{\Ds_k(x')-\Ds_{k+1}(x_+)}\geq 1}+\delta+\En_{x_+\sim p_{k+1}}[\eps_{x_+}]\,,
\end{align*}
where the second line uses the triangle inequality. 

For the second probability, we can apply \cref{cor:DDPM-coverage} to show (note that under \cref{eq:DM-Lip-Frob-eta} we can guarantee $\frac{1}{(2M_2)^2}\geq 8\LF^2\eta\log(e/\delta)$) that
\begin{align*}
    \PP_{x_+\sim p_{k+1}, x'\sim \rho_k(\cdot\mid x_+)}\prn*{2M_2\nrm{\Ds_k(x')-\Ds_{k+1}(x_+)}\geq 1}\leqsim \dstar^2(\delta+\max_{t\in[\sigma_k^2,\sigma_{k+1}^2]}\PP_{Y_t\sim q_t}\prn*{\nrmF{\nabla m_t(Y_t)}\geq \LF}). 
\end{align*}

For the first probability, we can apply change-of-measure (\cref{lem:Dcov}) again to get
\begin{align*}
    &~\PP_{x\sim \normal{\gds(x_+), \etabar\Id}, x'\sim \rho_k(\cdot\mid x_+)}\prn*{2M_2\nrm{\Ds_k(x)-\Ds_k(x')}\geq 1} \\
    \leq&~ e\PP_{x\sim \normal{\gds(x_+), \etabar\Id}, x'\sim \normal{\gds(x_+), \etabar\Id}}\prn*{2M_2\nrm{\Ds_k(x)-\Ds_k(x')}\geq 1}+\eps_{x_+}.
\end{align*}
Now, we express (using \eqref{eq:path-integral})
\begin{align*}
    \Ds_k(x)-\Ds_k(x')=\frac{\pi}{2}\int_0^1 \nabla \Ds_k(\sin(\pi r/2) x+\cos(\pi r/2) x')\cdot (\cos(\pi r/2) x-\sin(\pi r/2) x') dr.
\end{align*}
Note that with independent $x,x'\sim \normal{\gds(x_+), \etabar\Id}$, the random variable $(\sin(\pi r/2) x+\cos(\pi r/2) x',\cos(\pi r/2) x-\sin(\pi r/2) x')$ are independent Gaussian with marginal distribution $\normal{\gds(x_+), \etabar\Id}$ and $\normal{0, \etabar\Id}$, respectively. Therefore, using Markov's inequality, we bound
\begin{align*}
    \PP_{x,x'\sim \normal{\gds(x_+), \etabar\Id}}\prn*{2M_2\nrm{\Ds_k(x)-\Ds_k(x')}\geq 1}
    \leq &~\En_{x,x'\sim \normal{\gds(x_+), \etabar\Id}}\prn*{2M_2\nrm{\Ds_k(x)-\Ds_k(x')}-1}_+ \\
    \leq&~ \En_{x''\sim \normal{\gds(x_+), \etabar\Id}, w'\sim \normal{0, \etabar\Id}}\prn*{\pi M_2\nrm{\nabla \Ds_k(x'')w'}-1}_+.
\end{align*}
Further, we know the sub-exponential norm of $\nrm{\nabla \Ds_k(x)w'}^{2/3}$ is bounded by $O(\dbar \etabar^{1/3})$ (by the third inequality of \cref{lem:sub-exponential-bounds} with $r=0$). Therefore, by \cref{lem:sub-exp-trunc}, we can choose $M_3=C_3\sqrt{\dbar^3 \etabar \log^3(1/\delta)}$, and it holds that
\begin{align*}
    &~\En_{x_+\sim p_{k+1}, x\sim \normal{\gds(x_+), \etabar\Id}, w'\sim \normal{0, \etabar\Id}}\prn*{\pi M_2\nrm{\nabla \Ds_k(x)w'}-1}_+ \\
    \leqsim&~ M_2 M_3\prn*{ \delta+ \PP_{x_+\sim p_{k+1}, x'\sim \normal{\gds(x_+), \etabar\Id}, w'\sim \normal{0, \etabar\Id}}\prn*{\pi M_2\nrm{\nabla \Ds_k(x')w'}\geq 1} }.
\end{align*}
Note that $\PP_{w'\sim \normal{0, \etabar\Id}}\prn*{ \nrm{\nabla \Ds_k(x')w'}\geq 2\nrmF{\nabla \Ds_k(x')}\sqrt{\etabar\log(1/\delta)} }\leqsim \delta$. Therefore, combining the inequalities above, we get
\begin{align*}
    \En_{x_+\sim p_{k+1}} \Dkl{\rho_k(\cdot\mid x_+)}{\rhostar_k(\cdot\mid x_+)} 
    \leqsim
    &~\dstar^5\delta +\dstar^5\cdot \max_{t\in[\sigma_k^2,\sigma_{k+1}^2]}\PP_{Y_t\sim q_t}\prn*{\nrmF{\nabla m_t(Y_t)}\geq \LF}\\
    &~+\dstar^5\cdot \PP_{x_+\sim p_{k+1}, x'\sim \normal{\gds(x_+), \etabar\Id}}\prn*{\nrmF{\nabla \Ds_k(x')}\geq \LF}.
\end{align*}
Note that the distribution of $x'$ under $x_+\sim p_{k+1}, x'\sim \normal{\gds(x_+), \etabar\Id}$ is exactly $\wt{p}_k$. 
This is the desired upper bound.
\qed

\begin{lemma}\label{lem:sub-exponential-bounds}
Fix any $r\in[0,1]$. We consider the joint distribution of $x_+\sim p_{k+1}$, $x\sim \rho_k(\cdot\mid x_+)$, $\xhat\sim \normal{\gds(x_+),\frac12\etabar\Id}$, $z\sim \normal{0,\frac12\etabar\Id}$. Then the following holds for an absolute constant $c>0$:
\begin{align*}
    \En\exp\prn[\big]{ c\nrm{\Ds_k(\gamma_{z,r,\xhat}(x))-\Ds_{k+1}(x_+)}^2 }\leq e^{\dbar}, \qquad
    \En\exp\prn[\big]{ c\etabar ^{-1/3}\nrm{\nabla \Ds_k(\gamma_{z,r,\xhat}(x)) \dot \gamma_{z,r,\xhat}(x)}^{2/3} }\leq e^{\dbar}.
\end{align*}
Further, for $w\sim \normal{0,\etabar\Id}$ and $x'\sim \normal{\gds(x_+),\etabar\Id}$, it holds that
\begin{align*}
    \En\exp\prn*{ c\etabar ^{-1/3}\nrm{\nabla \Ds_k(x')w}^{2/3} }\leq e^{\dbar}.
\end{align*}
\end{lemma}

\begin{proof}
By \cref{cor:MGF-intrinsic}, there is a constant $c_0>0$ such that
\begin{align*}
    \En_{x_+\sim p_{k+1}, x\sim \rho_k(\cdot\mid x_+)}\exp\prn*{ 4c_0\sigma^{-2}\nrm{\Ds_k(x')-\Ds_{k+1}(x_+)}^2 }\leq e^{\dbar}, \\
    \En_{x_+\sim p_{k+1}, x\sim \rho_k(\cdot\mid x_+)}\exp\prn*{ 4c_0\tr(\nabla \Ds_k(x)) }\leq e^{\dbar}.
\end{align*}

We consider the following distributions of $(x_+,x,\xhat,z)$:
\begin{align*}
    P:&~  x_+\sim p_{k+1}, \quad x\sim \rho_k(\cdot\mid x_+), \quad \xhat\sim \normal{\gds(x_+),\frac12\etabar\Id}, \quad z\sim \normal{0,\frac12\etabar\Id}, \\
    Q:&~  x_+\sim p_{k+1}, \quad x\sim \normal{\gds(x_+),\etabar\Id}, \quad \xhat\sim \normal{\gds(x_+),\frac12\etabar\Id}, \quad z\sim \normal{0,\frac12\etabar\Id}.
\end{align*}
Then, by \cref{cor:DDPM-chi-sq}, we know $\Dsys[2]{P}{Q}\leq e^{\dbar}$. Further, under $Q$, we can apply \cref{lem:ind-Gaussian} to get
\begin{align*}
    \En_Q\exp\prn*{ 2c_0\sigma^{-2}\nrm{\Ds_k(\gamma_{z,r,\xhat}(x))-\Ds_{k+1}(x_+)}^2 }
    =&~\En_{x_+\sim p_{k+1}, x'\sim \normal{\gds(x_+),\etabar\Id}}\exp\prn*{ 2c_0\sigma^{-2}\nrm{\Ds_k(x')-\Ds_{k+1}(x_+)}^2 } \\
    =&~ \En_Q\exp\prn*{ 2c_0\sigma^{-2}\nrm{\Ds_k(x)-\Ds_{k+1}(x_+)}^2 }.
\end{align*}
Therefore, by $\chi^2$-change-of-measure, we can show that
\begin{align*}
    &~\En_P\exp\prn*{ c_0\sigma^{-2}\nrm{\Ds_k(\gamma_{z,r,\xhat}(x))-\Ds_{k+1}(x_+)}^2 } \\
    \leq&~ \sqrt{(1+\Dsys[2]{P}{Q})\En_Q\exp\prn*{ 2c_0\sigma^{-2}\nrm{\Ds_k(\gamma_{z,r,\xhat}(x))-\Ds_{k+1}(x_+)}^2 } } \\
    \leq&~ \sqrt{(1+e^{\dbar})\En_Q\exp\prn*{ 2c_0\sigma^{-2}\nrm{\Ds_k(x)-\Ds_{k+1}(x_+)}^2 } } \\
    \leq&~ (1+e^{\dbar})^{3/4} \prn*{ \En_P\exp\prn*{ 4c_0\sigma^{-2}\nrm{\Ds_k(x)-\Ds_{k+1}(x_+)}^2 } }
    \leq 1+e^{\dbar}.
\end{align*}
A corollary of this argument is that we also have
\begin{align*}
    \En_Q\exp\prn*{c_0\tr(\nabla \Ds_k(\gamma_{z,r,\xhat}(x)))}\leq 1+e^{\dbar}.
\end{align*}
Then, denote $w'\deq \dot \gamma_{z,r,\xhat}(x)$, we know $w'\sim \normal{0,c\etabar\Id}$ is independent of $\gamma_{z,r,\xhat}(x)$ under $Q$.
Then, by \cref{lem:Gaussian-vMGF} we know that for any matrix $A$ and $\delta\in(0,1)$,
\begin{align*}
    \PP_Q\prn*{ \nrm{Aw'}^2\leq c_1\etabar(\nrmF{A}^2+\nrmop{A}^2\log(1/\delta)) }\leq \delta,
\end{align*}
where $c_1>0$ is an absolute constant. Therefore, by a union bound,
\begin{align*}
    \PP_Q\prn*{ \nrm{\nabla \Ds_k(\gamma_{z,r,\xhat}(x)) w'}^2\leq c_2\etabar(\dbar+\log(1/\delta))^2\log(1/\delta) }\leq 2\delta.
\end{align*}
By integration, this implies that there is an absolute constant $c_3>0$ such that
\begin{align*}
\En_Q \exp\prn*{ c_3\etabar ^{-1/3}\nrm{\nabla \Ds_k(\gamma_{z,r,\xhat}(x)) w'}^{2/3} }\leq e^{\dbar}.
\end{align*}
Applying change-of-measure again, we see
\begin{align*}
\En_P \exp\prn*{ \frac12 c_3\etabar ^{-1/3}\nrm{\nabla \Ds_k(\gamma_{z,r,\xhat}(x)) \dot \gamma_{z,r,\xhat}(x)}^{2/3} }\leq 1+e^{\dbar}.
\end{align*}
An analogous argument also shows 
\begin{align*}
\En_P \En_{x'\sim \normal{\gds(x_+),\etabar\Id}, w\sim \normal{0,\etabar\Id}} \exp\prn*{ c_4\etabar ^{-1/3}\nrm{\nabla \Ds_k(x') w}^{2/3} }\leq 1+e^{\dbar}.
\end{align*}
Combining the inequalities above completes the proof.
\end{proof}

\section{Log-concave sampling}\label{appdx:log-concave}

Here, we apply~\cref{thm:gaussian_tilt} at each step of \cref{alg:prox-sampler} to sample from the RGO distribution.
Note that \cref{thm:gaussian_tilt} requires the choice of a point $x_+$ which is close to $\text{prox}_{\eta f}(x_0)$.
For simplicity, we assume that this error is zero, which amounts to assuming that we have access to the proximal oracle for $f$.
When $f$ is $\beta$-smooth and $h\le 1/(2\beta)$, then computation of the proximal map is a convex optimization problem, and error analysis can be done as in~\citet{AltChe24Warm}.
Otherwise, one could generalize the analysis to consider obtaining a stationary point of the proximal map optimization as in~\citet{FanYuaChe23ImprovedProx}.

We recall the definitions of PI and LSI in \cref{def:PI} and \cref{def:LSI}.

\begin{theorem}
    Let $\lambda\geq 2$ be a fixed constant, and $\varepsilon\in(0,\frac12]$.
    Suppose that $\mu\propto e^{-f}$ and that $f$ satisfies \cref{ass:holder}.
    Choose
    \begin{align*}
        \eta^{-1} = C\,\prn[\big]{\beta_s^2 d^s\log(1/\varepsilon) + \beta_s^2 d^{-(1-s)} \log^2(1/\varepsilon)}^{1/(1+s)}
    \end{align*}
    for a sufficiently large universal constant $C\fcedit{=C_\lambda} > 0$.
    Let $\wh\mu$ denote the law of the output of \cref{alg:prox-sampler} initialized at $\mu_0$, where in each step the RGO is implemented by \cref{alg:fors} via \cref{thm:gaussian_tilt}.
    Then, the following holds. 
    \begin{enumerate}
        \item Suppose that $\mu$ satisfies a log-Sobolev inequality with constant $\CLSI(\mu) < \infty$, which can only hold if $s = 1$ (i.e., $f$ is smooth).
        Then, $\Ren{\wh\mu}{\mu} \le \varepsilon^2$ using at most
        \begin{align*}
            N = \wt O\prn[\Big]{\kappa d^{1/2} \log^{3/2} \frac{\Ren{\mu_0}{\mu}}{\varepsilon^2} + \kappa \log^2 \frac{\Ren{\mu_0}{\mu}}{\varepsilon^2}} \qquad\text{first-order queries in expectation}\,,
        \end{align*}
        where $\kappa \deq \CLSI(\mu)\, \beta_1$ is the condition number. %
        \item Suppose that $\mu$ satisfies a Poincar\'e inequality with constant $\CPI(\mu) < \infty$.
        Then, $\Dchis{\wh\mu}{\mu} \le \varepsilon^2$ using at most
        \begin{align*}
            N = \wt O\prn[\Big]{\CPI(\mu)\,\beta_s^{2/(1+s)}\,d^{s/(1+s)} \log^{1/(1+s)}\prn[\big]{\frac{1}{\varepsilon}} \,\prn[\Big]{1 + \frac{\log^{1/(1+s)}(1/\varepsilon)}{d^{(1-s)/s}}}\log \frac{\Dchis{\mu_0}{\mu}}{\varepsilon^2}}
        \end{align*}
        first-order queries in expectation.
        \item Suppose that $\mu$ is log-concave.
        Then, $\Dkl{\wh\mu}{\mu} \le \varepsilon^2$ using at most
        \begin{align*}
            N = \wt O\prn[\Big]{\beta_s^{2/(1+s)}\,d^{s/(1+s)} \,\frac{W_2^2(\mu_0,\mu)}{\varepsilon^2}}\qquad\text{first-order queries in expectation}\,.
        \end{align*}
    \end{enumerate}
\end{theorem}
\begin{proof}
    This follows from tracking the error of the RGO implementation and choosing $\delta$ appropriately, exactly as in~\citet{AltChe24Warm}, and is therefore omitted.
\end{proof}

This theorem can be generalized in several directions.
One could assume that $\mu$ satisfies a Lata\l{}a--Oleszkiewicz inequality which interpolates between Poincar\'e and log-Sobolev, as in~\citet{Chen+22ProxSampler, Che+24LMC}.
One could also consider RGO implementations under more complicated settings, such as composite settings~\citep{FanYuaChe23ImprovedProx}.
For brevity, we omit such extensions.

\end{document}